%% file: fdiv-reloaded-neurips-arxiv.tex
\documentclass{article}
\usepackage[dvips]{graphicx}

\usepackage[main,final]{neurips_2026}

\usepackage[utf8]{inputenc} % allow utf-8 input
\usepackage[T1]{fontenc}    % use 8-bit T1 fonts
\usepackage{url}            % simple URL typesetting
\usepackage{booktabs}       % professional-quality tables
\usepackage{amsfonts}       % blackboard math symbols
\usepackage{amssymb,amsmath,color,amsthm}
\usepackage{nicefrac}       % compact symbols for 1/2, etc.
\usepackage{microtype}      % microtypography
\usepackage{xcolor}         % colors

\usepackage{wrapfig}

\usepackage[mathcal]{eucal}

\usepackage[colorlinks=true, allcolors=blue]{hyperref}       % hyperlinks

\newcommand{\BEAS}{\begin{eqnarray*}}
\newcommand{\EEAS}{\end{eqnarray*}}
\newcommand{\BEA}{\begin{eqnarray}}
\newcommand{\EEA}{\end{eqnarray}}
\newcommand{\BEQ}{\begin{equation}}
\newcommand{\EEQ}{\end{equation}}
\newcommand{\BIT}{\begin{itemize}}
\newcommand{\EIT}{\end{itemize}}
\newcommand{\BNUM}{\begin{enumerate}}
\newcommand{\ENUM}{\end{enumerate}}
\newcommand{\BA}{\begin{array}}
\newcommand{\EA}{\end{array}}

\newcommand{\Diag}{\mathop{\rm Diag}}

\newcommand{\tr}{\mathop{ \rm tr}}

\newcommand{\idm}{I}
\newcommand{\rb}{\mathbb{R}}
\newcommand{\mysec}[1]{Sec.~\ref{sec:#1}}
\newcommand{\eq}[1]{Eq.~(\ref{eq:#1})}
\newcommand{\myfig}[1]{Fig.~\ref{fig:#1}}
\def \E{{\mathbb E}}
\def \P{{\mathbb P}}
\def \H{{\mathcal H}}
\def \X{{\mathcal X}}
\def \W{{\mathcal W}}

\usepackage{enumitem}
\setlist[itemize]{leftmargin=*}

\newcommand{\bivec}[2]{\Big( \! \! \begin{array}{c} {#1} \\ {#2} \end{array} \! \! \Big)}
\newcommand{\bimat}[4]{\bigg(  \! \! \begin{array}{c@{\hspace{2mm}}c} {#1} &  {#2} \\ {#3} &  {#4} \end{array}   \! \! \bigg)}

\newtheorem{lemma}{Lemma}

\newtheorem{proposition}{Proposition}

\title{A Spectral Framework for Closed-Form \\ Relative Density Estimation}

\author{
Francis Bach \\
Inria - Ecole Normale Sup\'erieure \\ PSL Research University \\
\texttt{francis.bach@inria.fr}
}

\begin{document}

\maketitle

\begin{abstract}
We propose a closed-form spectral framework for relative log-density estimation in linearly parameterized probabilistic models, including unnormalized and conditional models. This is achieved by representing the Kullback-Leibler (KL) divergence as an integral of weighted chi-squared divergences, converting KL estimation into a family of least-squares problems. We derive an explicit spectral formula based only on first- and second-order feature moments, yielding closed-form estimators of both divergences and log-density potentials for fixed features. The framework extends to a broad class of $f$-divergences and can be combined with kernelization or feature learning with neural networks (with then only a closed form for the last layer). We prove convergence guarantees for the resulting estimators and empirically compare them on synthetic data with optimization-based variational formulations, including logistic and softmax regression for normalized conditional models.
\end{abstract}

\section{Introduction}
\label{sec:introduction}
Estimating the Kullback-Leibler (KL) divergence from samples is ubiquitous in data science and machine learning, as it reduces to relative log-density estimation~\cite{perez2008kullback}. This is an important example of the family of $f$-divergences between two probability distributions $p$ and $q$ on the same space $\X$, defined,  for convex  $f$ with domain containing $\rb_+^\ast$, as (see, e.g.,~\cite{polyanskiy2025information}):

\vspace*{-.3cm}

\[D(p\|q) =  \int_\X f\Big( \frac{dp}{dq}(x) \Big) dq(x),\]

\vspace*{-.2cm}

Key examples include $f(t) = t \log t - t + 1$ for the KL  divergence $\int_\X \log \big( \frac{dp}{dq}(x) \big) dp(x)$, and $f(t) = \frac{1}{2} (t-1)^2$ for the Pearson chi-squared divergence $\frac{1}{2}\int_\X   \big( \frac{dp}{dq}(x) - 1\big)^{\raisebox{-0.3ex}{$\scriptstyle 2$}} dq(x)$.

\textbf{Variational formulation.}
Representing $f(t)  = \sup_{ v \in \rb}v t - f^\ast(v)$ using the Fenchel conjugate $f^\ast$ of $f$, this is equivalent to relative density estimation through the variational formulation~\cite{broniatowski2006minimization,nguyen2010estimating}

\vspace*{-.3cm}

\BEQ
\label{eq:varDPq} D(p\|q)  = \sup_{ v : \X \to \rb }  \ 
  \int_\X\! v(x) dp(x)  -  \int_\X \! f^\ast(v(x)) dq(x),
\EEQ

\vspace*{-.15cm}

with an optimal potential $v$ defined as $v(x) = f'( dp/dq(x))$. This leads to a convex formulation (a convex problem when $v$ is linearly parameterized) without the need to normalize models, in particular when $p$ and $q$ are available as empirical distributions (since \eq{varDPq} is defined through expectations). In this paper, we refer to models where the normalizing constant can be tractably computed as ``normalized'' (such as when summing with respect to a small finite number of components), and to models with non-tractable normalizing constants as ``unnormalized.''

However, for the KL divergence, we have $f^\ast(u) = e^{u}-1$, and we obtain in \eq{varDPq} an expectation (with respect to  $q$) of exponentials, which can be statistically and numerically unstable when the expectation is replaced in practice by an empirical average (see, e.g.,~\cite{liu2015estimating,chehab2023provable}).

 For the Pearson divergence, however, we get $f^\ast(u) = \frac{u^2}{2} \!+\! u$, and \eq{varDPq} becomes quadratic in~$v$. For a linear model $v(x) = \theta^\top \varphi(x)$, based on a certain feature map $\varphi:\X \to \rb^m$, this leads to the maximization of
$\theta^\top (\mu_p\! -\! \mu_q) - \frac{1}{2} \theta^\top \Sigma_q \theta$ in closed form (based on the first two moments $\mu_p, \mu_q\in \rb^m, \Sigma_p, \Sigma_q \in \rb^{m \times m}$ of $\varphi(x)$ under $p$ and $q$), with optimal predictor $\theta = \Sigma_q^{-1} ( \mu_p- \mu_q) $ and optimal value
$\frac{1}{2}  (\mu_p\! -\! \mu_q)^\top \Sigma_q^{-1}  (\mu_p \!- \!\mu_q)$; this is equivalent to discriminant analysis for binary classification~\cite{hastie1995penalized,harchaoui2008testing}, and can be seen as a least-squares approach for density-ratio estimation~\cite{kanamori2009least,sugiyama2012density}.

The key features for this particular case are that  (1) the estimator and optimal values are given by a closed-form formula based on first- and second-order moments, so that (2) classical simple algorithms (based on regularization or gradient ascent) are available with a precise, sharp theoretical analysis~(see, e.g.,~\cite{bach2024learning}), with (3) usual extensions beyond fixed finite-dimensional feature spaces, through kernelization~\cite{scholkopf2002learning} or feature learning through neural networks (with then a closed form only for the last layer, with all other layers learned iteratively).
 The main question we aim to tackle in this paper is: \emph{Can we extend all these nice properties to other $f$-divergences, including the KL divergence?} In other words, can we leverage the least-squares formalism to tightly approximate a different geometry (in the conditional case, can we perform logistic regression through a fixed combination of least-squares regressions, and not by Newton's  method, which is iterative~\cite{mccullagh2019generalized})?

 \textbf{Variational formulations for chi-squared divergences.}
  In order to define our extension, we need an equivalent \emph{constrained} variational formulation with \emph{two}  functions $v$ and $w:\X \to \rb$~\cite{broniatowski2006minimization,nguyen2010estimating}:
\BEQ
\label{eq:varfdiv}  D(p\|q) =  \sup_{ v,w : \X \to \rb } \ 
  \int_\X \! v(x) dp(x)  \!+ \!\int_\X \!w(x) dq(x)  \  \mbox{ such that }  \ \forall x \in \X,
\ w(x) \leqslant - f^\ast(v(x)).
   \EEQ
   At optimality, we have $w(x) = - f^\ast(v(x))$, and we recover \eq{varDPq}, but this new formulation allows some additional flexibility that we can leverage.
   For the Pearson divergence, the function $w$ corresponding to the linear function $v$ is quadratic in $\varphi(x)$ since $f^\ast$ is quadratic (where ``quadratic'' means ``quadratic+linear''). Our variational formulation will estimate separately the two potentials $v$ and $w$ for a reduced variance compared to estimating only $v$ (which is what the variational method does), at the expense of some additional bias.
   
   We can extend these quadratic representations to \emph{weighted} chi-squared divergences such as those used in~\cite{yamada2013relative}, with $f(t) = \frac{1}{2} \frac{(t-1)^2}{\rho t + 1-\rho} $, where $\rho=0$ leads to the Pearson divergence. With the representation $f(t) =  \sup_{u \in \rb} u(t-1) - \frac{u^2}{2} ( \rho t + 1-\rho)$, we then have the representation
   
   \vspace*{-.7cm}
   
   \BEA
  \notag D(p\|q) & \!=\! &  \sup_{u:\X\to \rb}\  \int_\X \Big[
   u(x) \Big( \frac{dp}{dq}(x) - 1 \Big) - \frac{u(x)^2}{2} \Big( \rho \frac{dp}{dq}(x) + 1-\rho\Big)
   \Big] dq(x)\\
\label{eq:DpqUU}   & \!=\! &  \sup_{u:\X\to \rb}\  \int_\X \Big[ u(x) - \frac{\rho}{2} u(x)^2 \Big] dp(x)
   + \int_\X \Big[ - u(x) - \frac{1-\rho}{2} u(x)^2 \Big] dq(x), 
    \EEA
    which is quadratic in $u(x)$, corresponding to potentials
    $v(x) = u(x) - \frac{\rho}{2} u(x)^2$ and $w(x) = - u(x) - \frac{1-\rho}{2} u(x)^2$, which (non-obviously) satisfy $w(x) \leqslant - f^\ast(v(x))$. When $u$ is linearly parameterized as $u(x) = \theta^\top \varphi(x)$, we obtain a least-squares problem with an explicit solution, leading to the expression  $\frac{1}{2}  (\mu_p - \mu_q)^\top (\rho \Sigma_p + (1-\rho) \Sigma_q)^{-1}  (\mu_p - \mu_q)$, which is a lower bound on $D(p\|q)$.  All three nice properties highlighted above for the Pearson divergence also extend to this case.

\textbf{From chi-squared to ``all'' $f$-divergences.}
 It turns out that many convex functions are convex combinations of the functions above, that is, $$ f(t) = \int_0^1 \frac{1}{2} \frac{(t-1)^2}{\rho t + 1-\rho} d\nu(\rho),$$ for a probability distribution~$\nu$ on $[0,1]$. These are exactly the ``operator-convex'' functions (for which the definition of convexity holds for the Loewner order between symmetric matrices), such that $f(1)=f'(1)=0$ and $f''(1)=1$~\cite{lesniewski1999monotone,bach2022sum}. This includes the KL divergence, with~$\nu$ having density $2(1-\rho)$ with respect to the Lebesgue measure, and the weighted chi-squared divergences for $\nu$ equal to a Dirac at $\rho$ (see additional examples in Table~\ref{table:fdivcomplete} in App.~\ref{app:fdiv}).

This implies by integration that the associated relaxation of $f$-divergences has a formulation where the functions $v$ and $w$ are quadratic in $\varphi$, and are integrals of the individual quadratic functions defined above for each $\rho$. This leads to the definition of the following lower bound on $D(p\|q)$,
\BEQ
\label{eq:defF}
 F(p\|q ,\varphi)  =   \frac{1}{2}
\int_0^1 ( \mu_p - \mu_q)^\top (\rho \Sigma_p +(1-\rho)\Sigma_q)^{-1} ( \mu_p - \mu_q) 
d\nu(\rho) .
\EEQ
It inherits all the numerical and statistical properties of the Pearson divergence and chi-squared divergences, except that it requires integration with respect to $\rho$, which is problematic for algorithms.

\textbf{Spectral formulation.} As we show in \mysec{prop}, building on \cite{bach2022information,bach2022sum}, if we consider the generalized eigenvalue decomposition of the pair $(\Sigma_p,\Sigma_q)$, that is, with a basis $(v_i)_{i\in \{1,\dots,m\}}$ such that $  v_i^\top \Sigma_q v_{j}  = 1_{i=j}$ and $\Sigma_p v_i = \lambda_i \Sigma_q v_i$, we have (using the limit $f''(1)/2$ instead of $f(\lambda_i)/(\lambda_i-1)^2$ for $\lambda_i=1$) the explicit formula:
\BEQ
\label{eq:defFspec}
 F(p\|q ,\varphi)  =  \sum_{i=1}^m
\frac{ f(\lambda_i)}{(\lambda_i - 1)^2}  \big( (   \mu_p  - \mu_q)^\top v_i \big)^2.
\EEQ
%\vspace*{-.5cm}
Moreover, $F(p\|q ,\varphi) $ is convex with respect to the moments of $p$ and $q$, and gradients can be computed in closed form (with simple linear algebra) given the spectral decomposition. In addition, the associated (feasible) potentials $v$ and $w$ that are quadratic forms in $\varphi$ can be obtained \emph{in closed form}, thus leading, for the KL divergence, to the estimation of the log relative density as a quadratic function of $\varphi$ in closed form based on the first two moments (no normalization is performed, as opposed to the classical variational method). When considering the estimation of mutual information between two random variables, this leads to conditional density models (including normalized models such as logistic and softmax regression, but also unnormalized models, see \mysec{MI}). This is the main contribution of the paper, on which all other developments build.

\textbf{Feature learning algorithms.}
  Since, given $\varphi$ and the associated moments $\mu_p,\Sigma_p, \mu_q,\Sigma_q$, $F(p\|q ,\varphi) $ is a lower bound on $D(p\|q)$, $F(p\|q ,\varphi) $ can be naturally maximized with respect to $\varphi$, leading to explicit feature learning, which we explore in two forms in \mysec{featlearn}: (1) linear parameterizations, as a way to learn a more compact representation, akin to PCA, and (2) nonlinear parameterizations (e.g., using neural networks). Gradient ascent based on empirical moments can be performed, but stochastic approximation methods are developed as well, ``\`a la'' online/incremental EM~\cite{cappe2009line}.
  
  \textbf{Theoretical results.} The algorithmic contributions above are completed in \mysec{theory} with a theoretical analysis of the resulting estimators. We obtain partial adaptivity to smoothness with positive kernel methods, and adaptivity to linear latent variables with overparameterized neural networks, extending classical results in supervised learning~\cite{bach2017breaking}.

 \textbf{Empirical results.} In \mysec{experiments}, we illustrate with synthetic experiments our algorithmic and theoretical contributions, with empirical scaling laws, and empirical illustrations of the adaptivity to linear latent variables. In particular, we consider the estimation of mutual information, which leads to conditional density estimation (see details in \mysec{mutual}).

 Overall, we obtain a new flexible loss framework for estimating $f$-divergences (and thus log-density estimation, potentially conditional and for unnormalized models), which is numerically competitive with the plain variational method. For unnormalized models, the estimator does not explicitly aim to compute the normalization constant and is empirically more sample-efficient than the tested variational baseline (this can thus be seen as a potential alternative to score matching~\cite{hyvarinen2005score} and noise-contrastive estimation~\cite{gutmann2012noise}, with closed-form relative-density estimators for fixed features).
 
 \textbf{Related work.} This paper builds on kernel-based variational formulations of $f$-divergences~\cite{bach2022information,bach2022sum} but instead of relying on quantum divergences that require normalized feature maps (which are less adapted to feature learning), we extend the line of work~\cite{harchaoui2008testing,ribero2026regularized,kanamori2009least,sugiyama2012density}, which was dedicated to the Pearson divergence, to a broad subclass of $f$-divergences, and thus to the Kullback-Leibler divergence (in particular, a simpler spectral formulation already appears in \cite{bach2022information} but is only available for normalized feature maps and is not considered for conditional estimation). Not requiring normalization is crucial for the feature learning extensions that make the overall framework competitive (see detailed comparison in App.~\ref{app:comparison}). Another related line of work estimates $f$-divergences from binary classification with a specific loss~\cite{JMLR:v12:reid11a}, but, as shown in App.~\ref{app:binary}, this does not apply as easily to the KL divergence.

\section{Properties of the moment-based divergence}
\label{sec:prop}
 We now list some properties of the divergence $F$ defined in \eq{defF},  including a spectral formulation that avoids the integral in $\rho$. Unless otherwise stated, we consider finite-dimensional feature spaces $\varphi:\X \to \H$, with $\H  = \rb^m$. See proofs in App.~\ref{app:proof_main}.

  \textbf{Sufficient condition for finiteness.} In this paper, for simplicity, we will assume that $dp/dq$ exists almost everywhere and is strictly positive, that is, lies in an interval $[\alpha,1/\alpha]$, for $\alpha \in (0,1]$. Then $D(p\|q)$ is well-defined and both $D(p\|q)$ and $F(p\|q,\varphi)$ are bounded by $\max\{ f(\alpha), f(1/\alpha)\}$. 
   
\textbf{Convexity.} The function $(p,q) \mapsto F(p\|q,\varphi)$ is  jointly convex in $p$ and $q$, as the $f$-divergence is. We will compute its derivatives in Prop.~\ref{prop:closedform} below.

\textbf{Monotonicity.} The divergence $F$ is increasing in the feature map, that is, if $\tilde{\varphi} = (\varphi^\top, \varphi_+^\top)^\top$, then $ {F}(p\|q, \tilde{\varphi}) \geqslant  {F}(p\|q, {\varphi})$. As a lower bound on $D(p\|q)$, adding more features will make it closer.
 
\textbf{Linear invariance.}  If $\varphi(x)$ is  replaced by $V\varphi(x)$ for an \emph{invertible} map $V$, the value remains the same. If $V$ is injective, this remains true. If $V$ is surjective, the value of $F$ cannot increase (it is thus lower than or equal to the original value). See App.~\ref{app:constant} for an affine-invariant formulation.

 \textbf{Concavity in the kernel.} As for quantum divergences~\cite{bach2022information,bach2022sum}, the estimate is concave in the kernel function $(x,y) \mapsto \varphi(x)^\top \varphi(y)$, and hence in the kernel matrix. This  will be leveraged in \mysec{featlearn}.

\textbf{Link with $f$-divergence.} We always obtain a lower bound, which is tight under certain conditions on the representability of relative density $dp/dq$. See \mysec{infdim}  for a generic sufficient condition under which tightness holds for infinite-dimensional feature spaces (a key deviation from~\cite{bach2022information,bach2022sum} is that features do not need to be normalized).

 \begin{proposition}[Tightness]
 \label{prop:universality} Assume $\varphi: \X \to \rb^m$.
For all probability distributions $p$ and $q$ that have full support and  are mutually absolutely continuous (which implies that $dp/dq$ is finite and strictly positive), 
$ F(p\|q ,\varphi)    \leqslant D(p\|q),$
with equality if and only if for almost all $\rho$ in the support of $\nu$, there exists $\theta(\rho)$ such that for almost all $x \in \X$, $\big( \frac{dp}{dq}(x) - 1 \big) / \big( \rho \frac{dp}{dq}(x) + 1-\rho\big) = \theta(\rho)^\top \varphi(x)$.
\end{proposition}
Prop.~\ref{prop:universality} extends density ratio estimation with kernel methods and linear models~\cite{kanamori2009least,sugiyama2012density} that correspond to single $\rho$ (and thus weighted chi-squared divergences) to all $f$-divergences admitting the stated integral representation through integration with respect to $\rho$.

\subsection{Spectral formulation and closed-form potentials}

\label{sec:var}\label{sec:variational}
The variational formulation of $D(p\|q)$ in \eq{varfdiv} is central in the use of $f$-divergences in practice, with potentials $v:\X \to \rb$ and $w:\X\to \rb$ obtained from the relative density $dp/dq$ as $v(x) = f'( dp/dq(x))$ and $w(x) = g'(dq/dp(x))$, where $g(t) = t f(1/t)$ is the function recovering the same value of the $f$-divergence when $p$ and $q$ are swapped. It turns out that our new closed-form divergence also leads to estimates of these potentials. We assume that we are in a finite-dimensional setting, as this is what we obtain in practice (where the added regularizer allows us to use the representer theorem, as shown in \mysec{estimate}). 
An integral representation, obtained by integrating the solutions to \eq{DpqUU}, valid even for infinite-dimensional $\H$, can be found in App.~\ref{app:var}. 

\begin{proposition}
\label{prop:closedform} Assume that $\H$ is finite-dimensional and $\Sigma_q$ invertible.
Consider the generalized eigenvalue decomposition of the pair $(\Sigma_p,\Sigma_q)$, that is, with a basis $(v_i)_{i  \in\{1,\dots,m\}}$ such that $  v_i^\top \Sigma_q v_{j}   = 1_{i=j}$ and $\Sigma_p v_i = \lambda_i \Sigma_q v_i$. Define $v(x) =  \varphi(x)^\top M \varphi(x)  + 2   c^\top \varphi(x)$ and $w(x) = \varphi(x)^\top N \varphi(x) - 2   c ^\top \varphi(x) $, as quadratic functions of $\varphi(x)$, with

\vspace*{-.5cm}

 \BEAS
\textstyle c    
 & = & \textstyle \sum_{i = 1}^m  \frac{{f}(\lambda_i)}{(\lambda_i-1)^2} v_i v_i^\top (\mu_p-\mu_q) \\[-.1cm]
M  & = & \textstyle \sum_{i,j =1}^m   (\mu_p-\mu_q)^\top  v_i  v_j^\top   (\mu_p-\mu_q)   \frac{f(\lambda_i)/(\lambda_i-1)^2-f(\lambda_j)/(\lambda_j-1)^2}{\lambda_i - \lambda_j} v_i  v_j^\top  \\[-.1cm]
N   & = &\textstyle  - 
 \sum_{i,j = 1}^m  (\mu_p-\mu_q)^\top  v_i  v_j^\top   (\mu_p-\mu_q)  \frac{f(\lambda_i) \lambda_i /(\lambda_i-1)^2-f(\lambda_j) \lambda_j /(\lambda_j-1)^2}{\lambda_i - \lambda_j}  v_i   v_j^\top ; \\[-.4cm]
 \EEAS
$v$ and $w$ are feasible for \eq{varfdiv}, and $F(p\|q,\varphi) 
 = \int_{\X} v(x) dp(x) + \int_{\X} w(x) dq(x)$, leading to \eq{defFspec} 
  (above, the divided differences are replaced by the corresponding derivatives when $\lambda_i = \lambda_j$).
\end{proposition}
We can make the following observations:

\vspace*{-.2cm}

\begin{itemize}[itemsep=2pt, parsep=0pt, topsep=2pt]

\item The proof in App.~\ref{app:var} is based on providing a variational formulation of $F(p\|q,\varphi) $ adapted from~\cite{bach2022sum}, which is an explicit convex relaxation of \eq{varfdiv} using quadratic functions for $v$ and $w$ and semi-definite constraints to make them feasible, akin to sum-of-squares optimization~\cite{henrion2020moment}.  It turns out that $M,N,2c$ are the respective partial derivatives of $F(p\|q,\varphi)$ seen as a function of $\Sigma_p, \Sigma_q, \mu_p-\mu_q$. Prop.~\ref{prop:universality} gives sufficient conditions for the potentials given in Prop.~\ref{prop:closedform} to be equal to the optimal ones.
\item When $\Sigma_q$ is not invertible, in particular when using empirical estimates in \mysec{estimate}, we add to $\Sigma_q$ and $\Sigma_p$ the matrix $\lambda \idm$ (Prop.~\ref{prop:closedform} then also applies).  
 
\item For the Pearson divergence, i.e., $f(t) = \frac{1}{2} (t-1)^2$, $M$ in Prop.~\ref{prop:closedform} is equal to zero, and we obtain a linear form in $\varphi(x)$ that aims to estimate $dp/dq(x) -1$, exactly recovering existing work~\cite{eric2007testing,harchaoui2008testing,ribero2026regularized}. Note that when the relative density is close to zero, there is an issue with this estimate, which may be negative (when $v(x)$ is less than $-1$, see \myfig{examples_regular} in App.~\ref{app:experiments}).
\item For the KL divergence, i.e., $f(t) = t \log t - t + 1$, $v(x) = \log (dp/dq(x))$ is estimated by a quadratic form in $\varphi(x)$. This provides a new way to perform relative log-density estimation in closed form, illustrated in \myfig{examples_regular} in App.~\ref{app:experiments}.
\EIT

\vspace*{-.2cm}

Overall, we obtain a closed-form expression for an estimate of the relative log-density. As shown in \myfig{examples_regular} in App.~\ref{app:experiments}, for many observations, approaching the population limit, this estimator is close to the one obtained from the variational method as soon as the feature space is large enough.
 
\subsection{Regularized estimation and kernelization}
\label{sec:estimate}
We are given $n_p$ points $x_1,\dots,x_{n_p}$ and $n_q$ points $y_1,\dots,y_{n_q}$, sampled independently and identically distributed (i.i.d.) from $p$ and~$q$, respectively. Following a long series of work on kernel measures of dissimilarity between distributions~\cite{bach2002kernel,gretton2012kernel,harchaoui2008testing,bach2022information,bach2022sum}, a natural estimator for 
$F(p\|q,\varphi) $ is  obtained by  (1) adding a regularizer, which corresponds to replacing $ {\Sigma}_p$, $ {\Sigma}_q $ with $ {\Sigma}_p +\lambda \idm$, $ {\Sigma}_q+\lambda \idm$, where $\lambda > 0$ is a regularization parameter,
which leads to the definition of a regularized divergence (for which Prop.~\ref{prop:closedform} can be used to derive spectral estimates)
\BEQ
\label{eq:defFreg}
 F_\lambda(p\|q ,\varphi)  =   \frac{1}{2}
\int_0^1 ( \mu_p - \mu_q)^\top (\rho \Sigma_p +(1-\rho)\Sigma_q + \lambda \idm)^{-1} ( \mu_p - \mu_q) 
d\nu(\rho) ,
\EEQ
and
(2) considering the empirical moment matrices, which correspond to using the empirical distributions $\hat{p}$ and $\hat{q}$. This also corresponds to estimating each function $u_\rho: \X \to \rb$ in \eq{DpqUU} as $\theta_\rho^\top \varphi(x)$, and penalizing with $-\frac{\lambda}{2} \int_0^1 \| \theta_\rho\| ^2 d\nu(\rho)$.

When the $m$-dimensional feature map $\varphi$ is given explicitly, we can compute the data matrices $\Phi_p \in \rb^{n_p \times m}$ and $\Phi_q \in \rb^{n_q \times m}$ whose rows are feature vectors. We then have empirical moments $ 
\hat{\mu}_p = \frac{1}{n_p} \Phi_p^\top 1_{n_p}, \ \hat{\mu}_q = \frac{1}{n_q} \Phi_q^\top 1_{n_q}, \ \hat{\Sigma}_p = \frac{1}{n_p} \Phi_p^\top \Phi_p, \ \hat{\Sigma}_q = \frac{1}{n_q} \Phi_q^\top \Phi_q
$ (that can be computed in running time $O(m^2n)$, with $n = n_p+n_q$),
and we simply use \eq{defFspec}, with cost $O(m^3)$.

\textbf{Kernelization.} A first step towards infinite-dimensional spaces is to show that the estimator can be computed solely based on dot-products between features. This is possible because $F_\lambda$ is defined as an integral over $\rho$, and each summand is a classical least-squares problem where the usual representer theorem can be used~\cite{scholkopf2002learning}. See App.~\ref{sec:effcomp} for explicit expressions. In practice, to avoid computing the full kernel matrix, we can consider explicit feature maps constructed from the kernel functions, using random features~\cite{rahimi2008random,rudi2017generalization} or the Nystr\"om extension~\cite{mahoney2009cur,martinsson2020randomized,williams2000using,rudi2015less}. In \mysec{experiments}, we use random features at the interface between neural networks and multivariate splines~(see, e.g., \cite{bach2023relationship}).

\subsection{Mutual information}
\label{sec:mutual}
\label{sec:MI}

Given two generic spaces $\X_1$, $\X_2$ and a joint distribution $p$ on $\X_1\times \X_2$, let $q$  be the product of its marginal distributions. Then $D(p\|q)$, which we will refer to as $I(p)$, is exactly the mutual information (MI) when $D$ is the Kullback-Leibler divergence. This measure of statistical dependence extends to all $f$-divergences~\cite{csiszar1967information,letizia2024mutual}, and the variational formulation from \eq{varDPq} has been considered before~\cite{suzuki2010sufficient,belghazi2018mine,poole2019variational}.  This provides a special case of independent interest, in particular as it provides a generic way to obtain conditional densities using Prop.~\ref{prop:closedform} (such as for logistic or softmax regression, see App.~\ref{app:experiments}). In this section, we only consider two random variables $x_1$ and $x_2$, but the framework generalizes to mutual information between more than two random variables.

\textbf{Expression of $F$ using factorized moments.}
Given two variables $(x_1,x_2) \in \X_1 \times \X_2$, with feature map $\varphi_1:\X_1\to\H_1$ and $\varphi_2: \X_2 \to \mathcal{H}_2$, the first two moments for the specific feature map $\varphi_2(x_2) \otimes \varphi_1(x_1) $, which is the Kronecker product of $\varphi_1(x_1)$ and  $\varphi_2(x_2)$, can be computed as follows, with all expectations taken with respect to the joint distribution of $(x_1,x_2)$:
\BEAS
\!\!\!\!\!\!\!\!\!\!& & \mu_p = \E [ \varphi_2(x_2) \otimes \varphi_1(x_1) ] , \ \ \hspace*{3.05cm}
\mu_q =  \E[ \varphi_2(x_2) ] \otimes \E [ \varphi_1(x_1) ]    \\[-.075cm]
\!\!\!\!\! \!\!\!\!\!&&\Sigma_p = \E [  \varphi_2(x_2)\varphi_2(x_2)^\top \otimes   \varphi_1(x_1)\varphi_1(x_1)^\top   ] ,  \ \ \ \
\Sigma_q = \E [\varphi_2(x_2)\varphi_2(x_2)^\top ]   \otimes \E[ \varphi_1(x_1)\varphi_1(x_1) ^\top   ]  ,
\EEAS
with natural empirical estimates based on the empirical distribution of data $(x_{1i},x_{2i})$, $i=1,\dots,n$.

\textbf{Closed-form estimation.} For the KL divergence, the potential $v(x_1,x_2)$ that is the solution of the variational problem in \eq{varDPq} is  $f'(dp(x_1,x_2) / dp(x_1)dp(x_2))
= \log [ dp(x_1,x_2) / dp(x_1)dp(x_2)] = \log [ dp(x_2|x_1)/ dp(x_2)]$, where we use the usual graphical model convention that $p(x_i)$ is the marginal distribution on the variable $x_i$. If $m_1 = {\rm dim}(\H_1)$ and $m_2 = {\rm dim}(\H_2)$, the running-time complexity is $O(m_1^2m_2^2 n + m_1^3 m_2^3)$, which is not practical as soon as $m_1$ and $m_2$ exceed 100. This is where the feature learning algorithms from \mysec{featlearn} are crucial.

\textbf{Link with softmax regression.} For $\X_2$ finite with $m_2$ elements, where the marginal probability over $x_2$ can easily be estimated, with one-hot encodings for $x_2$, this leads to an estimate of the logarithm of the conditional probabilities of $x_2$ given~$x_1$, which we write $p(x_2|x_1)$ (more precisely, $v(x_1,x_2)$ estimates $\log p(x_2|x_1) - \log p(x_2)$). Our framework thus leads to an estimate of $\log p(x_2|x_1)$ that
is, for each value of $x_2$, a quadratic function of $\varphi(x_1)$, using only first- and second-order moments of class-conditional distributions, without the need to solve any optimization problem (as logistic or softmax regression would require); note that these estimates are not normalized, that is, $\sum_{x_2 \in \X_2} \exp( v(x_1,x_2) + \log p(x_2))$ may not be equal to $1$. Moreover, this lack of normalization is a bias issue (due to estimating two potentials $v$ and $w$) and will not be alleviated with infinitely many observations (see \myfig{examples_MI} in App.~\ref{app:experiments}).

In terms of computations, for this special case, we only need, for each $j \in \{1,\dots,m_2\}$, class frequencies $\pi_j$, means $\mu_j$, and second-moment $\Sigma_j$, to compute (potentially in parallel) the $m_2$ generalized eigenvalue decompositions of the pairs $(\Sigma_j,\Sigma)$, and then apply Prop.~\ref{prop:closedform} (see App.~\ref{app:softmax} for explicit expressions). This can be done in running time $O(m_1^2n)$ to compute moments and $O(m_2m_1^3)$ for the eigenvalue decompositions (which is advantageous compared with softmax regression using Newton’s method, which is $O(m_2^2m_1^2n+m_2^3 m_1^3)$) per iteration---see \mysec{featlearn} below for an approximation algorithm with complexity $O(m_1m_2 n)$ per iteration. \myfig{examples_MI} in App.~\ref{app:experiments} illustrates how our estimator performs similarly to softmax regression while remaining closed form.

\section{Feature learning algorithms}

\label{sec:featlearn}
A key issue with kernel-based methods is the choice of kernel. In our context, as in~\cite{bach2022information,bach2022sum}, our divergences, which are lower bounds on the true value $D(p\|q)$, are concave in the kernel, making them natural candidates for various frameworks for kernel learning, such as multiple kernel learning~\cite{bach2004multiple,rakotomamonjy2008simplemkl,gonen2011multiple} or continuous spline learning~\cite{follain2025enhanced}. In this paper, we focus on a more versatile direct optimization method amenable to nonlinear extensions based on neural networks.

Given two distributions $p$ and $q$ available via samples, we aim to learn a parameterized feature map $\psi_\Gamma: \X\to \rb^m$, with any parameterized function $\psi_\Gamma$. We will consider estimates for the function~$u$ in \eq{DpqUU} of the form $\theta^\top \psi_\Gamma$, where $\theta$ will be estimated in closed form and $\psi_\Gamma$ iteratively (as in classical supervised learning).

If we manage to find a map such that the obtained $F$-divergence is a good approximation of $D(p\|q)$ so that the optimal $v(x)$ and $w(x)$ are close to quadratic functions of $\psi_\Gamma(x)$, then by looking at equality cases in the data-processing inequality~\cite{polyanskiy2025information}, we see that the relative density $dp/dq$ has to be a function of $\psi_\Gamma(x)$. Thus, we can view this feature learning as the task of identifying sufficient statistics for relative density estimation. For finite $\X$ and one-hot encodings, a single feature (which is $dp/dq(x)$) is enough for the exact divergence, but our relaxation may still require more; we use $r=8$ for our generic density-ratio experiments (\myfig{comparisonNN}, right), more features are required for the mutual information (see \myfig{comparisonNN}, left).

\textbf{Linear learning from a fixed feature map.} This corresponds to maximizing with respect to $\Gamma \in \rb^{m \times r}$ the quantity $F(\hat{p}\| \hat{q}, \Gamma^\top \varphi)$, with a regularization of second moments by $\lambda \Gamma^\top \Gamma$ rather than $\lambda \idm$, so that this exactly provides a low-rank approximation from below of $F(\hat{p}\| \hat{q}, \varphi)$.

Since $F$ is a convex function, we can use a ``minorization-maximization'' algorithm~\cite{hunter2004tutorial,mairal2013stochastic} where $F$ is locally approximated by its tangent as a function of the moments, with a linear term exactly equal
to (with $M,N,c$ defined in Prop.~\ref{prop:closedform}): 
\[
\tr\Big[ M \!\!\int_\X \! \Gamma^\top \varphi(x) \varphi(x)^\top \Gamma d\hat{p}(x) \Big]
+ \tr\Big[ N \!\!\int_\X \! \Gamma^\top \varphi(x) \varphi(x)^\top \Gamma d\hat{q}(x)  \Big] + 2c^\top \!\!
\int_\X \! \Gamma^\top \varphi(x) (d\hat{p}(x)-d\hat{q}(x)).\]
This quadratic lower bound can be maximized in closed form by solving a linear system, with complexity $O(m^3 r^3)$ once the full moments are available in $O(m^2n)$, but the system above can be solved approximately with a few iterations of conjugate gradient~\cite{golub83matrix} with complexity $O(rmn)$ to compute each gradient, and $O(r^3)$ to apply Prop.~\ref{prop:closedform}, which is an improvement over the full rank version in $O(m^3 + m^2n)$---we consider $r=8$ in experiments.   Overall, this is a monotonic algorithm with no convergence guarantees unless $\nu$ is a Dirac (i.e., a weighted chi-squared divergence), in which case it reduces to the power method for singular values~\cite{golub83matrix}.

\textbf{Gradient ascent and stochastic approximation.} We can consider a generic feature map $\psi_\Gamma: \X \to \rb^r$ and consider a lower bound on $F(\hat{p}\|\hat{q},\psi_\Gamma)$ of the form
\[
\tr\Big[ M \!\!\int_\X \!  \psi_\Gamma(x)  \psi_\Gamma(x)^\top d\hat{p}(x) \Big]
+ \tr\Big[ N \!\!\int_\X \!  \psi_\Gamma(x)  \psi_\Gamma(x)^\top d\hat{q}(x)  \Big] + 2c^\top \!\!
\int_\X \!  \psi_\Gamma(x) (d\hat{p}(x)-d\hat{q}(x)).\]
We make a few gradient ascent steps with respect to $\Gamma$, and then recompute the reduced moments and update $M,N,c$ with Prop.~\ref{prop:closedform}. For small step-sizes, this is an ascent algorithm that needs to recompute the reduced moments at every iteration and compute the gradient of $\psi_\Gamma$ at each observation. %When the number of observations is large, this is not feasible.

When the number $n$ of observations is large, if $M,N,c$ are fixed, we can use stochastic gradient ascent for $\Gamma$, with the update, after seeing $x,y$,
$
\Gamma \leftarrow \Gamma + \gamma \nabla_\Gamma \big[ \tr( \psi_\Gamma(x)^\top M   \psi_\Gamma(x))   +\tr( \psi_\Gamma(y)^\top N   \psi_\Gamma(y))   + 2c^\top(\psi_\Gamma(x) - \psi_\Gamma(y))\big]
$. Similarly to online expectation-maximization (EM)~\cite{cappe2009line}, we can accumulate/update reduced moments of size $r$ and $r \times r$ with, e.g., $\Sigma_p \leftarrow \Sigma_p ( 1 - \gamma) + \gamma\psi_\Gamma(x)  \psi_\Gamma(x)^\top $,
and update the matrices $M,N,c$ at every stochastic gradient step (or any batch) through Prop.~\ref{prop:closedform}. These algorithms currently come with no theoretical guarantees (although we conjecture that stationary points are the only possible limits). In our experiments, we use multiple-pass stochastic gradient ascent with re-evaluation of the moments every few passes over the data.

\textbf{Regularization.} We choose terms that correspond to regularizing each least-squares estimation of~$u$ for each $\rho \in [0,1]$ in \eq{DpqUU} with a penalty term that encourages the desired learning behaviors (e.g., learning of shared representations). For example, for $\psi_\Gamma(x) = \Gamma^\top \varphi(x)$ for $\Gamma \in \rb^{m \times r}$ and $\varphi(x)$ a fixed feature map in $\rb^m$,  we regularize $\Sigma_p$ by replacing it with $\Sigma_p + \lambda \Gamma^\top \Gamma$ (and similarly for $\Sigma_q$).

\textbf{Neural networks.} For $x \in \rb^d$, we consider the feature map
$ \Psi_{\{w,b\}}(x) =  ((w_j^\top x + b_j)_+)_{j \in \{1,\dots,m\}}$ with parameters $w_j \in \rb^d$ and $b_j \in \rb$, for $j \in \{1,\dots,m\}$, and maximize the regularized functional $F_\lambda(\hat{p}\| \hat{q},\Psi_{\{w,b\}}) - \frac{\lambda}{2} \sum_{j=1}^m \{ \| w_j\|^2 + b_j^2\}$. As shown in  App.~\ref{app:proofneuraldecay}, this corresponds to neural networks for each $\rho$-dependent problem in \eq{DpqUU}, with a shared input layer and weight decay. Theoretical results in \mysec{theoryNN} are provided for this architecture; however, it requires applying Prop.~\ref{prop:closedform} with $m$ that can be quite large. In our experiment, we consider the feature map $\Lambda^\top \Psi_{\{w,b\}}$ for $\Lambda \in \rb^{m \times r}$ (as done for linear learning earlier in this section), for $r=8$. In terms of computational complexity, for the classical variational technique, each iteration has complexity $O(mdn)$, while, for ours, it is $O(mdn + mrn + r^3)$ to compute moments, perform the eigenvalue decomposition, and update parameters.

\textbf{Mutual information.} 
For the mutual information case from \mysec{MI}, we need to learn two feature maps $\psi_{\Gamma_1}^{(1)}(x_1)$ and $\psi_{\Gamma_2}^{(2)}(x_2)$, one for each variable, with essentially the same techniques (see App.~\ref{app:featlearnMI}, including, in particular, a link with canonical correlation analysis). A key difference from the generic case is that the number of required features is larger since the equality case in the data-processing inequality is that $x_1$ and $x_2$ are independent given $\psi_{\Gamma_1}^{(1)}(x_1)$ and $\psi_{\Gamma_2}^{(2)}(x_2)$, and there is, even in the simplest discrete situations, no simple one-dimensional feature that will capture the entire dependence. In the left plot of \myfig{comparisonNN}, we show how $r=24$ for both feature maps is enough in a simple case.

\section{Theoretical results}
\label{sec:theory}
We now provide convergence rates for the estimation of $f$-divergences in a canonical example:  $\X$ is the unit Euclidean ball in $\rb^d$, while $p$ and $q$  have strictly positive densities (with respect to the Lebesgue measure) with sufficiently many derivatives. While the results presented in this section can be made more general, this case allows comparison with existing results, and follows the framework of~\cite{bach2017breaking}, which tackles both kernel methods and neural networks (we consider high-dimensional Gaussian density estimation in follow-up work~\cite{Bach2026_gaussian}, with asymptotic computations rather than non-asymptotic bounds). We consider the setup of \mysec{estimate}, with $n = \min \{ n_p, n_q\}$.

\subsection{Convergence rate for kernel methods}
\label{sec:infdim}
We assume that $\varphi: \X \to \H$ is the canonical feature map associated with the positive definite kernel from~\cite{bach2017breaking}, which is essentially the limit of random features of the form $(w^\top x + b)_+^\kappa$ for some non-negative integer $\kappa$ (this is how we can implement it in practice, which is more efficient than using the kernel trick). This is a universal kernel that can potentially learn any function~\cite{sriperumbudur2011universality}.

\begin{proposition}
\label{prop:kernelrate}
Assume the densities of $p$ and $q$ are strictly positive and have $t \leqslant \kappa+\frac{d+1}{2}$ square-integrable derivatives. For the kernel defined from the $\kappa$-th power of the ReLU activation, we have, for a regularization parameter $\lambda \propto 1/ {n^{\frac{\kappa+{(d+1)}/{2}}{t+d/2}}}$, for a constant $C$ independent of $n$,
\BEQ
\E\big[ | F_\lambda(\hat{p},\hat{q}, \varphi )  - D(p\|q) | \big]  \leqslant C \big(  {n^{- \frac{t}{t+d/2}}} +  D(p\|q)   {n^{-\frac{t/2}{t+d/2}}} \sqrt{\log n} + n^{-1/2} \sqrt{ D(p\|q)}\big).
\EEQ
\end{proposition}
This proposition, proved in App.~\ref{app:proofkernel}, provides a rate with an additive term and two multiplicative terms. It slightly improves over the rate in the entropy-estimation setting of~\cite{bach2022information} (see App.~\ref{sec:tradeoffbound}) when $t=1$, but is still inferior to the optimal rate for the estimation of such functionals~\cite{laurent1996efficient}, which will be matched in \mysec{debias} for the additive term. Note that (1) as in supervised learning, we obtain adaptivity to smoothness, (2) we only escape the curse of dimensionality if the degree of smoothness~$t$ is proportional to $d$, and (3) the rates could also be extended to the estimation of mutual information.
 
 \textbf{Related work.} There are existing similar estimators, mostly for the Pearson case ($\rho=0$)~\cite{harchaoui2008testing,ribero2026regularized,kanamori2009least,sugiyama2012density} but also for fixed $\rho \neq 0$ \cite{yamada2013relative}. The rates that they obtain are similar to the ones in Prop.~\ref{prop:kernelrate} but do not apply to the KL divergence directly, typically lead to density ratios that can be negative at some points (as highlighted in 
\myfig{examples_regular} in App.~\ref{app:experiments}). We also propose in \mysec{debias} a debiased estimate that can lead to improved rates (and sometimes minimax) and that applies to all $f$-divergences that have our integral decomposition. We also have results for feature learning (Prop.~\ref{prop:latent}) that show partial adaptivity to linear latent variables.
%\vspace*{-.125cm}

\begin{figure}
\includegraphics[width=4cm]{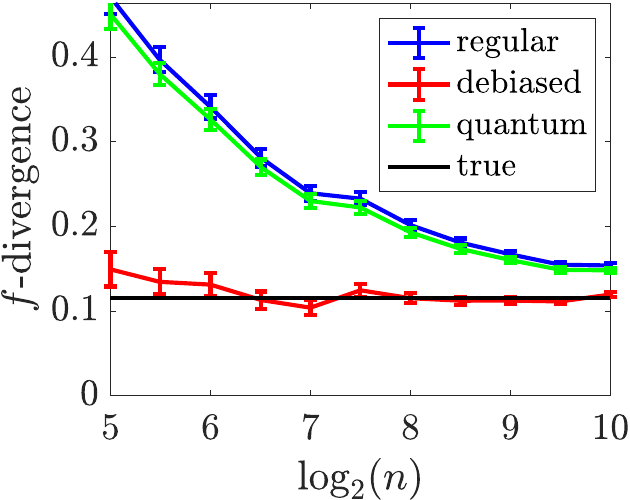} \hspace*{.0cm}
\includegraphics[width=4cm]{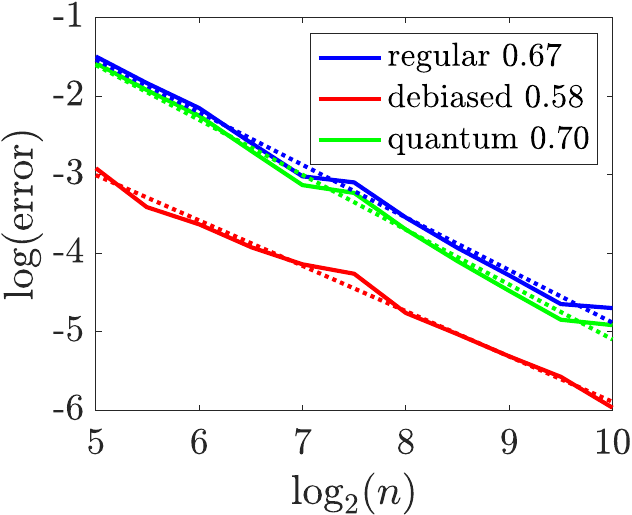}  \hspace*{.2cm} \includegraphics[width=5.5cm]{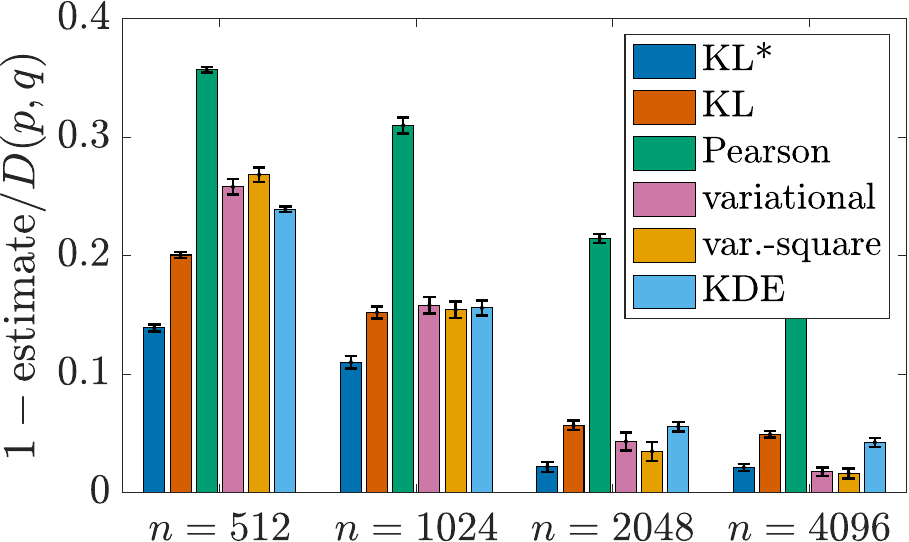}

\vspace*{-.2cm}

\caption{\normalsize \textbf{Left and Center:} Comparison of quantum, regular, and debiased estimators of the KL divergence for $q$ uniform on $[0,1]$ and $p$ with a density having one square-integrable derivative, with a kernel equivalent to $\kappa=0$ in Prop.~\ref{prop:kernelrate}, with a fixed $n$-dependent schedule for $\lambda$.  Left: comparison of estimators with increasing $n$ (mean and standard errors over 64 replications). Center: scaling laws with a power-law fit (decay proportional to  $n^{-\sigma}$ with $\sigma$ obtained from the data); the additive term in our bound is $\sigma = 2/3$. \textbf{Right:} Comparison of estimators of $v$ and $w$ (lower is better), with periodic data in $[0,1]^2$ mapped to a product of unit circles, with $q$ uniform and $p$ with independent components such that $\log(dp/dq)$ is a sum of a few cosines. We consider the ReLU kernel of order 1~\cite{bach2023relationship}, with $m=512$ random features. We estimate from $n$ observations and test on 1024 observations (for each method, we select the value of the hyperparameter that makes the  error smallest on an independent validation set). ${\rm KL}$ corresponds to our new estimators of $v$ and~$w$, while ${\rm KL}^\ast$ corresponds to using $-f^\ast(v) = 1 - e^{v}$ instead of $w$, which is expected to be better. \label{fig:bias}}
\end{figure}

\subsection{Debiased estimation} 
\label{sec:debias}
The estimate above corresponds to estimating for each $\rho$, $( \mu_p - \mu_q) ^\top M(\rho)^{-1} ( \mu_p - \mu_q)  $ by
$ ( \hat{\mu}_p - \hat{\mu}_q)^\top \hat{M}(\rho)^{-1} ( \hat{\mu}_p - \hat{\mu}_q) $,
where $\hat{M}(\rho)$ tends to $M(\rho)$ in probability, which leads to an extra term in the analysis because quadratic functions of unbiased estimators are not generally unbiased. Bias removal is a classical tool when analyzing V-statistics~\cite{lee2019u} (which our estimate is an example of when $\hat{M}(\rho)$ is replaced by $M(\rho)$).  As shown in App.~\ref{app:debiased}, it corresponds to removing an estimate of the bias, that is, subtracting the term $  \frac{1}{2}\int_0^1 \tr [ \hat{C}  \hat{M}(\rho) ^{-1}] d\nu(\rho)$,  with $\hat{C} =  ( \hat{\Sigma}_p\! - \! \hat{\mu}_p   \hat{\mu}_p^\top )/(n_p\!-\!1) 
+  (  \hat{\Sigma}_q \! -  \!\hat{\mu}_q    \hat{\mu}_q^\top )/ (n_q\!-\!1)$ and $\hat{M}(\rho) = \rho \hat{\Sigma}_p + (1\!-\!\rho) \hat{\Sigma}_q + \lambda \idm$. This term is non-negative, has a spectral expression, and leads to an improved rate of convergence when $t > d/4$ (with a smaller $\lambda$ yielding an additive term proportional to $ {1}/{n^{\frac{t}{t+d/4} }}$, which is now optimal~\cite{laurent1996efficient}, instead of $  {1}/{n^{\frac{t}{t+d/2} }}$, and an additional multiplicative term).  We compare in \myfig{bias} the two estimators, showing how debiasing significantly improves estimation.

\subsection{One-hidden-layer neural networks}
\label{sec:theoryNN}
Following~\cite{bach2017breaking}, we study the infinite-width limit of neural networks, where the functions $u_\rho$ in \eq{DpqUU} are of the form $\sum_{j=1}^m (\eta_\rho)_j ( w_j^\top x + b_j)_+$, corresponding to an $m$-dimensional feature map $\hat\varphi$ with components $( w_j^\top x + b_j)_+$. We consider a weight decay penalty with parameter $\lambda$,  and let $m$ go to infinity, assuming that we can find the empirical risk minimizer (results from~\cite{chizat2018global} qualitatively suggest that gradient ascent is then globally convergent). See proof in App.~\ref{app:proofneural}. 
\begin{proposition}
\label{prop:latent}
Assume $dp/dq$ is strictly positive, has $t$ square-integrable derivatives, and 
depends only on a projection onto a subspace of dimension $d_{\rm eff}$. When $t  > \frac{1}{2}(d_{\rm eff} +3)$, for $\lambda \propto  {1}/{\sqrt{n}}$ and $n = O(m^2)$, we have $ | F_\lambda(\hat{p}\|\hat{q}, \hat\varphi )  - D(p\|q) | = O(n^{-1/2})$ with high probability.
\end{proposition}
We see, as in supervised learning, adaptivity to linear latent variables, with a possibility of escaping the curse of dimensionality for small $d_{\rm eff}$~\cite{bach2017breaking}. Constants may be exponential in $d_{\rm eff}$ and rates can probably be improved (but our goal was to show some form of adaptivity for neural network learning).

\section{Experiments}
\label{sec:experiments}
We focus on the KL divergence, and thus the estimation of log-densities, as it encompasses most of the difficulties (see App.~\ref{app:experiments} for full experimental details\footnote{Matlab code is available at \url{https://www.di.ens.fr/~fbach/fdiv_2_0.zip}.}).

\textbf{Estimation of $D(p\|q)$.} In \myfig{bias} (left and center), we show scaling laws on a one-dimensional problem to compare regular and debiased versions and their respective rates, illustrating that our theoretical results are reasonably sharp and debiasing helps, while the quantum divergence is, as expected, similar  (but inapplicable when features are not normalized). In \myfig{bias} (right), we compare several estimators on a two-dimensional problem, showing that our new estimator outperforms the use of the Pearson divergence, direct variational formulations, and kernel density estimation. We analyze the robustness to the choice of regularization parameter in \myfig{lambdas} in App.~\ref{app:experiments}, showing that a wide range of the hyperparameter leads to close to best estimation.

\textbf{Comparison with softmax regression.} In \myfig{softmax}, we compare our closed-form estimator to the full optimization of softmax regression (an example of normalized models). In terms of training criterion (left), it matches closely and can sometimes be slightly 
higher (here better) because its induced potential is an integral of quadratic functions of linear least-squares estimators, and thus quadratic in the used features (but with a lower computational cost than linear estimators based on other losses), whereas the softmax baseline is linear in those features; a direct quadratic variational baseline would be more expensive. On test data (\myfig{softmax}, center), the performance is slightly worse, but with a smaller computational cost for our closed-form estimators.

\textbf{Estimation of mutual information.} For unnormalized models, we show in \myfig{softmax} (right) that our estimator outperforms the variational approach (we do not consider quadratic features for the variational baseline, because they lead to overly large matrices).

\textbf{Feature learning with neural networks.} In \myfig{comparisonNN} (right), we show how using neural networks allows us to scale to higher dimensions, preserve adaptivity (while fixed feature spaces, referred to as ``KL'' and ``variational'', do not), and outperform the variational baseline.  The nonparametric problem of estimating KL is facing the curse of dimensionality that
neural-network methods alleviate by learning the linear projection that matters for $dp/dq$.

More precisely, in this controlled experiment, we generate data for $q$ a distribution in $[0,1]^d$ uniformly and $\log(dp/dq)$ is chosen as a sum of cosines
for the first dimension only (exactly, $2 \cos(8\pi x) - 2 \cos(6\pi  x) + \sin(4\pi x)$). Then the data are mapped
component-wise to $\rb^{2d}$
using cosines and sines (standard coordinatewise embedding of the torus). Thus the active feature dimension is 2 and corresponds to the first two canonical axes
(since all our methods are rotation invariant, being aligned with the canonical basis does not matter). See more details in App.~\ref{app:experiments}.

\begin{figure}
\begin{center}
\hspace*{-0.5cm}\includegraphics[width=3.7cm]{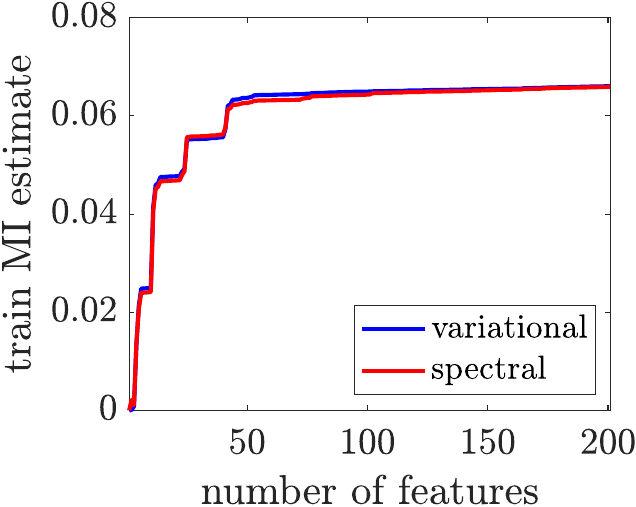}
\hspace*{0.1cm}\includegraphics[width=3.7cm]{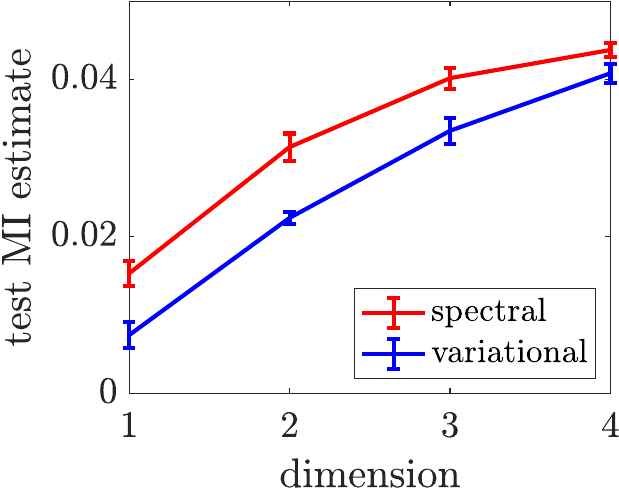}
\hspace*{0.5cm}\includegraphics[width=5cm]{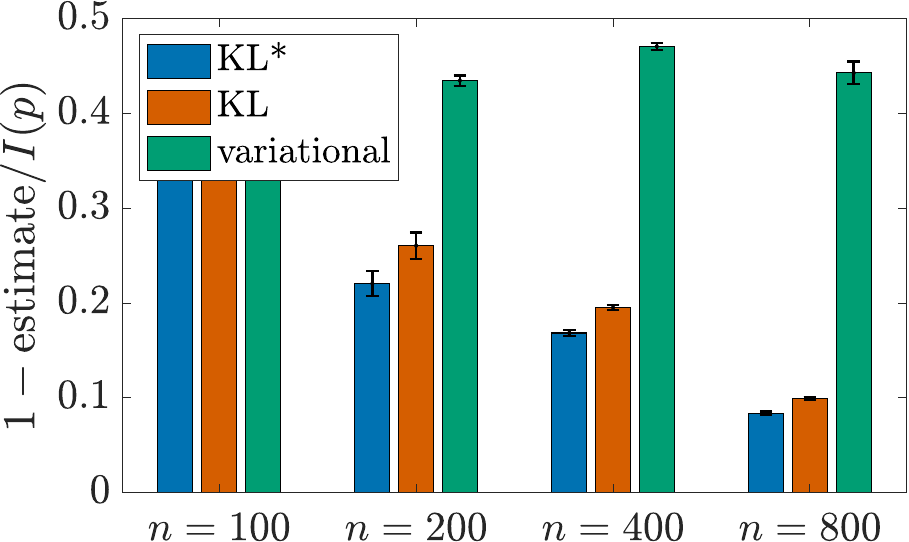}

\end{center}

\vspace*{-.35cm}

\caption{\textbf{Left and Center:} Comparison of softmax regression (equivalent here to the variational method, see App.~\ref{app:softmaxvar}) and the spectral estimator, with data in two dimensions and class-conditional distributions that are combinations of a few cosines. Left: we perform PCA on a large set of random ReLU features, and then learn with an increasing number of features, comparing softmax regression (solved with Newton's method) and our closed-form estimator, in terms of estimated mutual information (variational formulation on the training data). Center: comparison of test  MI estimates.
\textbf{Right:}  Comparison of estimators of $(v,w)$ through their normalized MI  estimates. \label{fig:softmax}
}

\vspace*{-.125cm}

\end{figure}

\begin{figure}
\begin{center}

\includegraphics[width=3.5cm]{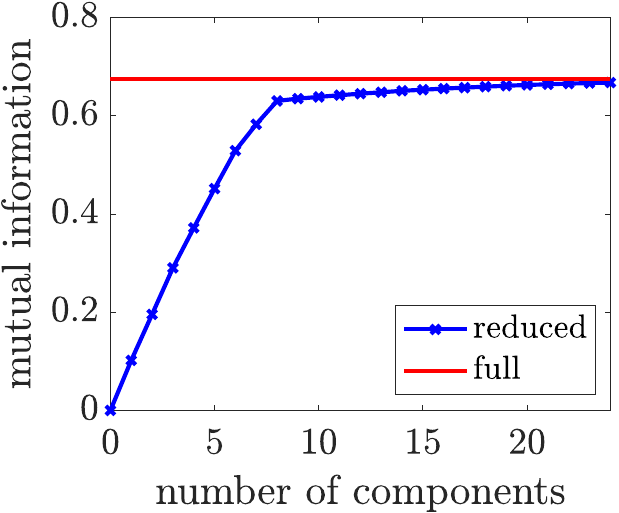}
\hspace*{0.5cm}\includegraphics[width=8.5cm]{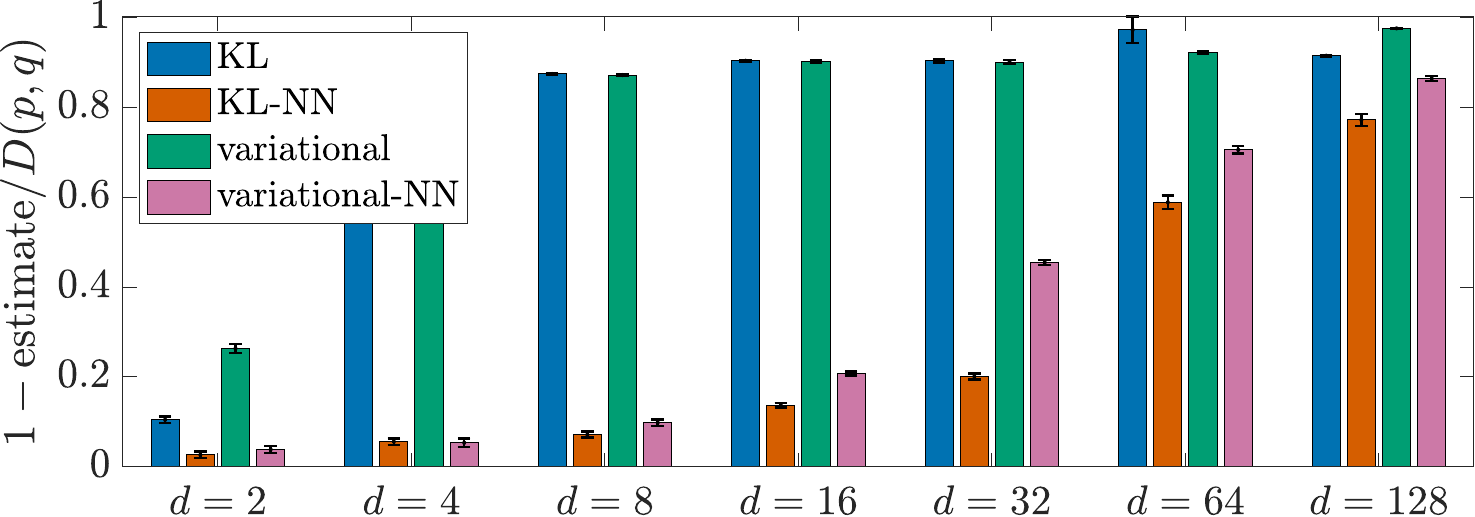}
\end{center}

\vspace*{-.35cm}

\caption{\textbf{Left:} Number of latent dimensions needed for estimating mutual information. \textbf{Right:} Comparison of estimators of $v$ and $w$ through their normalized divergence estimates (lower is better), by selecting the best hyperparameter on validation data over 4 replications ($\X = \rb^d$, with $d \in \{2,4,8,16,32,64,128\}$, $n=10 000$ training observations, and $m=100$ hidden neurons).
\label{fig:comparisonNN}}

\vspace*{-.125cm}

\end{figure}

\section{Conclusion}
In this paper, we have shown how various forms of log-density estimation can be obtained in closed form by viewing them as integrals of several least-squares estimators with an efficient spectral formulation. This leads to direct (though approximate) estimation without iterative algorithms for fixed features, and to iterative feature-learning variants when the representation is learned, in particular for classical cases such as softmax regression. In our synthetic experiments, our estimates are empirically more efficient for unnormalized models.

As shown in \mysec{featlearn}, this new cost function can be seen as a potential replacement for the last layer within a deep learning framework with added benefits compared to the usual softmax layers, with potential for 
parallelization~\cite{schubert2018numerically} or sketching~\cite{gribonval2021sketching}. Moreover, the use of the new divergences could prove useful in variational inference~\cite{wainwright2008graphical}. Indeed,
variational methods correspond to estimating the distribution $p$ rather than observing samples. This
could be done by adapting the framework of \cite{bach2022information,bach2022sum}, which was based on fixed features only (and thus
dedicated only to low dimensions), and can then be extended to feature learning. Finally, applications to sampling~\cite{ribero2026regularized} or independent component analysis~\cite{bach2002kernel,suzuki2011least} could be naturally considered.

\newpage

 \textbf{Acknowledgements.}     This work has received support from the French government, managed by the National Research Agency, under the France 2030 program with the reference ``PR[AI]RIE-PSAI'' (ANR-23-IACL-0008). We thank the reviewers for their constructive comments that helped improve the presentation of the paper.

\setlength{\bibsep}{4pt}

\bibliography{fdiv}

%%%%%%%%%%%%%%%%%%%%%%%%%%%%%%%%%%%%%%%%%%%%%%%%%%%%%%%%%%%%

\newpage
\appendix

\section{Additional examples of $f$-divergences}
We provide in Table~\ref{table:fdivcomplete} below a list of operator-convex functions $f$, their associated $f$-divergences, and measures $\nu$ on $[0,1]$. Some important cases of $f$-divergences, such as the total variation corresponding to $f(t)=|t-1|$, are not part of our framework. Note that there is a slight ambiguity in defining the Fenchel conjugate, because, when the function $f$ can be defined on $\rb$, the Fenchel conjugate depends on the chosen domain. We consider the largest possible domain (for the Pearson divergence, if we choose to restrict to $\rb_+$, then the Fenchel conjugate would be $u \mapsto  \frac{1}{2} (u+1)_+^2 - \frac{1}{2}$).
\label{app:fdiv}

\begin{table}[h]
\caption{Examples of functions $f$ (with domain that includes $\rb_+^\ast$), their Fenchel-conjugates $f^\ast$ (which typically do not have full domain), and the associated probability measure $d\nu$ on $[0,1]$. \label{table:fdivcomplete}}

\vspace*{-.5cm}

\begin{center}\hspace*{-.3cm}
\begin{tabular}{|l|l|l|l|l|}
\hline
Divergence & $f(t)$ & $f^\ast(u)$   & $d\nu(\rho) $  \\
\hline
$\alpha$-divergence, $\alpha \!\in\! [-1,2]\!\!$ & $\textcolor{white}{\Big|}  \!\!\!\!\frac{ t^\alpha - \alpha t  +  (\alpha-1)}{\alpha(\alpha-1)} \!\!\!$
& $  \!\!\frac{\!\! \textcolor{white}{\big|}- 1 + ( 1 + (\alpha-1) u )^{\alpha/(\alpha-1)}\!}{\alpha}  \!\!$  
& $\!\!\!\frac{2}{\alpha} \frac{ \sin(\alpha-1)\pi}{(\alpha-1)\pi} (1\!-\!\rho)^\alpha \rho^{1-\alpha}  d\rho \!\!  $  \\
Kullback-Leibler, $\alpha\!=\! 1$ & $ t \log t -t + 1 $  & $ e^{u}-1$ & $ 2(1-\rho) d\rho  $   \\
Reverse KL, $\alpha\!=\! 0$ & $- \log t + t - 1$ & $  - \log(1-u)$ & $2\rho d\rho  $    \\
squared Hellinger, $\alpha\!=\! \frac{1}{2}\!\!$ &  $2 (\sqrt{t}-1)^2$  & $\frac{u}{1-u/2}$
& $\frac{8}{\pi} \sqrt{ \rho(1-\rho)}  d\rho$  \\[-.1cm]
Pearson $\chi^2$, $\alpha\!=\!2$ & $\textcolor{white}{\Big|}\!\! \frac{1}{2} (t-1)^2$  & $\frac{u^2}{2} + u$
& $\delta_0(\rho)$ \\[-.15cm]
Reverse Pearson, $\alpha\!=\!-\!\,1\!\!$ & $\textcolor{white}{\Big|} \!\!\frac{1}{2t } (t-1)^2$ & $1- \sqrt{1-2u}$ 
&  $\delta_1(\rho)$ \\[-.1cm]
Le Cam & $ \frac{(t-1)^2}{t+1} $ & $4-u - 4\sqrt{1-u}$ &    $\delta_{1/2} (\rho) $  \\[-.1cm]
 Jensen-Shannon & $ \textcolor{white}{\Big|}\!\!\! \! 2 t \log \frac{2t}{t+1} \!+\!  2\log \frac{2}{t+1}\!\!\!$ & $-2 \log(2-e^{u/2})$    & $(2 - 2|1-2\rho|) d\rho$\\
\hline
\end{tabular}
\end{center}
\end{table}

 \section{Comparisons to related work}

  \subsection{Comparison with prior quantum-divergence work}
\label{app:comparison}
The divergence proposed in this paper is inspired by \cite{bach2022sum}, with both similarities and differences. 

\textbf{Properties of the quantum divergence.}
To relate the two divergences, we present various related results from~\cite{bach2022sum} for the quantum divergence, which we denote by $Q(p\|q,\varphi)$.
\BIT
\item 
\textbf{Integral formula:} it is equal to 
\[ Q(p\|q,\varphi) = \frac{1}{2}
\int_0^1 \tr \big[ ( \Sigma_p - \Sigma_q) (\rho \Sigma_p +(1-\rho)\Sigma_q  )^{-1} ( \Sigma_p - \Sigma_q) 
\big]
d\nu(\rho),\]
which leads to similar linear algebra, but without using first-order moments $\mu_p$ and $\mu_q$.

\item \textbf{Spectral formula:} 
We have:
\[ Q(p\|q,\varphi) =  \tr \big[
 \Sigma_q^{1/2}  f ( \Sigma_q^{-1/2} \Sigma_p \Sigma_q^{-1/2} )  \Sigma_q^{1/2}  \big], \]
 which is to be related to \eq{Fspec}.

\item \textbf{Variational formulation:} The key difference with our framework is the need for feature normalization to obtain a lower bound on $D(p\|q)$, that is, $\varphi(x)^\top \varphi(x) = 1$ for all $x \in \X$, which leads to using only quadratic forms for $v(x)$ and $w(x)$, and with the constraint $\lambda M + N - f(\lambda) \idm \preccurlyeq 0$ for all $\lambda >0$, instead of $\lambda \bimat{M}{c}{c^\top}{0}  + 
 \bimat{N}{-c}{-c^\top}{0} + \bimat{0}{0}{0}{-f(\lambda) } \preccurlyeq 0$ in \eq{DDD}.

\item \textbf{Properties:} They are mostly the same, except, as already mentioned, the need for normalization, which makes it hard to maximize the lower bound because it must be done with the constraint of feature normalization. In \cite{bach2022sum}, this is done using sum-of-squares techniques, which are not easy to extend to high dimensions.
\EIT

\textbf{Differences.}
Our construction is closely related to the covariance-operator, or
quantum-divergence, approach of
\cite{bach2022information,bach2022sum}. These works already provide spectral
formulas and kernelized or sum-of-squares relaxations for
information-theoretic quantities, including entropy, relative entropy,
mutual information, and more general \(f\)-divergences. The novelty of
the present paper is therefore not the use of spectral information
quantities per se, but a different moment relaxation tailored to
unnormalized feature maps and explicit density-ratio estimation.

In the covariance-operator formulation, lower bounds on classical
divergences are obtained from non-centered second-moment operators, and
the associated variational lower bounds typically require a pointwise
feature normalization, for instance
$
    \varphi(x)^\top \varphi(x)=1 .
$

Such constraints are natural in that framework, but they complicate
feature learning. Our lifted relaxation includes first-moment terms in
addition to second moments, leading to ``quadratic+linear'' potentials in lifted
features and to a valid lower bound without imposing pointwise feature
normalization. This makes linear feature learning, stochastic updates, and
neural-network parameterizations more direct.

Thus, our contribution should be seen as complementary to
\cite{bach2022information,bach2022sum}. The kernel convergence analysis is
largely parallel to the covariance-operator setting, while the main
practical difference is the combination of unnormalized-feature moment
relaxations, closed-form spectral formulas for divergence and potential
estimation, and feature-learning algorithms that exploit the resulting
unconstrained parameterization.

  \subsection{Comparison to direct density-ratio estimation}
Closely related are direct density-ratio estimators, which estimate
\(u = dp/dq\) without separately estimating the densities \(p\) and \(q\).
The Kullback-Leibler importance estimation procedure (KLIEP)
\citep{sugiyama2007direct} fits a nonnegative ratio model
\(u\) by maximizing an empirical log-likelihood,
\[
     {\mathbb{E}}_{\hat p}[\log u(x)],
\]
subject to the normalization constraint
\[
     {\mathbb{E}}_{\hat q}[u(x)] = 1,
\]
thereby targeting a KL discrepancy between \(p\) and the tilted reference
measure \(u q\). Least-squares importance fitting (LSIF)
\citep{kanamori2009least} instead estimates \(u\) by minimizing a
squared-loss criterion, leading to a convex quadratic program and, in
unconstrained variants, closed-form ridge-type estimators under fixed
basis functions. Our approach is complementary: rather than fitting a
nonnegative, normalized ratio model directly, we construct spectral lower
bounds for operator-convex \(f\)-divergences from first- and second-order
feature moments. For fixed features, this yields closed-form eigenvalue
estimators for divergences and associated potentials, replacing
constrained ratio fitting by a moment-based spectral relaxation.

\subsection{Link between binary classification and $f$-divergence}
\label{app:binary}
Following~\cite{JMLR:v12:reid11a}, for any $\sigma \in (0,1)$, for the binary classification problem of $y\in \{0,1\}$ given $x$, with class-conditional distributions $p(x|y=1) = p(x)$ and $p(x|y=0) = q(x)$, marginal probabilities $\P(y=1) = \sigma$ and $\P(y=0) = 1-\sigma$, we can consider the loss function
\[
\ell(y,h(x) )= a(h(x)) - (2y-1) h(x),
\]
for a convex function $a$ such that $a'(0) = 2\sigma - 1$, and the closure of the domain of $a^\ast$ is $[-1,1]$. We then have 
\BEAS
& & - a^\ast(2\sigma-1) - \inf_{h: \X \to \rb} \E[ \ell(y,h(x)) ] \\
& = &   \int_\X ( \sigma p(x) + (1-\sigma) q(x) ) \Big[ a^\ast \Big(
\frac{\sigma p(x) - (1-\sigma) q(x)}{\sigma p(x) + (1-\sigma) q(x)}
\Big) - a^\ast(2\sigma-1) \Big] dx,
\EEAS
which is an $f$-divergence for
\[
f(t) = ( \sigma t +  1-\sigma ) \Big[ a^\ast \Big(
\frac{\sigma t- (1-\sigma)  }{\sigma t + 1-\sigma  }
\Big) - a^\ast(2\sigma-1) \Big],
\]
or equivalently
\[
a^\ast(u) - a^\ast(2\sigma-1) = f\Big(
\frac{1-\sigma}{\sigma} \frac{1+u}{1-u} 
\Big) \frac{1-u}{2} \frac{1}{1-\sigma}.
\]
This allows us to go from the binary classification formulation to the $f$-divergence, yielding a variational formulation in terms of binary classification
\[
D(p\|q) = \sup_{h: \X \to \rb} - \sigma \int_\X [ a(h(x))-h(x) ] dp(x) 
- (1-\sigma) \int_\X [ a(h(x))+h(x) ] dq(x) - a^\ast(2\sigma-1).
\]
There is thus a link between computing $f$-divergences and binary classification~\cite{JMLR:v12:reid11a}. This link is explicit for the Jensen-Shannon divergence and logistic regression~\cite{nowozin2016f}, by writing 
\[
  f^*(v)=-2\log(2-e^{v/2}).
\]
Setting
\[
  v_g(x)=-2\log\frac{1+e^{g(x)}}{2}
        =2\log\frac{2}{1+e^{g(x)}}
\]
gives
\[
  -f^*(v_g(x))=2\log\frac{2}{1+e^{-g(x)}}.
\]
Therefore, the variational formulation becomes (with a constant factor difference with the usual convention for the Jensen-Shannon divergence)
\[
D_{\rm JS}(p\|q)
=\sup_{g : \X \to \rb} \Big\{
2\mathbb E_p\log\frac{2}{1+e^{g(x)}}
+2\mathbb E_q\log\frac{2}{1+e^{-g(x)}}
\Big\},
\]
which is the equal-prior logistic classification objective, up to a sign convention for $g$, additive constants, and an overall positive multiplicative factor.

 Unfortunately, there is no such simple behavior for the KL divergence, for which the same calculation formally defines a binary-classification loss, but the resulting loss is not the standard logistic loss and is less useful computationally. Indeed, applying the relation above with
\(f(t)=t\log t-t+1\) gives, on \(u\in(-1,1)\),
\[
  a^*_{\rm KL}(u)-a^*_{\rm KL}(2\sigma-1)
  =\frac{1-u}{2(1-\sigma)}
    \left[\tau(u)\log\tau(u)-\tau(u)+1\right],
  \qquad
  \tau(u)=\frac{1-\sigma}{\sigma}\frac{1+u}{1-u}.
\]
For example, when \(\sigma=1/2\), this becomes
\[
  a^*_{\rm KL}(u)=(1+u)\log\frac{1+u}{1-u}-2u
\]
(up to an additive constant), and the primal loss involves the inverse of the function
\(u\mapsto \log\frac{1+u}{1-u}+\frac{2u}{1-u}\), which is not easy to handle computationally, and does not lead to a simple closed-form expression like our estimator.

Recent work \cite{zellingerbinary,zellinger2023adaptive} studies density-ratio estimation through binary classification. In particular, \cite{menon2016linking,zellingerbinary} characterize binary losses whose excess classification risk controls a prescribed Bregman error for the density ratio, while~\cite{zellinger2023adaptive} gives adaptive RKHS rates for the corresponding regularized estimators. Our approach is complementary: instead of choosing a classification loss and solving the resulting empirical-risk problem, we integrate relative-Pearson least-squares objectives and obtain a closed-form spectral estimator from first and second moments. Thus, for the KL divergence, \cite{zellingerbinary,zellinger2023adaptive} provide a principled classifier-based route to density-ratio estimation, whereas our contribution is a moment-based spectral route that avoids the nonstandard KL-induced binary loss above.

\section{Proof of main properties from \mysec{prop}}
\label{app:proof_main}

 \subsection{Sufficient condition for finiteness}
   The definition in \eq{defF} through an integral as 
  \[
 F(p\|q ,\varphi)  =   \frac{1}{2}
\int_0^1 ( \mu_p - \mu_q)^\top (\rho \Sigma_p +(1-\rho)\Sigma_q)^{-1} ( \mu_p - \mu_q) 
d\nu(\rho)\]
can be made less ambiguous, in particular, when second-order moments $\Sigma_p$ and $\Sigma_q$ are not invertible. 
  Even with non-invertible (noncentered) covariance matrices/operators (which is always the case in infinite dimension), we can define  $ F(p\|q ,\varphi)  $ instead as
   \BEQ
\label{eq:official}    \inf_{t: (0,1) \to \rb} \frac{1}{2}\int_0^1 t(\rho) d\nu(\rho), \mbox{ such that } \forall \rho \in (0,1), \ \bimat{\rho \Sigma_p + (1-\rho)\Sigma_q}{(\mu_p-\mu_q)}{(\mu_p-\mu_q)^\top}{t(\rho)} \succcurlyeq  0
,\EEQ
which is finite as soon as $\mu_p - \mu_q$ belongs to the range of $\rho \Sigma_p + (1-\rho)\Sigma_q$, which is the case for all $\rho \in (0,1)$ since $\Sigma_p \succcurlyeq \mu_p  \mu_p^\top$ and 
 $\Sigma_q \succcurlyeq \mu_q  \mu_q^\top$. Thus, $F(p\|q ,\varphi) $ is the integral over $(0,1)$ of a finite function (which is continuous in $\rho$). It may, however, not be integrable on $[0,1]$, and it is then considered infinite. Since by construction, we always have $F(p\|q ,\varphi) \leqslant D(p\|q)$, it is always finite when $D(p\|q)$ is, and the condition that $dp/dq$ exists everywhere and lies in an interval $[\alpha,1/\alpha]$, for $\alpha \leqslant 1$, is sufficient. Then $D(p\|q)$ is well-defined and bounded by $\max\{ f(\alpha), f(1/\alpha)\}$.

Note that when \(\nu\) has atoms at the endpoints, as for the Pearson and reverse-Pearson cases, the constraint in \eq{official} is understood in the closed sense for all \(\rho\in\operatorname{supp}(\nu)\subset[0,1]\), with the cases \(\rho=0\) and \(\rho=1\) obtained by replacing \(\rho\Sigma_p+(1-\rho)\Sigma_q\) by \(\Sigma_q\) and \(\Sigma_p\), respectively.

 \subsection{Convexity}
 Each term $( \mu_p - \mu_q)^\top (\rho \Sigma_p +(1-\rho)\Sigma_q)^{-1} ( \mu_p - \mu_q) 
$ is jointly convex in $\mu_p, \mu_q, \Sigma_p,\Sigma_q$~\cite{Boyd2004Convex}, and since these moments are linear in $p$ and $q$, $F(p\|q ,\varphi)$ is jointly convex in $(p,q)$.

  \subsection{Monotonicity}
  The variational formulation for each $\rho$ through the optimization of a function $u$ that is linear in the feature map shows that if we augment the feature map, the value can only be higher.
  
  \subsection{Linear invariance}
 
  If $\varphi(x)$ is  replaced by $V\varphi(x)$ for an \emph{invertible} map $V$, 
  then $\mu_p,\mu_q$ are replaced with $V\mu_p$ and $V\mu_q$, and $\Sigma_p,\Sigma_q$ are replaced with $V\Sigma_p V^\top $ and  $V\Sigma_q V^\top $. The term $( \mu_p - \mu_q)^\top (\rho \Sigma_p +(1-\rho)\Sigma_q)^{-1} ( \mu_p - \mu_q) 
$ then remains the same. If $V$ is only injective, then the formulation in \eq{official} can be used. If $V$ is surjective, the monotonicity result above applies. For the regularized version $F_\lambda$, the invariance only applies to orthogonal matrices $V$.

  \subsection{Concavity in the kernel}
  We only need to show it for $f(t) = \frac{1}{2} \frac{(t-1)^2}{\rho t + 1-\rho} $; we follow the same proof technique as~\cite[Sec.~7.2]{bach2022information} and consider two feature maps $\varphi_1: \X \to \rb^{m_1} $ and $\varphi_2: \X \to \rb^{m_2} $, as well as the feature map $\varphi = 	( \sqrt{\eta_1} \varphi_1^\top, \sqrt{\eta_2} \varphi_2^\top)^\top \in \rb^{m_1 + m_2}$, with $\eta_1,\eta_2 > 0$ summing to one, corresponding to a kernel
  \[
  \varphi^\top(x) \varphi(y) =    \eta_1\varphi_1^\top(x) \varphi_1(y) + \eta_2  \varphi_2^\top(x) \varphi_2(y),
  \]
  which is the convex combination of the kernels for $\varphi_1$ and $\varphi_2$. We will show concavity for the regularized version, and thus consider
  \[
\!\! \!\!\!\!{  \bivec{\sqrt{\eta_1} \delta_1}{\sqrt{\eta_2} \delta_2}
  \!} ^\top \!\!\bimat{ \eta_1 M_{11} + \lambda \idm}{\sqrt{\eta_1\eta_2}M_{12}}
  {\sqrt{\eta_1\eta_2}M_{21}}{\eta_2 M_{22} + \lambda \idm}^{-1} \!\! \bivec{\sqrt{\eta_1} \delta_1}{\sqrt{\eta_2} \delta_2}
  =  {\bivec{  \delta_1}{  \delta_2}\!\!}
  ^\top \!\! \bimat{  M_{11} \!+\! \lambda \eta_1^{-1}  \idm}{ M_{12}}
  { M_{21}}{  M_{22} \!+\! \lambda \eta_2^{-1}  \idm}^{-1}\!\!  \bivec{ \delta_1}{ \delta_2}
  \]
  for $\delta_i = (\E_p - \E_q) \varphi_i$, $M_{ij} = \E_{\rho p + (1-\rho) q} \varphi_i \varphi_j^\top$ (for $i,j \in \{1,2\}$). It is of the form
  \[
  v^\top ( NN^\top +  D^{-1})^{-1} v,
  \]
  with $D$ block-diagonal with blocks $\lambda^{-1} \eta_1  \idm $ and $\lambda^{-1} \eta_2\idm$. By the matrix inversion lemma, it is of the form
  \[
  v^\top D v - v^\top D N ( \idm + N^\top D N)^{-1} N^\top D v ,
  \]
  which is concave in $D$~\cite{ando1979concavity}, hence the result.

  \subsection{Link with $f$-divergence (proof of Prop.~\ref{prop:universality})}
  We only need to prove the result for $f(t) = \frac{1}{2} \frac{(t-1)^2}{\rho t + 1-\rho} $, and then proceed by integration. The lower-bound result is a direct consequence of the variational formulation in \eq{DpqUU} in \mysec{introduction} with $u$ being linear in $\varphi$.
  
  The tightness condition is simply a consequence of the fact that the maximizer of \eq{DpqUU} is exactly the function $ x \mapsto \big( \frac{dp}{dq}(x) - 1 \big) / \big( \rho \frac{dp}{dq}(x) + 1-\rho\big) $.

  \subsection{Gaussian case}
   When $\X = \rb^d$ and $\varphi(x) = x$, then our estimate for $D(p\|q)$ depends only on the first two moments and has to be strictly lower than the value of $D(p\|q)$ for Gaussians (except in degenerate cases). In other words, as opposed to~\cite[Sec.~3.4]{bach2002kernel}, our new divergence does not rely on any Gaussian analogy or assumption, which can never be true unless $\varphi$ has special structure.

\section{Proof of variational formulation (Prop.~\ref{prop:closedform})}
\label{app:var}
Throughout this section, we assume that $\H = \rb^m$ is finite-dimensional and $\Sigma_q$ is invertible. We denote by $\mathcal{A}$ the set of $m\times m$ symmetric matrices.
We first define a new divergence $G$ and show by convex duality that it is equal to $F$.
 
\begin{proposition}
\label{prop:G}
Assume $\Sigma_q$ invertible. Define:
\BEAS
G(\mu_p,\Sigma_p\|\mu_q,\Sigma_q)
& = & 
\sup_{M,N \in \mathcal{A}, \ c \in \H}
\tr[ M\Sigma_p] + 2  c^\top(\mu_p-\mu_q) + \tr[ N\Sigma_q]   \\
& & \mbox{ such that } \forall \lambda >0, \bimat{\lambda M+N}{(\lambda-1)c}{(\lambda-1)c ^\top}{-f(\lambda)} \preccurlyeq 0.
\EEAS

\BIT
\item If $f$ is twice continuously differentiable, then $G(\mu_p,\Sigma_p\|\mu_q,\Sigma_q) \leqslant F(p\|q,\varphi)$.  \item If $f$ is operator-convex, then $G(\mu_p,\Sigma_p\|\mu_q,\Sigma_q)\geqslant  F(p\|q,\varphi)$, and thus the two values are equal.
\EIT
\end{proposition}

\begin{proof} 
By construction, $G$ is convex in $\Sigma_p, \Sigma_q, \mu_p - \mu_q$, and any maximizers $M,N,2c$ are respective (sub)gradients with respect to these moments.

Following~\cite{bach2022sum}, the dual problem to the one defining $G(\mu_p,\Sigma_p\|\mu_q,\Sigma_q)$ is
\BEAS
 \!\!\!\!\!\!\! \!\!\!\!\!\!\!& & \inf_{\Lambda, \xi, \zeta} \int_0^{+\infty}\!\!\! f(\lambda) d\zeta(\lambda) \mbox{ such that } \Sigma_p \!=\! \int_0^{+\infty} \!\!\!\! \lambda d\Lambda(\lambda), \ 
\Sigma_q \!=\! \int_0^{+\infty} \!\!\!\! d\Lambda(\lambda), \ 
\mu_p-\mu_q \!=\! \int_0^{+\infty} \!\! \! ( \lambda -1)d\xi(\lambda) \\
  \!\!\!\!\!\!\!\!\!\!\!\!\!\!& & \hspace*{4.4cm} \bimat{\Lambda}{\xi}{\xi^\top}{\zeta} \mbox{ is a PSD-valued measure}.  
\EEAS
Let $\Sigma_q^{-1/2} \Sigma_p \Sigma_q^{-1/2} = \sum_{i =1}^m \lambda_i u_i u_i^\top$ be the eigenvalue decomposition of  $\Sigma_q^{-1/2} \Sigma_p \Sigma_q^{-1/2}$, thus with  $v_i = \Sigma_q^{-1/2} u_i$ defining the generalized eigenvalue decomposition of the pair $(\Sigma_p,\Sigma_q)$. Define
\[
\bimat{\Lambda}{\xi}{\xi^\top}{\zeta}
= \sum_{i =1}^m 
\bimat{\Sigma_q v_i v_i^\top \Sigma_q }{ \Sigma_q v_i v_i^\top  (\mu_p - \mu_q)  /(\lambda_i-1)}
{  (\mu_p - \mu_q)^\top v_i v_i^\top \Sigma_q  /(\lambda_i-1)}{( v_i^\top  (\mu_p-\mu_q ) /(\lambda_i-1))^2}
\delta_{\lambda_i} ,
\]
which is dual-feasible since
$\sum_{i=1}^m v_i v_i^\top = \Sigma_q^{-1}$, $\sum_{i=1}^m \lambda_i v_i v_i^\top = \Sigma_q^{-1} \Sigma_p \Sigma_q^{-1}$. We then have by construction:
\BEAS
 \int_0^{+\infty} \!\! f(\lambda) d\zeta(\lambda)
 & = & F(p\|q,\varphi).
 \EEAS
 Thus, $F(p\|q,\varphi)$ is a dual value, leading to the first result (note that the case $\lambda_i=1$ can be treated by a limiting argument because we have assumed that when $t$ tends to $1$, $f(t)/(t-1)^2$ tends to $ f''(1)/2$).

In order to prove the second result, it suffices to find feasible $M,N,c$ so that $\tr[ M\Sigma_p] + 2  c^\top(\mu_p-\mu_q) + \tr[ N\Sigma_q]  $ is equal to $F(p\|q,\varphi)$. Following the proof of the operator Jensen's inequality~\cite{lesniewski1999monotone},
  it suffices to do so for $f(t) = \frac{1}{2} \frac{ (t-1)^2}{\rho t + 1-\rho}$ for $\rho \in [0,1]$ (as all properly rescaled operator-convex functions are positive linear combinations of these). We can then consider candidates $M,N,2c$ from derivatives of
$
\frac{1}{2} (\mu_p-\mu_q)^\top ( \rho \Sigma_p + (1-\rho)\Sigma_q )^{-1} (\mu_p-\mu_q)
$ with respect to $\Sigma_p,\Sigma_q,\mu_p-\mu_q$, that is,
\BEAS
c_\rho & \!\!\!\!= \!\!\!\!&\frac{1}{2} ( \rho \Sigma_p + (1-\rho)\Sigma_q )^{-1} (\mu_p-\mu_q) \\
M_\rho & \!\!\!\!= \!\!\!\!&-\frac{1}{2} \rho  ( \rho \Sigma_p + (1-\rho)\Sigma_q )^{-1} (\mu_p-\mu_q)   (\mu_p-\mu_q)^\top
 ( \rho \Sigma_p + (1-\rho)\Sigma_q )^{-1} =- 2\rho c_\rho c_\rho^\top \\
N_\rho &\!\!\!\! = \!\!\!\!& -\frac{1}{2} (1-\rho)  ( \rho \Sigma_p + (1-\rho)\Sigma_q )^{-1} (\mu_p-\mu_q)   (\mu_p-\mu_q)^\top
 ( \rho \Sigma_p + (1-\rho)\Sigma_q )^{-1} \\
 &  = & - 2(1-\rho)  c_\rho  c_\rho^\top.
\EEAS
For these, for any $\lambda>0$, we have:
\[
\bimat{\lambda M_\rho+N_\rho}{(\lambda-1)c_\rho}{(\lambda-1)c_\rho ^\top}{-f(\lambda)} 
=  
\bimat{- 2 ( \lambda \rho + 1-\rho)  c_\rho  c_\rho^\top}{(\lambda-1)c_\rho}{(\lambda-1)c_\rho ^\top}{- \frac{1}{2} \frac{(\lambda-1)^2}{\rho \lambda + 1-\rho}} ,
\]
which is indeed negative semi-definite.
 These candidates are thus feasible for the optimization problem with the value $F(p\|q,\varphi)$ for that particular $f(t) = \frac{1}{2} \frac{ (t-1)^2}{\rho t + 1-\rho}$. For a generic $f$ that is the integral over the probability distribution $\nu$ of that particular $f$, we can take
 \BEQ
 \label{eq:candrho}
  M = \int_0^1 M_\rho d\nu(\rho), \ \ N = \int_0^1 N_\rho d\nu(\rho), \ \ c = \int_0^1 c_\rho d\nu(\rho),
  \EEQ
   which are feasible (by convexity) for the problem defining $G$. These candidates lead to  $\tr[ M\Sigma_p] + 2   c^\top(\mu_p-\mu_q) + \tr[ N\Sigma_q]  $ that is equal to $F(p\|q,\varphi),$ hence the second result.
 \end{proof}

\textbf{Link with variational formulation of $D(p\|q)$.}
We can rewrite $G$ as
\BEAS
\!\!\!\!\!\!\!\! G(\mu_p,\Sigma_p \| \mu_q, \Sigma_q)
& \!\!\!\!= \!\!\!\!&  \sup_{ v,w : \X \to \rb } \!
  \int_\X\! v(x) dp(x)  + \!\!\int_\X\! w(x) dq(x)  \\
  & & \hspace*{3cm} \mbox{ such that }  \ \forall x \in \X, \ \forall \lambda \geqslant 0, \  \lambda v(x)  + w(x) \leqslant f(\lambda) ,
\EEAS
 with the additional constraint that $v(x) = \bivec{\varphi(x)}{1} ^\top\bimat{M}{c}{c^\top}{0} \bivec{\varphi(x)}{1} $ and $w(x) = \bivec{\varphi(x)}{1} ^\top\bimat{N}{-c}{-c^\top}{0} \bivec{\varphi(x)}{1} $, with
 \BEQ
 \label{eq:DDD}
\lambda \bimat{M}{c}{c^\top}{0}  + 
 \bimat{N}{-c}{-c^\top}{0} + \bimat{0}{0}{0}{-f(\lambda) } \preccurlyeq 0,
 \EEQ
 which implies that $\lambda v(x) + w(x) \leqslant f(\lambda)$ for all $x \in \X$, which is exactly the constraint described in \mysec{introduction}.
 
 \textbf{Proof  of Prop.~\ref{prop:closedform}.} Given Prop.~\ref{prop:G} and the developments above regarding expressions of potentials $v$ and $w$, we only need to find closed-form expressions for $M$, $N$, and $c$ above. This can be obtained by looking at gradients of $F$ or $G$ with respect to $\Sigma_p, \Sigma_q$, and $\mu_p - \mu_q$. This can be done in two ways that lead to the same result.
 
 \BIT
 \item \textbf{Integration with respect to $\rho$}: we consider the candidates defined in \eq{candrho}, and notice that we have 
 \BEAS c_\rho 
 & = & 
 \frac{1}{2} ( \rho \Sigma_p + (1-\rho)\Sigma_q )^{-1} (\mu_p-\mu_q)
\\
& = &  \frac{1}{2} \Sigma_q^{-1/2} ( \rho \Sigma_q^{-1/2}\Sigma_p \Sigma_q^{-1/2} + (1-\rho )\idm  )^{-1} \Sigma_q^{-1/2} (\mu_p-\mu_q) \\
& = &  \frac{1}{2} \sum_{i=1}^m ( \rho \lambda_i + 1-\rho   )^{-1} v_i v_i^\top (\mu_p-\mu_q),
  \EEAS
  leading to, using $f(t) =\frac{1}{2} \int_0^1 \frac{(t-1)^2}{\rho t + 1-\rho} d\nu(\rho)$, $c = \int_0^1 c_\rho d\nu(\rho) = 
\sum_{i=1}^m \frac{f(\lambda_i)}{(\lambda_i-1)^2} v_i v_i^\top (\mu_p-\mu_q)$. Similarly, we have:
\BEAS
\!M_\rho &\!\! \!\!=\!\!\! \!& - 2 \rho c_\rho c_\rho^\top 
= -    \frac{\rho}{2} \sum_{i,j=1}^m ( \rho \lambda_i + 1-\rho   )^{-1} ( \rho \lambda_j + 1-\rho   )^{-1} v_i v_i^\top (\mu_p-\mu_q)(\mu_p-\mu_q)^\top v_j v_j^\top.
\EEAS
We thus need to compute the integral
$\displaystyle
\int_0^1
\frac{\rho}{2}  ( \rho \lambda_i + 1-\rho   )^{-1} ( \rho \lambda_j + 1-\rho   )^{-1} d\nu(\rho)
$, and notice that for $\lambda_i \neq \lambda_j$ (with the equality case obtained by a limiting argument):
\BEAS
&&\frac{f(\lambda_i)/(\lambda_i-1)^2-f(\lambda_j)/(\lambda_j-1)^2}{\lambda_i - \lambda_j} \\
& = & 
\frac{1}{2(\lambda_i - \lambda_j) } \int_0^1
\Big( \frac{1}{\rho (\lambda_i-1) + 1}
-  \frac{1}{\rho (\lambda_j - 1) + 1}
\Big)d\nu(\rho)\\
& = & 
\frac{1}{2(\lambda_i - \lambda_j) } \int_0^1
\frac{\rho(\lambda_j - \lambda_i)}{ (\rho (\lambda_i-1) + 1)(\rho (\lambda_j - 1) + 1)}
d\nu(\rho) \\
& = &  -\int_0^1
\frac{\rho}{2}  ( \rho \lambda_i + 1-\rho   )^{-1} ( \rho \lambda_j + 1-\rho   )^{-1} d\nu(\rho),
\EEAS
  which shows the expression for $M$. A similar argument leads to the expression for $N$.
  
  \item \textbf{Derivatives of spectral functions}: We can first express $F(p\|q,\varphi)$ as a spectral function, as
  \BEAS
  F(p\|q ,\varphi) &  \!\!\!= \!\!\!&    \frac{1}{2}
\int_0^1 ( \mu_p - \mu_q)^\top (\rho \Sigma_p +(1-\rho)\Sigma_q  )^{-1} ( \mu_p - \mu_q) 
d\nu(\rho)  \\
& \!\!\! = \!\!\!&    \frac{1}{2}
\int_0^1 ( \mu_p - \mu_q)^\top \Sigma_q^{-1/2} (\rho \Sigma_q^{-1/2} \Sigma_p \Sigma_q^{-1/2} +(1-\rho) \idm)^{-1} \Sigma_q^{-1/2}  ( \mu_p - \mu_q) 
d\nu(\rho)  
\\
& \!\!\! = \!\!\!&    \frac{1}{2}
\int_0^1 ( \mu_p - \mu_q)^\top  
\sum_{i=1}^m  (\rho\lambda_i + 1-\rho )^{-1} v_i v_i^\top  ( \mu_p - \mu_q) 
d\nu(\rho)  \\
& \!\!\! = \!\!\!&  \sum_{i=1}^m \frac{f(\lambda_i)}{(\lambda_i-1)^2}  ( v_i^\top  ( \mu_p - \mu_q) )^2,
\EEAS
which is exactly  \eq{defFspec}. Moreover,  we have
 \BEQ
 \label{eq:Fspec}
 \!\!\! F(p\|q ,\varphi) 
=
  ( \mu_p - \mu_q)^\top \Sigma_q^{-1/2}  f ( \Sigma_q^{-1/2} \Sigma_p \Sigma_q^{-1/2} )  ( \Sigma_q^{-1/2} \Sigma_p \Sigma_q^{-1/2} - \idm)^{-2} \Sigma_q^{-1/2}   ( \mu_p - \mu_q) 
  ,
  \EEQ
  which is obtained from a spectral function of $\Sigma_q^{-1/2} \Sigma_p \Sigma_q^{-1/2}$. We can then use derivatives of spectral functions~\cite{lewis1996derivatives} to obtain the same result.
  
  \EIT

  \textbf{Additional observation.}
  The formulation in Prop.~\ref{prop:closedform} is not true for infinite-dimensional $\H$, as the pair $(\Sigma_p,\Sigma_q)$ typically does not have a discrete generalized spectrum (see App.~\ref{app:noeig}). This is not a problem, however, in practice, as, for the natural estimator defined in \mysec{estimate}, the representer theorem allows a representation in a finite vector space, and thus closed-form spectral estimates.

The results from \mysec{theory} could be extended to assess the performance of our functions $v(x)$ and $w(x)$ at test time (for the test variational formulation).

\section{Expression in terms of kernel matrix}
\label{app:kerdir}
\label{sec:effcomp}
 Writing $\Psi = {  \Phi_p \choose   \Phi_q} \in \rb^{(n_p+n_q)\times m}$, for  empirical feature matrices $\Phi_p \in \rb^{n_p \times m}$ and
 $\Phi_q \in \rb^{n_q \times m}$, 
 based on a possibly thin Cholesky factorization
 of
 \[
 K = \bimat{ K_{pp}}{  K_{pq}}{  K_{qp}}{ K_{qq}}
 = \Psi \Psi^\top \in \rb^{ (n_p + n_q) \times (n_p + n_q) },
 \]
 we
 have, up to a subspace projection, 
 \[
 \hat{\Sigma}_p = \Psi^\top \bimat{\frac{1}{n_p} \idm}{0}{0}{0} \Psi
 \ \ \  \hat{\Sigma}_q = \Psi^\top \bimat{0}{0}{0}{\frac{1}{n_q}\idm} \Psi, \ \ \ \hat{\mu}_p =  \Psi^\top { \frac{1}{ {n_p}} 1_{n_p} \choose 0}, \ \ \  \mbox{ and } \hat{\mu}_q = \Psi^\top { 0 \choose \frac{1}{ {n_q}} 1_{n_q} }.
 \]
We then consider the estimator
$
F_\lambda(\hat{p}\| \hat{q},\varphi),
$ using the spectral formulation in \eq{defFspec},
which can be computed with complexity $O(m^3)$. If only kernel matrices are given, we can compute $\Psi$ as a Cholesky decomposition of $K$, leading to a complexity of $O(n^3)$ and $m\leqslant n$.

We now provide an expression for the estimate of our divergence $F$ based solely on kernel matrices. This could be useful for computing partial derivatives. 

\textbf{Special case: single $\rho \in (0,1)$.}
For the particular case $f(t) = \frac{1}{2} \frac{(t-1)^2}{\rho t + 1-\rho}$,
we can express the estimators using the kernel matrix $\Psi \Psi^\top$ directly, as:  
  \BEAS
 F_\lambda(\hat{p}\|\hat{q},\varphi) 
 & \!\!\!= \!\!\!& \frac{1}{2} ( \hat\mu_p - \hat\mu_q)^\top   (  \rho \hat{\Sigma}_p +(1-\rho) \hat{\Sigma}_q + \lambda \idm ) ^{-1} ( \hat\mu_p - \hat\mu_q)  \\
 &\!\!\! = \!\!\!&  \frac{1}{2} { \frac{1}{ {n_p}} 1_{n_p} \choose -\frac{1}{ {n_q}} 1_{n_q}}^\top \Psi
 \bigg(  \Psi^\top \bimat{\frac{\rho}{n_p}\idm}{0}{0}{\frac{1-\rho}{n_q}\idm} \Psi + \lambda \idm\bigg)^{-1}
 \Psi^\top { \frac{1}{ {n_p}} 1_{n_p} \choose -\frac{1}{ {n_q}} 1_{n_q}}
 \\
& \!\!\!= \!\!\!& \frac{1}{2}\frac{1}{\rho(1-\rho)}
   -  \frac{1}{2} \lambda { \frac{1}{\rho} 1_{n_p} \choose \frac{-1}{1-\rho}1_{n_q} }
   \bigg( \Psi \Psi^\top + \lambda \bimat{\frac{n_p}{\rho} \idm}{0}{0}{\frac{n_q} {1-\rho} \idm} \bigg)^{-1}
    { \frac{1}{\rho} 1_{n_p} \choose \frac{-1}{1-\rho} 1_{n_q}}.
    \\
\EEAS
 
 \textbf{Expression for invertible $K = \Psi \Psi^\top$.} We must then have $m \geqslant n$, and we can simply redefine the feature map by projection onto the span of the data to reduce to the case $m=n$, and thus $\Psi$ invertible.
 Then for $D(\rho) = \bimat{\frac{\rho}{n_p} \idm}{0}{0}{\frac {1-\rho}{n_q} \idm}$ (which is affine in $\rho$), and
 $z =  { \frac{1}{ {n_p}} 1_{n_p} \choose -\frac{1}{ {n_q}} 1_{n_q}}$, we get the expression for a single $\rho$:
\BEAS
  F_\lambda(\hat{p}\|\hat{q},\varphi) 
 & \!\!\!= \!\!\!&\frac{1}{2} z^\top \Psi ( \Psi^\top D(\rho) \Psi + \lambda \idm)^{-1} \Psi^\top z  =\frac{1}{2}  z^\top  (  D(\rho)   + \lambda (\Psi \Psi^\top)^{-1})^{-1}   z 
 \EEAS
 Since $D(\rho)$ is affine in $\rho$, we can integrate with respect to $d\nu(\rho)$ in closed form using values of $f$.

\section{Dealing with an unregularized constant feature}
\label{app:constant}
In order to go beyond linear invariance towards affine invariance, we look at the effect of adding a constant feature (which, by monotonicity, can only lead to larger values). 

\BIT

 \item \emph{Constant feature and affine invariance:}  Adding a constant to $\varphi$, that is, using $\bivec{\varphi}{1}$, we get from \eq{official}:
\[ \inf_{t: (0,1) \to \rb} \frac{1}{2}\int_0^1 t(\rho) d\nu(\rho) \mbox{ such that } \left( \begin{array}{ccc} \rho \Sigma_p + (1-\rho)\Sigma_q& \rho \mu_p + (1-\rho)\mu_q & \mu_p-\mu_q\\
\rho \mu_p^\top + (1-\rho)\mu_q^\top & 1 & 0 \\ 
(\mu_p-\mu_q)^\top & 0 & t(\rho)
\end{array} \right)
\succcurlyeq 0\]
with a constraint that is equivalent to
\[
\rho \Sigma_p + (1-\rho)\Sigma_q - ( \rho \mu_p + (1-\rho)\mu_q )   (\rho \mu_p + (1-\rho)\mu_q )^\top \succcurlyeq \frac{1}{t(\rho)} ( 
\mu_p-\mu_q)   (\mu_p-\mu_q)^\top.
\]
This leads to $F(p\|q,{\varphi \choose 1})$ being equal to
\[
 \frac{1}{2}
\int_0^1 \!\! ( \mu_p-\mu_q)^\top (\rho \Sigma_p+(1-\rho)\Sigma_q  - ( \rho \mu_p + (1-\rho)\mu_q  )   (\rho \mu_p + (1-\rho)\mu_q )^\top )^{-1} (\mu_p-\mu_q) 
d\nu(\rho),  \]
with 
\BEAS
& & \rho \Sigma_p+(1-\rho)\Sigma_q  - ( \rho \mu_p + (1-\rho) \mu_q  )   (\rho \mu_p + (1-\rho)\mu_q )^\top \\
& = &  \rho( \Sigma_p - \mu_p  \mu_p^\top) + (1-\rho) ( \Sigma_q - \mu_q  \mu_q^\top) + \rho(1-\rho) (\mu_p-\mu_q)  (\mu_p-\mu_q)^\top,
\EEAS
which is the (centered) covariance matrix $C_{r(\rho)}$ for the distribution $\rho p + (1-\rho) q$. We thus have the expression
\[
F\Big(p\|q,{\varphi \choose 1}\Big) = \frac{1}{2} \int_0^1 ( \mu_p - \mu_q)^\top C_{r(\rho)}^{-1} (\mu_p -\mu_q)   d\nu(\rho).
\]
This makes the divergence invariant under any invertible \emph{affine} mapping.

In terms of regularization, in order to avoid penalizing constant functions, we can consider
\[\!\! \inf_{t: (0,1) \to \rb} \frac{1}{2}\int_0^1\!\! t(\rho) d\nu(\rho) \mbox{ such that } \left( \begin{array}{ccc} \!\rho \Sigma_p + (1-\rho)\Sigma_q \textcolor{red}{+ \lambda \idm}\!\!\!& \rho \mu_p + (1-\rho)\mu_q & \mu_p-\mu_q\\
\rho \mu_p^\top + (1-\rho)\mu_q^\top & 1 & 0 \\ 
(\mu_p-\mu_q)^\top & 0 & t(\rho)
\end{array} \right)
\succcurlyeq 0,\]
which we use in experiments. From the variational representations in terms of matrices $M = \bimat{\bar{M}}{\bar{c}}{\bar{c}^\top}{\bar{e}}$
and $N = \bimat{\bar{N}}{\bar{d}}{(\bar{d})^\top}{\bar{f}}$, this corresponds to gradients with respect to $\Sigma_p$, $\Sigma_q$, $\mu_p$ and $\mu_q$ equal to $\bar{M}$, $\bar{N}$,  $2(c+\bar{c})$, and  $2(-c +\bar{d})$.

\item \emph{Existence of constant feature:}  A special formulation applies when there exists $u \in \H$ such that for all $x\in \X$, $  u^\top \varphi(x) = 1$. We get a dependence only on $\Sigma_p$ and $\Sigma_q$:	
\[
F(p\|q,\varphi) =  u^\top \Sigma_q^{1/2} f\big(
 \Sigma_q^{-1/2} \Sigma_p \Sigma_q^{-1/2}
 \big)  \Sigma_q^{1/2} u.\]

\item \emph{Variational formulation:}
When there exists a $u$ to represent constants, Prop.~\ref{prop:G} can be modified into the following result.
\begin{proposition}
 Define:
\BEAS
H(\mu_p,\Sigma_p\|\mu_q,\Sigma_q)
& = & 
\sup_{M,N \in \mathcal{A} }
\tr[ M\Sigma_p] +   \tr[ N\Sigma_q]    \\
& & \mbox{ such that } \forall \lambda >0, \ \lambda M + N  - f(\lambda) uu^\top \preccurlyeq 0.
\EEAS

\BIT
\item If $f$ is twice continuously differentiable, then $H(\mu_p,\Sigma_p\|\mu_q,\Sigma_q) \leqslant F(p\|q,\varphi)$. 
\item If $f$ is operator-convex, then $H(\mu_p,\Sigma_p\|\mu_q,\Sigma_q) \geqslant F(p\|q,\varphi)$, and thus the values are equal.
\EIT
\end{proposition}

\EIT

\section{Proof of convergence rate for kernel methods}
\label{app:proofkernel}
Throughout this section, we will use matrix notation for simplicity for operators between separable Hilbert spaces (i.e., which have a countable basis), rather than tensor products. For $u,v \in \H$, $uv^\top$ is the operator from $\H$ to $\H$ that maps $w$ to $u ( v^\top w)$, where $v^\top w$ denotes the Hilbert space dot-product. This makes the notation simpler, in particular when we consider Cartesian products in \mysec{mutual}. Note that because of the ``kernel trick'', all computations will be performed on well-defined finite-dimensional spaces (spanned by the observed data).

\subsection{Impossibility of generalized eigenvalue decomposition}
\label{app:noeig}
We first note that the spectral formula cannot hold for infinite-dimensional spaces by replacing a sum with an infinite sum (which is usual in the analysis of kernel methods), as there is no discrete generalized spectrum in general (we still provide an integral with respect to $\rho$ in \eq{defF}, which is valid, and all algorithms end up operating in finite dimension, so that the spectral formula holds).

 In the simple case $\X=[0,1]$, with a translation-invariant kernel $k(x-y)$ where $k$ is $1$-periodic, we can choose $\varphi(x)_\omega = \hat{k}_\omega^{1/2} e^{2 i\pi \omega x}$ for $\omega \in \mathbb{Z}$, with $(\Sigma_p)_{\omega\omega'} = \hat{p}(\omega-\omega') \hat{k}_\omega^{1/2}\hat{k}_{\omega'}^{1/2}$ and $(\Sigma_q)_{\omega\omega'} = \hat{q}(\omega-\omega') \hat{k}_\omega^{1/2}\hat{k}_{\omega'}^{1/2}$, leading to, for \emph{any} $x \in [0,1]$, $\Sigma_p u(x) = [ p(x)/q(x)] \Sigma_q u(x)$ for $u(x)_\omega = \hat{k}_{\omega}^{-1/2} e^{2i\pi\omega x}$, and thus potentially with a continuous spectrum.

\subsection{Generic ``abstract'' result}
We start with a result that can be applied to other situations, with classical notions of covariance and integral operators. We consider a standard set of assumptions used in the analysis of kernel-based divergences and more generally supervised learning~\cite{cap2007optimal,harchaoui2008testing,ribero2026regularized,bach2022information,lin2020optimal}.
We make the following assumptions (denoted \textbf{H}):
\BIT
\item \textbf{Strictly positive and bounded relative density}: we assume that the relative density $dp/dq$ exists and is always in $[\alpha,1/\alpha]$. This will lead to terms in $1/\alpha$ that can be avoided if finer assumptions on $p$ and $q$ are used (see proof). We will consider the classical notions of degrees of freedom leading to the ``capacity condition'' (see below).

\item \textbf{``Source condition''}: we assume that for all $\rho \in [0,1]$,
$\big( \frac{dp}{dq}(x) - 1 \big) / \big( \rho \frac{dp}{dq}(x) + 1-\rho\big) = \theta(\rho)^\top 
 \Sigma_{\rho p + (1-\rho) q}^{r-1/2}  \varphi(x) $, where $\Sigma_{(\rho)}$ is the integral operator extension (that can be applied to all square-integrable functions) of the covariance operator~\cite{cap2007optimal,lin2020optimal}, and $r \in (0,1/2]$. \EIT

Throughout the proof, we will need several quantities that are classical in the analysis of kernel methods, the ``degrees of freedom'' and their ``maximal'' versions~\cite{bach2013sharp,alaoui2015fast}:
\BEAS
{\rm df}_p(\lambda) & = & \tr[ \Sigma_p (\Sigma_p + \lambda \idm)^{-1} ] \\
{\rm df}_p^{\max} (\lambda) & = & \sup_{x \in \X} \    \varphi(x)^\top (\Sigma_p + \lambda \idm)^{-1} \varphi(x)   \geqslant {\rm df}_p(\lambda) \\
{\rm df}_q(\lambda) & = & \tr[ \Sigma_q (\Sigma_q + \lambda \idm)^{-1} ] \\
{\rm df}_q^{\max} (\lambda) & = & \sup_{x \in \X}   \varphi(x)^\top (\Sigma_q + \lambda \idm)^{-1} \varphi(x)   \geqslant {\rm df}_q(\lambda) .
\EEAS
We also consider the maximal version over $p$ and $q$, and over cross-products, that is,
\BEAS
{\rm df}(\lambda) & = &  \max\{ {\rm df}_p(\lambda), {\rm df}_q(\lambda), \tr[ \Sigma_q (\Sigma_p + \lambda \idm)^{-1} ],\tr[ \Sigma_p (\Sigma_q + \lambda \idm)^{-1} ], \\
& & \hspace*{1cm}
\tr[ \Sigma_q (\Sigma_p + \lambda \idm)^{-1} \Sigma_q (\Sigma_p + \lambda \idm)^{-1} ],  
\tr[ \Sigma_p (\Sigma_q + \lambda \idm)^{-1} \Sigma_p (\Sigma_q + \lambda \idm)^{-1} ]\},
\EEAS
and ${\rm df}^{\max}(\lambda) = \max\{ {\rm df}_p^{\max}(\lambda), {\rm df}_q^{\max}(\lambda)\}$, and we consider $n = \min \{n_p,n_q\}$. See \mysec{tradeoffbound} for examples of scalings of ${\rm df}(\lambda)$ and ${\rm df}^{\max}(\lambda)$.  

We also assume that the feature map is uniformly bounded, and write
\[
    R_\varphi^2 =\sup_{x\in\mathcal X}\|\varphi(x)\|_{\mathcal H}^2<\infty .
\]

Given Assumption \textbf{H}, we have the following bound
\BEAS
{\rm df}(\lambda) & \leqslant &  \max\Big\{ {\rm df}_p(\lambda), {\rm df}_q(\lambda), \frac{1}{\alpha^2}  {\rm df}_p(\lambda/\alpha), \frac{1}{\alpha^2}  {\rm df}_q(\lambda/\alpha)\Big\}.
\EEAS

The main fully non-asymptotic technical result is the following.

\begin{proposition}
\label{prop:rate}
Assume \textbf{H}, and, for $\xi \geqslant 1$, assume $\lambda \leqslant \min\{ \| \Sigma_p\|_{\rm op}, \| \Sigma_q\|_{\rm op}\}$,  
$ \sqrt{\xi} \sqrt{{\rm df}^{\max}(\lambda) / n}  \sqrt{\log (2n)} \leqslant  \frac{3}{16}$. We have:
\BEA
\label{eq:bound}
 \Big(\E \big[  ( F_\lambda(\hat{p}\| \hat{q},\varphi)  - D(p\|q) )^2 \big] \Big)^{1/2} 
 & \leqslant  &   32 \frac{ {\rm df}(\lambda)}{n} +
    \frac{8R_\varphi^2}{\lambda}
  \frac{\sqrt{8}}{n^{2\xi}} + 
\lambda^{2r} \!\!\!\int_0^1  \!\!\| \theta_\rho \|^2 d\nu(\rho) \\
\notag & & \hspace*{.5cm} +16\sqrt{\xi}\sqrt{\frac{{\rm df}^{\max}(\lambda) }{n} }  \sqrt{\log (n)} \cdot  D(p\|q)
+  
8\sqrt{\frac{2D(p\|q)}{\alpha n}}
.
\EEA
\end{proposition}
It has both multiplicative and additive terms.

Throughout the proof, we will use the notation $\Sigma_{(\rho)} = \rho \Sigma_p + (1-\rho) \Sigma_q$, as well
as its empirical counterpart $\hat{\Sigma}_{(\rho)} = \rho \hat{\Sigma}_p + (1-\rho) \hat{\Sigma}_q$. We will also use $M(\rho) = \Sigma_{(\rho)} +\lambda \idm$ and $\hat{M}(\rho) = \hat{\Sigma}_{(\rho)} +\lambda \idm$.

The proof is organized as follows: (1) we prove a variance bound in \mysec{varbound},  (2) prove a bias bound based on assumptions on $p$ and $q$ in \mysec{biasbound}, and (3) combine them together and study trade-offs in \mysec{tradeoffbound}.

\subsection{Variance bound for $F_\lambda(\hat{p}\| \hat{q},\varphi)$ }
\label{sec:varbound}
We consider $x_1,\dots,x_{n_p} $ i.i.d. from $p$, and $y_1,\dots,y_{n_q}$ i.i.d. from $q$, and the quantities
\BEAS
\hat{a}_\lambda & = & \int_0^1
( \hat\mu_p - \hat\mu_q)^\top (    \hat{ {\Sigma}}_{(\rho)}  +\lambda \idm) ^{-1} ( \hat\mu_p - \hat\mu_q)   d\nu(\rho)
\\
 \hat{b}_\lambda  & = & \int_0^1
( \hat\mu_p - \hat\mu_q)^\top  (   {\Sigma}_{(\rho)} +\lambda \idm ) ^{-1} ( \hat\mu_p - \hat\mu_q)  d\nu(\rho)
\\
b_\lambda & = & \int_0^1 
(  \mu_p -  \mu_q)^\top  (  {\Sigma}_{(\rho)} +\lambda \idm) ^{-1} (  \mu_p -  \mu_q)   
d\nu(\rho).
\EEAS
The quantity $\hat{a}_\lambda$ is (twice, since the factor $1/2$ has been removed throughout App.~\ref{app:proofkernel}) our estimate, while $\hat{b}_\lambda$ is the random quantity with $\hat{ {\Sigma}}_{(\rho)}$ replaced by ${ {\Sigma}}_{(\rho)}$, and $b_\lambda$ is the same quantity with all empirical estimates replaced by their population quantities.

We start with bounding $|\hat{a}_\lambda - \hat{b}_\lambda|$ and then bound $|\hat{b}_\lambda - b_\lambda|$.

\textbf{Bounding $|\hat{a}_\lambda - \hat{b}_\lambda|$.} We follow the standard analysis of kernel methods (see, e.g.,~\cite{lin2020optimal,bach2024learning}).
For $t \in [0,\frac{3}{4}]$, we consider the event $\mathcal{A}$, where for both distributions:
\[
\big\| ( \Sigma_p + \lambda \idm)^{-1/2} ( \Sigma_p - \hat{\Sigma}_p) 
( \Sigma_p + \lambda \idm)^{-1/2} \big\|_{\rm op} , 
\big\| ( \Sigma_q + \lambda \idm)^{-1/2} ( \Sigma_q - \hat{\Sigma}_q) 
( \Sigma_q + \lambda \idm)^{-1/2} \big\|_{\rm op} \leqslant t.
\]
The following lemma shows that this is achieved with high probability.

\begin{lemma}
\label{lemma:conc}
Assume $\lambda \leqslant \min\{ \| \Sigma_p\|_{\rm op}, \| \Sigma_q\|_{\rm op}\}$, and 
$ \sqrt{\xi}\sqrt{{\rm df}^{\max}(\lambda) / n}  \sqrt{\log (2n)} \leqslant  \frac{3}{16}$. Then, with $t = 4\sqrt{\xi} \sqrt{{\rm df}^{\max}(\lambda) / n}  \sqrt{\log (2n)} \in (0,3/4]$, we have
$\P(\mathcal{A}^c)  \leqslant \frac{8}{n^{4\xi}}$.
\end{lemma}

\begin{proof}
We prove the bound for \(p\); the proof for \(q\) is identical. Let
\[
    \mathcal{A}
    =
    \left\{
    \left\|
    (\Sigma_p+\lambda I)^{-1/2}
    (\Sigma_p-\hat\Sigma_p)
    (\Sigma_p+\lambda I)^{-1/2}
    \right\|_{\rm op}
    \le t
    \right\}.
\]
By the intrinsic-dimension matrix Bernstein bound \cite[Theorem 7.7.1]{tropp2015introduction} applied to
\[
    (\Sigma_p+\lambda I)^{-1/2}
    \bigl(\varphi(x)\varphi(x)^\top-\Sigma_p\bigr)
    (\Sigma_p+\lambda I)^{-1/2},
\]
provided
\[
    t
    \ge
    \sqrt{\frac{\operatorname{df}^{\max}_p(\lambda)}{n_p}}
    +
    \frac{\operatorname{df}^{\max}_p(\lambda)}{3n_p},
\]
we have
\[
\mathbb P( \mathcal{A}^c)
\leqslant
8
\frac{\operatorname{df}_p(\lambda)}
     {\|\Sigma_p\|_{\rm op}/(\|\Sigma_p\|_{\rm op}+\lambda)}
\exp\left(
    -
    \frac{n_p}{2\operatorname{df}^{\max}_p(\lambda)}
    \frac{t^2}{1+t/3}
\right).
\]
Since \(\lambda\leqslant \|\Sigma_p\|_{\rm op}\),
\[
    \frac{\|\Sigma_p\|_{\rm op}}
         {\|\Sigma_p\|_{\rm op}+\lambda}
    \geqslant \frac12,
\]
and since
\[
    \operatorname{df}_p(\lambda)\leqslant \operatorname{df}^{\max}_p(\lambda),
\]
we get
\[
\mathbb P( \mathcal{A}^c)
\leqslant
16\operatorname{df}^{\max}_p(\lambda)
\exp\left(
    -
    \frac{n_p}{2\operatorname{df}^{\max}_p(\lambda)}
    \frac{t^2}{1+t/3}
\right).
\]
The assumption of the lemma gives \(t\leqslant 3/4\) and
\[
    \frac{\operatorname{df}^{\max}(\lambda)}{n}
    \leqslant
    \frac{9}{256\,\xi\log(2n)}
    \leqslant
    \frac14 .
\]
Thus
\[
    \frac{\operatorname{df}^{\max}_p(\lambda)}{n_p}
    \le \frac14,
\]
and hence
\[
    \sqrt{\frac{\operatorname{df}^{\max}_p(\lambda)}{n_p}}
    +
    \frac{\operatorname{df}^{\max}_p(\lambda)}{3n_p}
    \leqslant
    \frac76
    \sqrt{\frac{\operatorname{df}^{\max}_p(\lambda)}{n_p}}
    \leqslant t,
\]
so the preceding Bernstein bound applies.

Since \(t\leqslant 3/4\),
\[
    \frac{1}{2(1+t/3)}\geqslant \frac25.
\]
Therefore
\[
\mathbb P( \mathcal{A}^c)
\leqslant
16\operatorname{df}^{\max}_p(\lambda)
\exp\left(
    -
    \frac{2n_p t^2}{5\operatorname{df}^{\max}_p(\lambda)}
\right).
\]
The same bound holds for \(q\). Taking a union bound and using
\(n=\min\{n_p,n_q\}\) gives
\[
\mathbb P( \mathcal{A}^c)
\leqslant
32\operatorname{df}^{\max}(\lambda)
\exp\left(
    -
    \frac{2n t^2}{5\operatorname{df}^{\max}(\lambda)}
\right).
\]
Substituting
\[
    t^2
    =
    16\xi
    \frac{\operatorname{df}^{\max}(\lambda)}{n}
    \log(2n)
\]
yields
\[
\mathbb P( \mathcal{A}^c)
\leqslant
32\operatorname{df}^{\max}(\lambda)
\exp\left(
    -\frac{32}{5}\xi\log(2n)
\right).
\]
Finally, the assumption above implies
\(\operatorname{df}^{\max}(\lambda)\le n/4\), and hence
\[
\mathbb P( \mathcal{A}^c)
\leqslant
8n(2n)^{-32\xi/5}
\leqslant
\frac{8}{n^{4\xi}},
\]
for \(\xi\geqslant 1\). This proves the result.
\end{proof}

Using Lemma~\ref{lemma:conc} above, under the event $\mathcal{A}$, we have $(1-t) (\Sigma_{(\rho)} + \lambda \idm) \preccurlyeq
(\hat{\Sigma}_{(\rho)} + \lambda \idm) \preccurlyeq (1+t) (\Sigma_{(\rho)} + \lambda \idm)$.
This implies 
\[
- \frac{t}{1+t}
 (\Sigma_{(\rho)} + \lambda \idm)^{-1} \preccurlyeq (\hat{\Sigma}_{(\rho)} + \lambda \idm)^{-1} - (\Sigma_{(\rho)} + \lambda \idm)^{-1} \preccurlyeq \frac{t}{1-t}  (\Sigma_{(\rho)} + \lambda \idm)^{-1}.
\]
This leads to, after integrating with respect to $\rho \in (0,1)$, since $t \leqslant 3/4$,
\[
| \hat{a}_\lambda  - \hat{b}_\lambda  |  \leqslant \max \Big\{ \frac{t}{1+t},\frac{t}{1-t} \Big\}\cdot \hat{b}_\lambda \leqslant 4 t \cdot \hat{b}_\lambda.
\]

This leads to, on the event \(\mathcal{A}\),
\[
    |\hat a_\lambda-\hat b_\lambda|
    \leqslant 4t\,\hat b_\lambda .
\]
On \(\mathcal{A}^c\), we use a deterministic ridge bound. Since
\(\hat M(\rho)\succcurlyeq \lambda I\), and since the kernel is bounded,
\[
    \hat a_\lambda
    \le
    \int_0^1
    \frac{\|\hat\mu_p-\hat\mu_q\|^2}{\lambda}
    \,d\nu(\rho)
    \leqslant
    \frac{4R_\varphi^2}{\lambda},
    \qquad
   R_\varphi^2 = \sup_{x\in\mathcal X}\|\varphi(x)\|^2,
\]
with a similar bound for $\hat{b}_\lambda$. 
Hence
\[
\begin{aligned}
\Big(\mathbb E|\hat a_\lambda-\hat b_\lambda|^2\Big)^{1/2}
&\le
4t\Big(\mathbb E|\hat b_\lambda|^2\Big)^{1/2}
+
    \frac{8R_\varphi^2}{\lambda}
\mathbb P(\mathcal{A}^c)^{1/2}
\\
&\le
4t\Big(\mathbb E|\hat b_\lambda-b_\lambda|^2\Big)^{1/2}
+
4tb_\lambda
+
    \frac{8R_\varphi^2}{\lambda}
  \frac{\sqrt{8}}{n^{2\xi}}.
\end{aligned}
\]

We thus only need to understand the behavior of $\hat{b}_\lambda$, which is much easier, as this is the squared norm of a weighted sum of independent random variables. Since we consider only convergence in mean, we only need to compute and bound moments.

\textbf{Studying $| \hat{b}_\lambda - b_\lambda|$ using V-statistics moment computations.}
 We consider a fixed $\rho$ and study 
 \[
 t_\rho = ( \hat{\mu}_p- \hat{\mu}_q)^\top M(\rho)^{-1}( \hat{\mu}_p- \hat{\mu}_q )   
 = \Big\|
 \sum_{i=1}^{n_p + n_q} w_i u_i
 \Big\|^2 ,
 \]
 with $w = { 1_{n_p} / n_p \choose -1_{n_q} / n_q}$, and $u_i = M(\rho)^{-1/2}\varphi(x_i) \in \H$ for $i \in \{1,\dots,n_p\}$ 
and $u_i = M(\rho)^{-1/2}\varphi(y_{i-n_p})\in \H$ for $i \in \{n_p+1,\dots,n_p+n_q\}$,
with $\mu_i = \E [u_i]$ and $C_i = \E[ u_i  u_i^\top] - \mu_i  \mu_i^\top$.

Define
\[
m_\rho=\sum_i w_i\mu_i,
\qquad
S_\rho=\sum_i w_i\bar u_i,
\qquad
\Gamma_\rho=\sum_i w_i^2 C_i .
\]
Then
$
u_\rho
=
(\mu_p-\mu_q)^\top M(\rho)^{-1}(\mu_p-\mu_q)
=
\|m_\rho\|^2,
$
and
\[
t_\rho-u_\rho
=
\|m_\rho+S_\rho\|^2-\|m_\rho\|^2
=
2m_\rho^\top S_\rho+\|S_\rho\|^2 .
\]
Therefore, by Minkowski's inequality,
\[
\left(
\mathbb E[(t_\rho-u_\rho)^2]
\right)^{1/2}
\le
2\left(m_\rho^\top\Gamma_\rho m_\rho\right)^{1/2}
+
\left(
\mathbb E\|S_\rho\|^4
\right)^{1/2}.
\]

We first bound the linear term. We have
\[
\Gamma_\rho
\preccurlyeq
M(\rho)^{-1/2}
\left(
\frac{1}{n_p}\Sigma_p+\frac{1}{n_q}\Sigma_q
\right)
M(\rho)^{-1/2},
\]
and Assumption \textbf{H} implies
\[
\Sigma(\rho)\succcurlyeq \alpha\Sigma_p,
\qquad
\Sigma(\rho)\succcurlyeq \alpha\Sigma_q,
\]
so that
\[
M(\rho)^{-1/2}\Sigma_pM(\rho)^{-1/2}\preccurlyeq \alpha^{-1}I,
\qquad
M(\rho)^{-1/2}\Sigma_qM(\rho)^{-1/2}\preccurlyeq \alpha^{-1}I.
\]
Thus, using $n=\min\{n_p,n_q\}$,
\[
\Gamma_\rho
\preccurlyeq
\frac{1}{\alpha}
\left(
\frac{1}{n_p}+\frac{1}{n_q}
\right)I
\preccurlyeq
\frac{2}{\alpha n}I.
\]
Hence
\[
2\left(m_\rho^\top\Gamma_\rho m_\rho\right)^{1/2}
\le
2\sqrt{\frac{2u_\rho}{\alpha n}} .
\]

We now bound the fourth moment of the centered quadratic term. Since the
$\bar u_i$'s are independent and centered,
\[
\begin{aligned}
\mathbb E\|S_\rho\|^4
&=
\sum_i w_i^4\mathbb E\|\bar u_i\|^4
+
2\sum_{i<j}w_i^2w_j^2
\operatorname{tr}(C_i)\operatorname{tr}(C_j)  +
4\sum_{i<j}w_i^2w_j^2
\operatorname{tr}(C_iC_j)  \\
&\leqslant
\sum_i w_i^4\mathbb E\|\bar u_i\|^4
+
(\operatorname{tr}\Gamma_\rho)^2
+
2\operatorname{tr}(\Gamma_\rho^2) \leqslant
\sum_i w_i^4\mathbb E\|\bar u_i\|^4
+
3(\operatorname{tr}\Gamma_\rho)^2 .
\end{aligned}
\]
Moreover,
\[
\operatorname{tr}\Gamma_\rho
\leqslant
\operatorname{tr}\left[
\left(
\frac{1}{n_p}\Sigma_p+\frac{1}{n_q}\Sigma_q
\right)M(\rho)^{-1}
\right]
\leqslant
\frac{2\operatorname{df}(\lambda)}{n},
\]
where we used the convexity bound
\[
M(\rho)^{-1}
\preccurlyeq
\rho(\Sigma_p+\lambda I)^{-1}
+
(1-\rho)(\Sigma_q+\lambda I)^{-1}
\]
and the definition of $\operatorname{df}(\lambda)$.

Using also the leverage bound
$
\|u_i\|^2\leqslant \operatorname{df}^{\max}(\lambda),
$
we have
$
\|\bar u_i\|^2
\leqslant
4\operatorname{df}^{\max}(\lambda),
$
and therefore
\[
\mathbb E\|\bar u_i\|^4
\leqslant
4\operatorname{df}^{\max}(\lambda)\operatorname{tr}(C_i).
\]
Consequently,
\[
\begin{aligned}
\sum_i w_i^4\mathbb E\|\bar u_i\|^4
&\leqslant
4\operatorname{df}^{\max}(\lambda)
\sum_i w_i^4\operatorname{tr}(C_i)  \leqslant
\frac{4\operatorname{df}^{\max}(\lambda)}{n^2}
\sum_i w_i^2\operatorname{tr}(C_i)  =
\frac{4\operatorname{df}^{\max}(\lambda)}{n^2}
\operatorname{tr}\Gamma_\rho \\
&\leqslant
8
\frac{\operatorname{df}^{\max}(\lambda)\operatorname{df}(\lambda)}{n^3}.
\end{aligned}
\]
Since $\operatorname{df}^{\max}(\lambda)\leqslant n/4$ and
$\operatorname{df}(\lambda)\geqslant 1/2$, this gives
\[
\sum_i w_i^4\mathbb E\|\bar u_i\|^4
\leqslant
4\frac{\operatorname{df}(\lambda)^2}{n^2}.
\]
Combining the last bounds,
\[
\mathbb E\|S_\rho\|^4
\leqslant
16\frac{\operatorname{df}(\lambda)^2}{n^2}.
\]
Thus, for every fixed $\rho$,
\[
\left(
\mathbb E[(t_\rho-u_\rho)^2]
\right)^{1/2}
\leqslant
4\frac{\operatorname{df}(\lambda)}{n}
+
2\sqrt{\frac{2u_\rho}{\alpha n}} .
\]

Since
\[
\hat b_\lambda-b_\lambda
=
\int_0^1(t_\rho-u_\rho)\,d\nu(\rho),
\qquad
b_\lambda=\int_0^1u_\rho\,d\nu(\rho),
\]
Minkowski's inequality and Jensen's inequality yield
\[
\begin{aligned}
\left(
\mathbb E[|\hat b_\lambda-b_\lambda|^2]
\right)^{1/2}
&\leqslant
\int_0^1
\left(
\mathbb E[(t_\rho-u_\rho)^2]
\right)^{1/2}
d\nu(\rho) \\
&\leqslant
4\frac{\operatorname{df}(\lambda)}{n}
+
2\sqrt{\frac{2}{\alpha n}}
\int_0^1\sqrt{u_\rho}\,d\nu(\rho) \leqslant
4\frac{\operatorname{df}(\lambda)}{n}
+
2\sqrt{\frac{2b_\lambda}{\alpha n}} .
\end{aligned}
\]

Combining this bound with the previous bound on
\(|\hat b_\lambda-\hat a_\lambda|\), and using \(t\leqslant 3/4\), gives
\[
\begin{aligned}
\Big(\mathbb E[|\hat a_\lambda-b_\lambda|^2]\Big)^{1/2}
&\leqslant
(1+4t)
\Big(\mathbb E[|\hat b_\lambda-b_\lambda|^2]\Big)^{1/2}
+
4tb_\lambda
+
\frac{8\sqrt 8\,R_\varphi^2}{\lambda n^{2\xi}}
\\
&\leqslant
4\left(
    4\frac{\operatorname{df}(\lambda)}{n}
    +
    2\sqrt{\frac{2b_\lambda}{\alpha n}}
\right)
+
4tb_\lambda
+
\frac{8\sqrt 8\,R_\varphi^2}{\lambda n^{2\xi}} .
\end{aligned}
\]

With
\[
t=4\sqrt{\xi}\sqrt{\frac{\operatorname{df}^{\max}(\lambda)}{n}\log(2n)}
\]
and $\operatorname{df}(\lambda)\geqslant 1/2$, we finally obtain
\[
 {
\left(
\mathbb E[|\hat a_\lambda-b_\lambda|^2]
\right)^{1/2}
\leqslant
32\frac{\operatorname{df}(\lambda)}{n}
+
8\sqrt{\frac{2b_\lambda}{\alpha n}}
+
16\sqrt{\xi}b_\lambda
\sqrt{
\frac{\operatorname{df}^{\max}(\lambda)}{n}\log(2n)
} +
    \frac{4R_\varphi^2}{\lambda}
      \frac{8\sqrt{8}}{n^{2\xi}},
}
\]
which leads to the desired result.

\subsection{Bias bound for $F_\lambda(\hat{p}\| \hat{q},\varphi)$ }
\label{sec:biasbound}

From the source condition, we have
\BEAS
\mu_p -\mu_q & = & 
\int_\X \varphi(x) ( dp(x) - dq(x))
=\int_\X \varphi(x) \varphi(x)^\top \Sigma_{(\rho)}^{r-1/2}   \theta_\rho \big[ \rho {dp} (x) + (1-\rho) dq(x) \big] \\
& = & 
  \Sigma_{(\rho)} ^{r+1/2} \theta_\rho.
\EEAS
This leads to, since $r \in (0,1/2]$,
\BEAS
\big|  b_\lambda - 2 D(p\|q) \big| & \leqslant &\int_0^1 \big|
(  \mu_p -  \mu_q)^\top  \big[(     {\Sigma}_{(\rho)}  ) ^{-1}
-  (  {\Sigma}_{(\rho)}  +\lambda \idm) ^{-1} \big]  (  \mu_p -  \mu_q)   \big| d\nu(\rho) \\
& = & \int_0^1 \big|
\lambda (  \mu_p -  \mu_q)^\top   (  {\Sigma}_{(\rho)} +\lambda \idm) ^{-1}
   (  {\Sigma}_{(\rho)}   ) ^{-1}    (  \mu_p -  \mu_q)   \big| d\nu(\rho)  \\
& = & \int_0^1  
\lambda \theta_\rho ^\top   (  {\Sigma}_{(\rho)} +\lambda \idm) ^{-1}
   (  {\Sigma}_{(\rho)}   ) ^{-1}    \Sigma_{(\rho)} ^{2r+1}\theta_\rho  d\nu(\rho)  \\
& = & \int_0^1  
\lambda \theta_\rho ^\top   (  {\Sigma}_{(\rho)} +\lambda \idm) ^{-1+2r}
  (  {\Sigma}_{(\rho)} +\lambda \idm) ^{-2r}
   \Sigma_{(\rho)} ^{2r}\theta_\rho  d\nu(\rho)  \\
    & \leqslant & 
     \lambda^{2r} \int_0^1  \| \theta_\rho \|^2 d\nu(\rho). \EEAS 
Note that we see a saturation effect at $r=1/2$ instead of at $r=1$ for regular supervised learning~\cite{neubauer1997converse}. 

\subsection{Bias-variance trade-off for $F_\lambda(\hat{p}\| \hat{q},\varphi)$}
\label{sec:tradeoffbound}
Starting with the bound in \eq{bound}, we can now specify the various terms for the simple setup described in \mysec{infdim}, using the tools from \cite{bach2017breaking}.

In \cite{bach2017breaking}, the problem on the Euclidean ball of dimension $d$ with a kernel corresponding to random features with the $\kappa$-th power of the ReLU activation and both a weight vector and a constant term is transformed into a problem on the unit sphere embedded in $\rb^{d+1}$ without any intercept term, with a direct link with spherical harmonics.

Given the invariance of $D(p\|q)$ under changes of variables, we can place ourselves in this setup, and  both the smoothness and the bounds  (in $\alpha$) transfer. In order to have the same degrees of freedom (up to constants) as the uniform distribution on the sphere, we simply need that both densities are bounded and there is an open set where these densities are strictly positive.\footnote{The sphere can be covered by rotated copies of this open set, with a finite overlap, and by invariance under rotation, all the maximal degrees of freedom for the uniform distribution on these rotated copies are the same, and by convexity, we overall get a (large) constant times the maximal degrees of freedom for the uniform distribution on the sphere, which is by symmetry exactly equal to the usual degrees of freedom.}

The eigenvalues of the covariance operator for the uniform density on the sphere in dimension $d+1$ are proportional to $O( 1/ i^{1+(2\kappa+1)/d})$, which is the usual scaling for the Sobolev space of order $s = \kappa+\frac{d+1}{2}$~\cite{fischer2020sobolev}, which this kernel generates. This implies that all degrees of freedom are $O(\lambda^{-d/(2s)})$.

Given the regularity assumption, the source condition is satisfied with $r = t/(2s)$, with the condition that $t \leqslant s$ (to preserve $r \leqslant 1/2$). Since we have assumed a finite $L_2$-norm on the $t$-th order derivatives and densities that are finite and bounded away from zero, there is a uniform bound on all $\|\theta_\rho \|^2$, for $\rho \in [0,1]$, and thus on  $\int_0^1  \!\| \theta_\rho \|^2 d\nu(\rho)$.  

 Thus, we have a bound on the error, up to constants:
 \[
 \frac{1}{n} \lambda^{-d/(2s)}+ 
\lambda^{t/s}  + \frac{1}{\lambda n^{2\xi}} +  \sqrt{ \frac{\log n}{n} \lambda^{-d/(2s)}}\cdot  D(p\|q) + \sqrt{ \frac{D(p\|q)}{n}} .\]

If we optimize $\lambda$ for the additive term, we get, if $\xi$ is taken large enough (i.e., $\xi>\frac{s+t}{2t+d}$),
\[
\lambda \propto \frac{1}{n^{\frac{s}{t+d/2}}},
\]
for a rate bounded by
\[
\frac{1}{n^{\frac{t}{t+d/2}}} + \sqrt{\frac{\log n}{n^{\frac{t}{t+d/2}}}}  D(p\|q) + \sqrt{ \frac{D(p\|q)}{n}} .
\]
Note that if we want to optimize $\lambda$ for the first multiplicative term, we obtain an overall bound in
\[
\big( \frac{\log n}{n} \big)^{\frac{t}{2t+d/2}},
\]
which is, up to a logarithmic term, equal to the one from~\cite{bach2022information} (which only applies to $t=1$). When optimized for the additive term (e.g., when $D(p\|q)$ is small), this is an improvement.

\label{app:scaling}

\subsection{Bias removal}
\label{app:bias}
\label{app:debiased}
The dominant term in the variance bound at the end of \mysec{varbound} comes from the square of  $   \sum_i w_i^2 \tr(C_i)
= \frac{1}{n_p} \tr M(\rho)^{-1} ( \Sigma_p - \mu_p  \mu_p^\top )
+
\frac{1}{n_q} \tr M(\rho)^{-1}  ( \Sigma_q - \mu_q   \mu_q^\top )
  $; it can be removed by subtracting an unbiased estimate from $t_\rho$ (with the proper renormalization by $n_p-1$ and $n_q-1$ instead of $n_p$ and $n_q$), a standard technique in V-statistics~\cite{lee2019u}, as suggested by~\cite{laurent1996efficient} in a similar context as ours:
 \BEQ
 \label{eq:biascorr}
  \frac{1}{n_p-1} \tr M(\rho)^{-1} ( \hat{\Sigma}_p -  \hat{\mu}_p    \hat{\mu}_p^\top )
+
\frac{1}{n_q-1} \tr M(\rho)^{-1}  (  \hat{\Sigma}_q -  \hat{\mu}_q    \hat{\mu}_q^\top ).
\EEQ

For a single $\rho$, this corresponds exactly to changing the value of our current estimate
\[
( \hat{\mu}_p - \hat{\mu}_q)^\top ( \hat{\Sigma}_{(\rho)} + \lambda \idm)^{-1} ( \hat{\mu}_p - \hat{\mu}_q )  
= \sum_{i,j=1}^{n_p+n_q}
w_i w_j  \hat{u}_i^\top \hat{u}_j ,
\]
where $w = { 1_{n_p} / n_p \choose -1_{n_q} / n_q} \in \rb^{n_p+n_q}$, and $\hat{u}_i = \hat{M}(\rho)^{-1/2} \varphi(x_i) = \hat{M}(\rho)^{-1/2} \varphi(z_i) \in \H$ for $i \in \{1,\dots,n_p\}$ 
and $\hat{u}_i = \hat{M}(\rho)^{-1/2} \varphi(y_{i-n_p})= \hat{M}(\rho)^{-1/2} \varphi(z_i)\in \H$ for $i \in \{n_p+1,\dots,n_p+n_q\}$.

The subtraction of the term in \eq{biascorr} is equivalent to replacing $w_i w_j$ by $W_{ij}$,
with the matrix $W$ defined as:
\BEAS
\!W & \!\!\!= \!\!\!&  \bimat{\frac{1}{n_p(n_p-1)} ( 1_{n_p} 1_{n_p}^\top -  \idm)  }{-\frac{1}{n_p n_q} 1_{n_p}1_{n_q}^\top }{-\frac{1}{n_p n_q} 1_{n_q}1_{n_p}^\top }{\frac{1}{n_q(n_q-1)} ( 1_{n_q} 1_{n_q}^\top -  \idm) }
\\ 
& \!\!\!= \!\!\!& 
\bimat{\frac{1}{n_p^2}   1_{n_p} 1_{n_p}^\top  }{-\frac{1}{n_p n_q} 1_{n_p}1_{n_q}^\top }{-\frac{1}{n_p n_q} 1_{n_q}1_{n_p}^\top }{\frac{1}{n_q^2}   1_{n_q} 1_{n_q}^\top  }
- \bimat{\!\frac{1}{n_p(n_p-1)} ( \idm - \frac{1}{n_p} 1_{n_p}1_{n_p}^\top) \!}{0}{0}{\!\frac{1}{n_q(n_q-1)} ( \idm - \frac{1}{n_q} 1_{n_q}1_{n_q}^\top) \!},
\EEAS
which has \emph{zero diagonal} and satisfies $W_{ij}^2 \leqslant \big(\frac{n}{n-1}\big)^2 w_i^2 w_j^2$.
The estimate is then
\[
\hat{c}_\lambda^{(\rho)} = \sum_{i,j=1}^{n_p+n_q}
W_{ij}   \hat{u}_i^\top \hat{u}_j 
= \sum_{i,j=1}^{n_p+n_q}
W_{ij}   \varphi(z_i)^\top \hat{M}(\rho) ^{-1}  \varphi(z_j) .
\]
As in \mysec{varbound}, we first show that we can replace it with
\[
\hat{d}_\lambda^{(\rho)} = \sum_{i,j=1}^{n_p+n_q}
W_{ij}    {u}_i^\top {u}_j 
= \sum_{i,j=1}^{n_p+n_q}
W_{ij}  \varphi(z_i)^\top   M(\rho)^{-1}  \varphi(z_j) .
\]
where the inverse matrix is replaced by the inverse of its expectation.

\textbf{Bounding $|\hat{c}_\lambda^{(\rho)}  - \hat{d}_\lambda^{(\rho)} |$.}  
Let
\[
    \varepsilon_\lambda
    =
    4\sqrt{
        \xi
        \frac{\operatorname{df}^{\max}(\lambda)}{n}
        \log(2n)
    } .
\]
On the event \(\mathcal{A}\), uniformly in \(\rho\),
\[
    -\frac{\varepsilon_\lambda}{1+\varepsilon_\lambda}M(\rho)^{-1}
    \preccurlyeq
    \hat M(\rho)^{-1}-M(\rho)^{-1}
    \preccurlyeq
    \frac{\varepsilon_\lambda}{1-\varepsilon_\lambda}M(\rho)^{-1}.
\]
Since \(\varepsilon_\lambda\leq 3/4\), this implies
\[
    -C\varepsilon_\lambda M(\rho)^{-1}
    \preccurlyeq
    \hat M(\rho)^{-1}-M(\rho)^{-1}
    \preccurlyeq
    C\varepsilon_\lambda M(\rho)^{-1}.
\]

The rank-one part of \(W\) gives the same contribution as in
\mysec{varbound}, namely the inverse-perturbation term
$
    C\varepsilon_\lambda b_\lambda^{(\rho)}
$
up to the same negligible complement-event contribution.

It remains to control the diagonal correction. On \(\mathcal{A}\), its contribution is bounded by a constant times
\[
\varepsilon_\lambda
\left[
\frac{1}{n_p-1}
\operatorname{tr}
\left(
M(\rho)^{-1}
(\hat\Sigma_p-\hat\mu_p\hat\mu_p^\top)
\right)
+
\frac{1}{n_q-1}
\operatorname{tr}
\left(
M(\rho)^{-1}
(\hat\Sigma_q-\hat\mu_q\hat\mu_q^\top)
\right)
\right].
\]
The same trace and leverage bounds as in \mysec{varbound} imply
\[
\left(
\mathbb E
\left[
\operatorname{tr}
\left(
M(\rho)^{-1}
(\hat\Sigma_p-\hat\mu_p\hat\mu_p^\top)
\right)^2
\right]
\right)^{1/2}
\leqslant
C\operatorname{df}(\lambda),
\]
and similarly for \(q\). Hence the diagonal correction contributes, in \(L_2\),
\[
    C\varepsilon_\lambda\frac{\operatorname{df}(\lambda)}{n}.
\]

On \(\mathcal{A}^c\), the deterministic ridge bound gives
\[
    |\hat c_\lambda^{(\rho)}-\hat d_\lambda^{(\rho)}|
    \leqslant
    C\frac{R_\varphi^2}{\lambda},
\]
and Lemma~\ref{lemma:conc} gives
\[
    \mathbb P(\mathcal{A}^c)^{1/2}
    \leqslant
    \frac{\sqrt 8}{n^{2\xi}}.
\]
Therefore the complement-event contribution is bounded by
$
    C\frac{R_\varphi^2}{\lambda n^{2\xi}}
$, with $C$ a constant.

Combining these bounds, the additional error caused by replacing
\(\hat M(\rho)^{-1}\) by \(M(\rho)^{-1}\) is
\[
    C\varepsilon_\lambda b_\lambda^{(\rho)}
    +
    C\varepsilon_\lambda\frac{\operatorname{df}(\lambda)}{n}
    +
    C\frac{R_\varphi^2}{\lambda n^{2\xi}}.
\]
After integration in \(\rho\), the first term gives
\[
    C\varepsilon_\lambda b_\lambda
    \leqslant
    C D(p\|q)
    \sqrt{
        \xi
        \frac{\operatorname{df}^{\max}(\lambda)}{n}
        \log(2n)
    } .
\]

\textbf{Bounding $|\hat{d}_\lambda^{(\rho)} - {d}_\lambda^{(\rho)}|$.} We now only need to bound
$
|\hat{d}_\lambda^{(\rho)} - {d}_\lambda^{(\rho)}|$
and thus to control the fluctuation of the V-statistic with the population inverse. For fixed \(\rho\), write
\[
m_p^\rho=M(\rho)^{-1/2}\mu_p,\qquad
m_q^\rho=M(\rho)^{-1/2}\mu_q,\qquad
\delta_\rho=m_p^\rho-m_q^\rho,
\]
and let
\[
\xi_i^p=M(\rho)^{-1/2}\phi(x_i)-m_p^\rho,\qquad
\xi_j^q=M(\rho)^{-1/2}\phi(y_j)-m_q^\rho .
\]
Since the off-diagonal block weights in \(W\) have total mass \(1\) on the \(p\)-block and \(q\)-block and total mass \(-1\) on each cross-block,
\[
d_\lambda^{(\rho)}
=
\sum_{i,j=1}^{n_p+n_q} W_{ij}\mathbb E[u_i]^\top\mathbb E[u_j]
=
\|m_p^\rho-m_q^\rho\|^2
=
(\mu_p-\mu_q)^\top M(\rho)^{-1}(\mu_p-\mu_q).
\]
Moreover,
\[
\hat d_\lambda^{(\rho)}-d_\lambda^{(\rho)}
=
2\delta_\rho^\top \Big(
\frac1{n_p}\sum_{i=1}^{n_p}\xi_i^p
-
\frac1{n_q}\sum_{j=1}^{n_q}\xi_j^q
\Big)
+
R_\rho ,
\]
with
\[
R_\rho
=
\frac1{n_p(n_p-1)}
\sum_{i\neq j}( \xi_i^p)^\top \xi_j^p 
+
\frac1{n_q(n_q-1)}
\sum_{i\neq j}  (\xi_i^q)^\top\xi_j^q 
-
\frac{2}{n_pn_q}
\sum_{i,j}( \xi_i^p)^\top \xi_j^q  .
\]
Let \(C_p^\rho=\mathbb E[\xi_i^p(\xi_i^p)^\top]\) and \(C_q^\rho=\mathbb E[\xi_j^q(\xi_j^q)^\top]\). By independence,
\[
\mathbb E[R_\rho^2]
\leqslant
C\Big(
\frac{\operatorname{tr}[(C_p^\rho)^2]}{n_p^2}
+
\frac{\operatorname{tr}[(C_q^\rho)^2]}{n_q^2}
+
\frac{\operatorname{tr}[C_p^\rho C_q^\rho]}{n_pn_q}
\Big)
\leqslant
C\frac{\mathrm{df}(\lambda)}{n^2},
\]
where the last inequality follows from the same trace bounds as in \mysec{varbound}. The first part satisfies
\[
\mathbb E\bigg[
 \delta_\rho^\top \Big(
\frac1{n_p}\sum_{i=1}^{n_p}\xi_i^p
-
\frac1{n_q}\sum_{j=1}^{n_q}\xi_j^q
\Big)^2
\bigg]
\leqslant
C\frac{b_\lambda^{(\rho)}}{n},
\qquad
b_\lambda^{(\rho)}=\|\delta_\rho\|^2 .
\]
Consequently,
\[
\Big(\mathbb E\Big[
\Big|\hat d_\lambda^{(\rho)}-d_\lambda^{(\rho)}\Big|^2
\Big]\Big)^{1/2}
\le
C\Big(
\sqrt{\frac{b_\lambda^{(\rho)}}{n}}
+
\frac{\sqrt{\mathrm{df}(\lambda)}}{n}
\Big).
\]
After integration over \(\rho\), Jensen's inequality and \(b_\lambda\leqslant 2 D(p\|q)\) give the contribution
\[
C\Big(
\sqrt{\frac{D(p\|q)}{n}}
+
\frac{\sqrt{\mathrm{df}(\lambda)}}{n}
\Big).
\]

We can conclude the proof by using
\[
 {d}_\lambda^{(\rho)} = \sum_{i,j=1}^{n_p+n_q}
W_{ij}    \E[{u}_i]^\top \E[{u}_j]  
= ( \mu_p - \mu_q)^\top M(\rho)^{-1} ( \mu_p - \mu_q) .
\]
For multiple values of $\rho$, we get an expression with spectral functions since we have a linear function of $M(\rho)^{-1}$ in \eq{biascorr}, that is,
\BEAS
&&\int_0^1 \tr [ \hat{C}  ( \rho  (\hat{\Sigma}_p +\lambda \idm)  + (1-\rho) (\hat{\Sigma}_q +\lambda \idm) )^{-1} ] d\nu(\rho)\\
& = & 2 \tr \big[
 (\hat{\Sigma}_q +\lambda \idm)^{-1/2} \hat{C}  (\hat{\Sigma}_q +\lambda \idm)^{-1/2}
h \big[
(\hat{\Sigma}_q +\lambda \idm)^{-1/2} (\hat{\Sigma}_p +\lambda \idm)(\hat{\Sigma}_q +\lambda \idm)^{-1/2}
 \big]
\big],
\EEAS
with $\hat{C} =  \frac{1}{n_p-1}   ( \hat{\Sigma}_p -  \hat{\mu}_p    \hat{\mu}_p^\top )
+
\frac{1}{n_q-1}   (  \hat{\Sigma}_q -  \hat{\mu}_q    \hat{\mu}_q ^\top),
$
and $h(t) = f(t) / (t-1)^2$ (note that we need to remove the factor $2$ when we subtract the factor from $F_\lambda(\hat{p}\|\hat{q},\varphi)$, to add back the factor $1/2$ that had been removed throughout App.~\ref{app:proofkernel}).

\textbf{Final rate.}  
Combining the bias bound, the concentration term for the inverse, and the preceding V-statistic estimate gives, up to constants depending only on the fixed density-ratio bounds,
\[
\begin{aligned}
\Big(
\mathbb E
\Big[
\Big|
\hat F^{\rm deb}_\lambda(\hat p,\hat q)-D(p\|q)
\Big|^2
\Big]
\Big)^{1/2}
\le
C\bigg[
&\lambda^{2r}
+
\frac{\sqrt{\operatorname{df}(\lambda)}}{n}
+
\sqrt{\frac{D(p\|q)}{n}}
\\
&+
D(p\|q)
\sqrt{
    \frac{\operatorname{df}^{\max}(\lambda)\log n}{n}
}
+
\varepsilon_\lambda\frac{\operatorname{df}(\lambda)}{n}
+
\frac{R_\varphi^2}{\lambda n^{2\xi}}
\bigg].
\end{aligned}
\]

In the Sobolev scaling of App.~\ref{sec:tradeoffbound}, up to constants,
\[
2r=\frac{t}{s},
\qquad
\mathrm{df}(\lambda)\sim \mathrm{df}^{\max}(\lambda)
\sim \lambda^{-d/(2s)}.
\]
Ignoring lower-order remainders when they are absorbed, the leading terms become
\[
\lambda^{t/s}
+
\frac1n\lambda^{-d/(4s)}
+
D(p\|q)
\sqrt{\frac{\log n}{n}\lambda^{-d/(2s)}} .
\]
Balancing the first two terms gives, up to constants,
\[
\lambda
\sim
n^{-s/(t+d/4)},
\qquad
\lambda_\ast^{t/s}
\sim
n^{-t/(t+d/4)}.
\]
With this choice,
\[
D(p\|q)
\sqrt{\frac{\log n}{n}\lambda^{-d/(2s)}}
=
D(p\|q)\sqrt{\log n}\,
n^{-\frac{t-d/4}{2(t+d/4)}}.
\]
Thus the multiplicative term tends to zero exactly when
\[
t>\frac d4.
\]
Under this condition, the debiased estimator has the leading rate
\[
n^{-t/(t+d/4)}
+
D(p\|q)\sqrt{\log n}\,
n^{-\frac{t-d/4}{2(t+d/4)}} ,
\]
up to the explicitly displayed parametric and inverse-perturbation remainders.

\section{Proof of convergence rate for neural networks}
\label{app:proofneural}

\subsection{Weight decay for neural networks}
\label{app:proofneuraldecay}

We consider a neural network model for the function $u_\rho$ that should solve \eq{DpqUU}, with \emph{shared} parameters (input weights neuron parameters $w_j,b_j$, $j = 1,\dots,m$ shared by all functions $u_\rho$), that is,
\[
u_\rho(x) = \sum_{j=1}^m (\eta_\rho)_j (w_j^\top x + b_j)_+
\]
with a penalty added to the variational problem equal to
\[
- \frac{\lambda}{2} \sum_{j=1}^m \Big\{
\| w_j\|^2 + b_j^2 + \int_0^1 (\eta_\rho)_j^2 d\nu(\rho)
\Big\}.
\]
For algorithms,  when optimizing over each $\eta_\rho$, this is equivalent to using our spectral formulation with the feature map of size $m$ defined as $\psi_{\{w,b\}}(x)_j = (w_j^\top x + b_j)_+$, with an extra regularization $\lambda \idm$ added to $\Sigma_p$ and $\Sigma_q$, with the remaining term  $- \frac{\lambda}{2} \sum_{j=1}^m \big\{
\| w_j\| ^2 + b_j^2\big\} $ added to the cost function to take gradient steps with respect to $w_j$ and $b_j$.

For the algorithm, to avoid performing eigenvalue decompositions of size $m$, we instead consider the reduced feature map $\Gamma^\top \psi_{\{w,b\}}(x)$, with $\Gamma \in \rb^{m \times r}$ with $r$ small ($8$ in experiments).

For the analysis, we follow \cite{bach2017breaking} and consider optimizing over the scale of $(w_j,b_j)$ and $((\eta_\rho)_j )_{\rho \in [0,1]}$, which leads to a penalty
\[ -  {\lambda}  \sum_{j=1}^m  
\sqrt{\| w_j\| ^2 + b_j^2} \sqrt{ \int_0^1 (\eta_\rho)_j^2 d\nu(\rho)},
 \]
which is a form of group Lasso penalty that we study in the overparameterized regime.
The difficulties relative to \cite{bach2017breaking} are the need to consider least-squares instead of a Lipschitz-continuous loss, the use of a penalty instead of a constraint, and the integral over $\rho$ with a group penalty. All these will be taken care of below, by recasting it in terms of measures over neurons, following~\cite{bach2017breaking}.

\textbf{Learning setup.}
We consider functions $u$ from $[0,1] \times \X \to \rb$, which we denote through $(\rho,x) \mapsto u_\rho(x)$. Our aim is to estimate $u$ that maximizes the variational formulation in \eq{DpqUU}, that we recall here for a unique $\rho$
  \BEAS
  \notag D(p\|q) & = &  \sup_{u:\X\to \rb} \int_\X \Big[
   u(x) \Big( \frac{dp}{dq}(x) - 1 \Big) - \frac{u(x)^2}{2} \Big( \rho \frac{dp}{dq}(x) + 1-\rho\Big)
   \Big] dq(x)\\
    & = &  \sup_{u:\X\to \rb} \int_\X \Big[ u(x) - \frac{\rho}{2} u(x)^2 \Big] dp(x)
   + \int_\X \Big[ - u(x) - \frac{1-\rho}{2} u(x)^2 \Big] dq(x).
    \EEAS
    The optimal function $u_\rho^\ast$ is such that $u_\rho^\ast(x) = \frac{dp/dq(x) - 1}{\rho dp/dq(x) + 1-\rho}$.

We assume that we have $n_p$ i.i.d.~observations $x_1,\dots,x_{n_p}$ from $p$, and $n_q$ i.i.d.~observations $y_1,\dots,y_{n_q}$ from $q$.

As in the proof of the kernel-method result in App.~\ref{app:proofkernel}, we start with a generic result that we will apply to our setup.

\subsection{Generic concentration lemma}
Following~\cite{bach2017breaking},
we consider a parameterization of $u$ as, for any $\rho \in [0,1]$ and $x \in \X$,
\BEQ
\label{eq:urho}
u_\rho(x) = \int_\W \varphi(x,w) d\eta_\rho(w) ,
\EEQ
where, for each $\rho \in [0,1]$, $\eta_\rho$ is a signed measure on $\W$ (potentially a sum of Diracs) with a penalty
\BEQ
\label{eq:ddd}
\int_\W \Big( \int_0^1 d\eta_\rho(w)^2 d\nu(\rho) \Big)^{1/2}.
\EEQ
This creates a penalty $\Omega(u)$ obtained as the infimum of \eq{ddd} over all decompositions of $u$ through \eq{urho} (with a value of $+\infty$ otherwise). This extends the variation norm~\cite{kurkova2001bounds} to this particular setup to include  $\rho$ and a group Lasso penalty.\footnote{To make the notation precise, we define the penalty through a dominating measure. A representation of \(u\) is a pair \((\gamma,a)\), where \(\gamma\) is a finite positive measure on \(\W\) and \(a:\W\to L^2(\nu)\) is measurable, such that
$
  u_\rho(x)=\int_\W \phi(x,w)a_\rho(w)d\gamma(w).
$ The group-variation norm used below is
$
  \Omega(u)=\inf_{(\gamma,a)}\int_W
  \big(\int_0^1 a_\rho(w)^2d\nu(\rho)\big)^{1/2}d\gamma(w),
$
where the infimum is over all such representations of \(u\), and \(\Omega(u)=+\infty\) if no representation exists. For a finite neural network this reduces to
$
  \Omega(u)=\inf \sum_{j=1}^m
  \big(\int_0^1 (\eta_{\rho j})^2d\nu(\rho)\big)^{1/2},
$
up to the separate scaling of the hidden-layer weights described above. Thus \(\Omega\) is the variation norm of an \(L^2(\nu)\)-valued signed measure on \(\W\), i.e., the continuous analogue of a group-Lasso penalty over neurons.
}

We consider the expected risk, defined so that $D(p\|q) = - \inf_{u} \mathcal{R}(u)$, and its empirical version,
\BEAS
\mathcal{R}(u) & = &  \int_0^1 \bigg( \int_\X \Big[ -u_\rho(x) + \frac{\rho}{2} u_\rho(x)^2 \Big] dp(x)
   + \int_\X \Big[  u_\rho(x) + \frac{1-\rho}{2} u_\rho(x)^2 \Big] dq(x)\bigg) d\nu(\rho)
\\
 & = &   \int_\X \bigg[  - \int_0^1 u_\rho(x) d\nu(\rho) + \int_0^1\frac{\rho}{2} u_\rho(x)^2d\nu(\rho)  \bigg] dp(x) \\
 & & \hspace*{4cm}
   + \int_\X \bigg[   \int_0^1u_\rho(x) d\nu(\rho) + \int_0^1\frac{1-\rho}{2} u_\rho(x)^2 d\nu(\rho)\bigg] dq(x) \\
   \hat{\mathcal{R}}(u)
   & = & 
   \frac{1}{n_p} \sum_{i=1}^{n_p} \bigg[ -  \int_0^1 u_\rho(x_i) d\nu(\rho) + \int_0^1\frac{\rho}{2} u_\rho(x_i)^2d\nu(\rho) \bigg]  \\
 & & \hspace*{4cm}
   +  \frac{1}{n_q} \sum_{j=1}^{n_q} \bigg[  \int_0^1u_\rho(y_j) d\nu(\rho)+\int_0^1\frac{1-\rho}{2} u_\rho(y_j)^2 d\nu(\rho)\bigg] .
    \EEAS
The following lemma provides all necessary concentration arguments to derive a bound.
\begin{lemma} 
\label{lemma:generic}
With the definitions above, assuming that $\varphi: \X \times \W\to \rb$ is uniformly bounded by $M$, we have, for any function $g: [0,1] \to \rb$, and $n = \min\{ n_p, n_q\}$,
\BIT
\item[(a)] If $\Omega(u) \leqslant t$, then for all $x \in \X$, we have:
\[  \Big| \int_0^1 g(\rho)u_\rho(x) d\nu(\rho) \Big|\leqslant  M t \| g \|_\infty ,  \ \mbox{ and } \ 
\Big| \int_0^1g(\rho) u_\rho(x)^2 d\nu(\rho) \Big|\leqslant   M^2 t^2 \| g \|_\infty, 
\]
\item[(b)] With probability greater than $1-\delta$, we have
\[ \sup_{\Omega(u) \leqslant t} \big\{ \mathcal{R}(u) 
- \hat{\mathcal{R}}(u) \big\}  \leqslant \E \big[ \sup_{\Omega(u) \leqslant t} \big\{ \mathcal{R}(u) 
- \hat{\mathcal{R}}(u) \big\} \big] + ( 2Mt + M^2t^2) \sqrt{ \frac{1}{n} \log \frac{1}{\delta} }.\]
\EIT
\item[(c)] We have:
\BEAS
\E \big[ \sup_{\Omega(u) \leqslant t} \big\{ \mathcal{R}(u) 
- \hat{\mathcal{R}}(u) \big\} \big]
&  \leqslant &  \hspace*{.25cm} 2 t\ \E_{\varepsilon, {\rm data} } 
\sup_{ w \in \W} \big|  \frac{1}{n_p} \sum_{i=1}^{n_p} \varepsilon_i \varphi(x_i,w) \big| \\
& & +  2 t\ \E_{\varepsilon, {\rm data} } 
\sup_{ w \in \W} \big|  \frac{1}{n_q} \sum_{j=1}^{n_q} \varepsilon_j \varphi(y_j,w) \big|\\
& & + 2 t^2 \ \E_{\varepsilon, \rm data} 
 \sup_{v,w \in \W}  \Big| \frac{1}{n_p} \sum_{i=1}^{n_p} \varepsilon_i \varphi(x_i,w) \varphi(x_i,v) \Big| \\
 & & + 2 t^2 \ \E_{\varepsilon, \rm data} 
 \sup_{v,w \in \W}  \big| \frac{1}{n_q} \sum_{j=1}^{n_q} \varepsilon_j\varphi(y_j,w) \varphi(y_j,v) \big|.
 \EEAS
We also have the same inequalities for $\sup_{\Omega(u) \leqslant t} \big\{
  \hat{\mathcal{R}}(u) -  \mathcal{R}(u) \big\} $.
\end{lemma}
\begin{proof}
\textbf{(a).}
For the first result, we write, using the representation of $u$ from \eq{urho} and \eq{ddd},
\BEAS
\Big|\int_0^1 g(\rho)u_\rho(x) d\nu(\rho) \Big|
& \leqslant & \int_0^1 \int_\W | g(\rho)| | \varphi(x,w)| \ | d \eta_\rho(w) | d\nu(\rho) 
\leqslant M  \int_0^1 \int_\W | g(\rho)|  \ | d \eta_\rho(w) | d\nu(\rho)  \\
& \leqslant & M t \Big(  \int_0^1 g(\rho)^2  d\nu(\rho)\Big)^{1/2} \mbox{ by Cauchy-Schwarz inequality, }
\\
& \leqslant &  M t \| g \|_\infty.
 \EEAS
 For the second inequality,
 \BEAS
 \Big| \int_0^1g(\rho) u_\rho(x)^2 d\nu(\rho) \Big|
 & \leqslant & 
 \int_0^1  \int_\W \int_\W |g(\rho)| \ |\varphi(x,w)|\  | \varphi(x,v) |\  |d\eta_\rho(w)| \  |d\eta_\rho(v)| d\nu(\rho) \\
& \leqslant & 
 M^2 \| g\|_\infty \int_0^1  \int_\W \int_\W   \   |d\eta_\rho(w)| \  |d\eta_\rho(v)| d\nu(\rho) \\
 & \leqslant &  M^2 \| g\|_\infty  t^2 \mbox{ by Cauchy-Schwarz inequality. }
 \EEAS
 
 \textbf{(b).} We simply use (as is commonly done in this context, see, e.g.,~\cite[Section 4.4.1]{bach2024learning}) MacDiarmid's inequality using the fact that when changing a single observation in $\hat{\mathcal{R}}(u)$, the maximal deviation is of order $\frac{2}{n_p} ( Mt + \frac{M^2t^2}{2} )$ or  $\frac{2}{n_q} ( Mt + \frac{M^2t^2}{2} )$.

\textbf{(c).}  We split $\mathcal{R}(u) 
- \hat{\mathcal{R}}(u) $ into four terms. The first one leads to, with $\varepsilon$ a vector of Rademacher random variables, and using standard Rademacher complexity results~(see, e.g., \cite[Section 4]{bach2024learning}), together with the representation of a function with a bound on its variation norm in \eq{ddd}:
\BEAS
 & & \E_{\rm data} 
\sup_{\Omega(u) \leqslant t}
 \bigg\{  - \int_0^1 u_\rho(x) d\nu(\rho) + 
  \frac{1}{n_p} \sum_{i=1}^{n_p}   \int_0^1 u_\rho(x_i) d\nu(\rho) \bigg\}  \\
  & \leqslant & 2 \E_{\varepsilon, \rm data} \sup_{\Omega(u) \leqslant t}
  \frac{1}{n_p} \sum_{i=1}^{n_p} \varepsilon_i \int_0^1 u_\rho(x_i) d\nu(\rho) \\
 & \leqslant & 2 \E_{\varepsilon, \rm data} \sup_{\int_\W  ( \int_0^1 d\eta_\rho(w)^2 d\nu(\rho)  )^{1/2} \leqslant t }
 \int_0^1 \int_\W  \Big( \frac{1}{n_p} \sum_{i=1}^{n_p} \varepsilon_i  \varphi(x_i,w) \Big) d\eta_\rho(w)  d\nu(\rho)
\\
& \leqslant & 2 \E_{\varepsilon, \rm data} \sup_{w \in \W} \Big| \frac{1}{n_p} \sum_{i=1}^{n_p} \varepsilon_i  \varphi(x_i,w) \Big| 
\sup_{\int_\W  ( \int_0^1 d\eta_\rho(w)^2 d\nu(\rho)  )^{1/2} \leqslant t }
 \int_0^1 \int_\W   d\eta_\rho(w)  |d\nu(\rho)|
\\
& \leqslant & 2t \ \E_{\varepsilon, \rm data} \sup_{w \in \W} \Big| \frac{1}{n_p} \sum_{i=1}^{n_p} \varepsilon_i  \varphi(x_i,w) \Big| 
  \mbox{ by Cauchy-Schwarz inequality.}
\EEAS
A similar term for the distribution $q$ leads to a term 
$2t \ \E_{\varepsilon, \rm data} \sup_{w \in \W} \Big| \frac{1}{n_q} \sum_{j=1}^{n_q} \varepsilon_j \varphi(y_j,w) \Big| $.

The more difficult term is
\BEAS
&\!\!\!\!\!\!\!\!& 
\!\!\!\! \E_{\rm data} 
\sup_{\Omega(u) \leqslant t}
 \bigg\{   \int_0^1 \rho u_\rho(x)^2 d\nu(\rho) - 
  \frac{1}{n_p} \sum_{i=1}^{n_p}   \int_0^1 \rho  u_\rho(x_i)^2 d\nu(\rho) \bigg\}  \\
  &\!\!\!\! \!\!\!\!\!\!\!\!\leqslant \!\!\!\!\!\!\!\!& 
 2 \E_{\varepsilon, \rm data} 
  \sup_{\Omega(u) \leqslant t}
    \frac{1}{n_p} \sum_{i=1}^{n_p} \varepsilon_i \int_0^1 \rho u_\rho(x_i)^2 d\nu(\rho) \\
  &\!\!\!\!\!\!\!\! \!\!\!\!\leqslant \!\!\!\! \!\!\!\!& 
 2 \E_{\varepsilon, \rm data} 
 \sup_{\int_\W  ( \int_0^1 d\eta_\rho(w)^2 d\nu(\rho)  )^{1/2} \leqslant t }
    \frac{1}{n_p} \int_0^1  \int_\W \int_\W \Big( \sum_{i=1}^{n_p} \varepsilon_i \varphi(x_i,w) \varphi(x_i,v) \Big) d\eta_\rho(w)  d\eta_\rho(v) \rho d\nu(\rho)
\\
 &\!\!\!\!\!\!\!\! \!\!\!\!\leqslant \!\!\!\! \!\!\!\!& 
 2 \E_{\varepsilon, \rm data} \!
 \sup_{v,w \in \W}  \Big| \frac{1}{n_p} \sum_{i=1}^{n_p} \varepsilon_i \varphi(x_i,w) \varphi(x_i,v) \Big|
 \sup_{\int_\W  ( \int_0^1 d\eta_\rho(w)^2 d\nu(\rho)  )^{1/2} \leqslant t }
    \int_0^1 \!\! \int_\W \!\int_\W\!  | d\eta_\rho(w)|  |d\eta_\rho(v)| \rho d\nu(\rho)
    \\
    &\!\!\!\!\!\!\!\! \!\!\!\!\leqslant \!\!\!\! \!\!\!\!& 
 2 t^2 \ \E_{\varepsilon, \rm data} 
 \sup_{v,w \in \W}  \Big| \frac{1}{n_p} \sum_{i=1}^{n_p} \varepsilon_i \varphi(x_i,w) \varphi(x_i,v) \Big|
  \mbox{ by Cauchy-Schwarz inequality.} \EEAS
  We also get an extra term in a similar way
  $2 t^2 \ \E_{\varepsilon, \rm data} 
 \sup_{v,w \in \W}  \big| \frac{1}{n_q} \sum_{j=1}^{n_q} \varepsilon_j\varphi(y_j,w) \varphi(y_j,v) \big|$.
\end{proof}

\subsection{Concentration for ReLU neural networks}

In the neural-network application, \(\mathcal X\) is the unit ball in
\(\mathbb R^d\) and the features are $        (w^\top x+b)_+$. 
We handle the bias with the standard augmentation~\cite{bach2017breaking}
\[
        \tilde x=\frac{(x,1)}{\sqrt2}\in\mathbb R^{d+1},
        \qquad
        \|\tilde x\|\leqslant 1 .
\]
Since
$
        (w^\top x+b)_+
        =
        \sqrt2\,
        \bigl((w,b)^\top\tilde x\bigr)_+,
$
the factor \(\sqrt2\) can be absorbed into the output coefficients and
therefore only changes the universal constants in the variation norm.
Thus, up to universal constants, it suffices to prove the concentration
bounds for homogeneous ReLUs
\[
        \varphi(x,w)=(w^\top x)_+
\]
with \(\|x\|\leqslant 1\) and \(\|w\|\leqslant 1\).
Lemma~\ref{lemma:generic} can now be made more precise.
\begin{lemma} 
\label{lemma:special}
With the assumptions above, assume $n = \min\{ n_p, n_q\}$.
\BIT
\item[(a)] With probability greater than $1-\delta$, we have
\[ \sup_{\Omega(u) \leqslant t} \big\{ \mathcal{R}(u) 
- \hat{\mathcal{R}}(u) \big\}  \leqslant \E \big[ \sup_{\Omega(u) \leqslant t} \big\{ \mathcal{R}(u) 
- \hat{\mathcal{R}}(u) \big\} \big] + ( 2t + t^2) \sqrt{ \frac{1}{n} \log \frac{1}{\delta} }.\]
\EIT
\item[(b)] We have:
\BEAS
\E \big[ \sup_{\Omega(u) \leqslant t} \big\{ \mathcal{R}(u) 
- \hat{\mathcal{R}}(u) \big\} \big]
&  \leqslant &  \frac{8t}{\sqrt{n}} + \frac{32  t^2}{\sqrt{n}}.
 \EEAS
 We also have the same inequalities for $\sup_{\Omega(u) \leqslant t} \big\{
  \hat{\mathcal{R}}(u) -  \mathcal{R}(u) \big\} $.
 \begin{proof}
 Part (a) follows directly from Lemma~\ref{lemma:generic}(b), using the normalized
ReLU features for which \(M=1\). For part (b), Lemma~\ref{lemma:generic}(c) reduces the
claim to bounding the Rademacher complexities of the linear ReLU class
and of the product ReLU class. The linear terms are controlled by the
standard contraction principle (see, e.g.,~\cite[Section 9.2.3]{bach2024learning}), yielding
\[
        \mathbb E_\varepsilon
        \sup_{\|w\|\leq1}
        \left|
            \frac1n\sum_{i=1}^n
            \varepsilon_i(w^\top x_i)_+
        \right|
        \leqslant
        \frac{2}{\sqrt n}.
\]
It remains to control the product class. We will show that\footnote{We can consider suprema over all $\|w\| ,\|v\| \leqslant 1$, rather than $\|w\| ,\|v\| = 1$, since this can only increase the value.}
\[
\mathbb E_\varepsilon
\sup_{\|w\| ,\|v\| \leqslant 1}
\left|
\frac{1}{n_p}
\sum_{i=1}^{n_p}
\varepsilon_i
(w^\top x_i)_+
(v^\top x_i)_+
\right|
\le
\frac{8}{\sqrt{n_p}}.
\]
For $(w,v)$ satisfying $\|w\| ,\|v\| \leqslant 1$, set
\[
    \psi_i(w,v)
    =
    (w^\top x_i)_+(v^\top x_i)_+.
\]

First, we can remove the absolute value: for any class of functions indexed by
$(w,v)$ such that they are all equal to zero for a certain $(w,v)$,
\[
\sup_{\|w\| ,\|v\| \leqslant 1}
\left|
\sum_{i=1}^{n_p}
\varepsilon_i \psi_i(w,v)
\right|
\leqslant
\sup_{\|w\| ,\|v\| \leqslant 1}
\sum_{i=1}^{n_p}
\varepsilon_i \psi_i(w,v)
+
\sup_{\|w\| ,\|v\| \leqslant 1}
\sum_{i=1}^{n_p}
(-\varepsilon_i)\psi_i(w,v).
\]
By symmetry of the Rademacher variables,
\[
\mathbb E_\varepsilon
\sup_{\|w\| ,\|v\| \leqslant 1}
\left|
\sum_{i=1}^{n_p}
\varepsilon_i \psi_i(w,v)
\right|
\leqslant
2
\mathbb E_\varepsilon
\sup_{\|w\| ,\|v\| \leqslant 1}
\sum_{i=1}^{n_p}
\varepsilon_i \psi_i(w,v).
\]

We now apply Maurer's vector contraction inequality
\cite[Theorem~3]{maurer2016vector}. The required Lipschitz condition is
coordinatewise in $i$. Fix $i$, and let
\[
    a = w^\top x_i,
    \qquad
    b = v^\top x_i,
    \qquad
    a' = {w'}^\top x_i,
    \qquad
    b' = {v'}^\top x_i.
\]
Since $\|w\| ,\|v\| ,\|w'\| ,\|v'\| \leqslant 1$,
we have $
    |a|, |b|, |a'|, |b'|
    \leqslant   1.$ 
Using that $t\mapsto t_+$ is $1$-Lipschitz,
\[
\begin{aligned}
\left|
(a)_+(b)_+ - (a')_+(b')_+
\right|
&\leqslant
(a)_+\left|(b)_+-(b')_+\right|
+
(b')_+\left|(a)_+-(a')_+\right| \\
&\leqslant  |b-b'| +  |a-a'|  \leqslant
\sqrt 2  
\left((a-a')^2+(b-b')^2\right)^{1/2}.
\end{aligned}
\]
Therefore the map
$
    (a,b)\mapsto (a)_+(b)_+
$
is $\sqrt 2$-Lipschitz on the relevant domain.

Define the vector-valued map
\[
    \Phi_i(w,v)
    =
    \sqrt 2  
    \begin{pmatrix}
        w^\top x_i \\
        v^\top x_i
    \end{pmatrix}.
\]
Then for all admissible $(w,v)$ and $(w',v')$,
\[
  |  \psi_i(w,v)-\psi_i(w',v')|
    \leqslant    \|\Phi_i(w,v)-\Phi_i(w',v')\| .
\]
The countability assumption in Maurer's theorem is harmless here: the
process depends only on the projections of $w$ and $v$ onto the finite-dimensional space $\operatorname{span}\{x_1,\ldots,x_{n_p}\}$, so the
supremum can be taken over a countable dense subset and then extended by
continuity.

Let $\eta_{i1},\eta_{i2}$ be independent Rademacher variables, independent
of the $\varepsilon_i$'s. By vector contraction,
\[
\begin{aligned}
\mathbb E_\varepsilon
\sup_{\|w\| ,\|v\| \leqslant 1}
\sum_{i=1}^{n_p}
\varepsilon_i \psi_i(w,v)
&\le
\sqrt 2\,
\mathbb E_\eta
\sup_{\|w\| ,\|v\| \leqslant 1}
\sum_{i=1}^{n_p}
\sum_{k=1}^2
\eta_{ik}\Phi_i(w,v)_k \\
&=
2\,
\mathbb E_\eta
\sup_{\|w\| ,\|v\| \leqslant 1}
\sum_{i=1}^{n_p}
\left(
\eta_{i1} w^\top x_i
+
\eta_{i2} v^\top x_i
\right).
\end{aligned}
\]
The supremum separates in $w$ and $v$:
\[
\begin{aligned}
&\sup_{\|w\| ,\|v\| \leqslant 1}
\sum_{i=1}^{n_p}
\left(
\eta_{i1} w^\top x_i
+
\eta_{i2} v^\top x_i
\right)  =
\left\|
\sum_{i=1}^{n_p}
 \eta_{i1}x_i
\right\| 
+
\left\|
\sum_{i=1}^{n_p}
 \eta_{i2}x_i
\right\| .
\end{aligned}
\]
Hence
\[
\mathbb E_\varepsilon
\sup_{\|w\| ,\|v\| \leqslant 1}
\sum_{i=1}^{n_p}
\varepsilon_i \psi_i(w,v)
\leqslant
4
\mathbb E_\eta
\left\|
\sum_{i=1}^{n_p}
 \eta_i x_i
\right\| .
\]
Combining this with the previous absolute-value reduction gives
\[
\mathbb E_\varepsilon
\sup_{\|w\| ,\|v\| \leqslant 1}
\left|
\frac{1}{n_p}
\sum_{i=1}^{n_p}
\varepsilon_i
(w^\top x_i)_+
(v^\top x_i)_+
\right|
\leqslant
\frac{8}{n_p}
\mathbb E_\eta
\left\|
\sum_{i=1}^{n_p}
 \eta_i x_i
\right\| .
\]
Finally, Jensen's inequality gives
\[
\begin{aligned}
\mathbb E_\eta
\left\|
\sum_{i=1}^{n_p}
 \eta_i x_i
\right\| 
&\leqslant
\left(
\mathbb E_\eta
\left\|
\sum_{i=1}^{n_p}
 \eta_i x_i
\right\| ^2
\right)^{1/2}  =
\left(
\sum_{i=1}^{n_p}
 \|x_i\| ^2
\right)^{1/2} \leqslant \sqrt{n_p}.
\end{aligned}
\]
Therefore
\[
\mathbb E_\varepsilon
\sup_{\|w\| ,\|v\| \leqslant 1}
\left|
\frac{1}{n_p}
\sum_{i=1}^{n_p}
\varepsilon_i
(w^\top x_i)_+
(v^\top x_i)_+
\right|
\leqslant
\frac{8}{\sqrt{n_p}}.
\]
\end{proof}
 
\end{lemma}

\subsection{Bound for penalized estimation}

We first consider an estimator obtained from  potentially infinitely many neurons, with exact optimization
of $\hat{\mathcal{R}}(u) + \lambda \Omega(u)$,
which corresponds to weight decay and our algorithm in \mysec{featlearn}.

\begin{proposition}
\label{prop:nnbound}
Assume that $\Omega(u_\ast)$ is finite, and that we have an  $\eta$-approximate minimizer $\hat{u}$ of $\hat{\mathcal{R}}(u) + \lambda \Omega(u)$ with  $\eta \leqslant \frac{1}{2}  \lambda \Omega(u^\ast)$. For a fixed $\delta \in (0,1)$, take 
\[ \lambda  \geqslant \frac{16}{\sqrt{n}}
\Big( 4 +  32    \Omega(u^\ast) + ( 1 +  \Omega(u^\ast)) \sqrt{ \log \frac{2}{\delta} }
\Big).\]
Then, with probability greater than $1-\delta$,
\[
 \Omega(\hat{u}) \leqslant 2 \Omega(u^\ast)
\ \ \mbox{ and } \ \
\mathcal{R}(\hat{u}) \leqslant 
\mathcal{R}(u^\ast)  + 2 \lambda  \Omega(u^\ast).
\]
\end{proposition}
\begin{proof}
 By definition of $\eta$-approximate minimizer, we have
 \[ \hat{\mathcal{R}}(\hat{u} ) + \lambda \Omega(\hat{u} ) \leqslant \eta+\inf_{u} \hat{\mathcal{R}}(u) + \lambda \Omega(u)
 \] without any constraints on the number of neurons, which is a minimization over a convex domain. We can then extend~\cite[Prop.~4.7]{bach2024learning}, with $\eta \leqslant \frac{1}{2}  \lambda \Omega(u^\ast)$.

Let $u^\ast_\lambda$ be a minimizer of ${ \mathcal{R}}(u) + \lambda \Omega(u) $.
We consider the convex set $\mathcal{C} = \{ u, \  { \mathcal{R}}(u) + \lambda \Omega(u) \leqslant
 { \mathcal{R}}(u^\ast_\lambda) + \lambda \Omega(u^\ast_\lambda)  + \lambda \Omega(u^\ast) \}$ of all $\lambda \Omega(u^\ast) $-approximate minimizers of the regularized expected risk, which is included in the convex set 
 $\mathcal{B} = \{ u, \ \Omega(u) \leqslant 2 \Omega(u^\ast) \} $ (since, if $ u \in \mathcal{C}$, $\mathcal{R}(u) + \lambda \Omega(u)
 \leqslant { \mathcal{R}}(u^\ast_\lambda) + \lambda \Omega(u^\ast_\lambda) + \lambda \Omega(u^\ast)
 \leqslant { \mathcal{R}}(u^\ast) + \lambda \Omega(u^\ast) + \lambda \Omega(u^\ast) 
 \leqslant  { \mathcal{R}}(u) + 2 \lambda \Omega(u^\ast) $, which implies $u \in \mathcal{B}$).
 
 Using Lemma~\ref{lemma:special}, with probability greater than $1-\delta$, we have, for any $u$ such that $\Omega(u) \leqslant 2 \Omega(u^\ast)$ (i.e., $u \in \mathcal{B}$)
 \BEQ
 \label{eq:cc}  \big|\mathcal{R}(u) 
- \hat{\mathcal{R}}(u) \big|  \leqslant \frac{16 \Omega(u^\ast)}{\sqrt{n}} + \frac{128    \Omega(u^\ast)^2}{\sqrt{n}} + ( 4\Omega(u^\ast) + 4 \Omega(u^\ast)^2) \sqrt{ \frac{1}{n} \log \frac{2}{\delta} },
\EEQ
which we now assume. We now show by contradiction that we have $\hat{u} \in \mathcal{C}$. If $\hat{u} \notin \mathcal{C}$, by convexity there is $v$ on the segment $[u_\lambda^\ast, \hat{u}]$ such that $v$ is on the boundary of $\mathcal{C}$ (and in $\mathcal{B}$), that is, 
${\mathcal{R}}(v) + \lambda \Omega(v) =
 { \mathcal{R}}(u^\ast_\lambda) + \lambda \Omega(u^\ast_\lambda) + \lambda \Omega(u^\ast)$. Since $u \mapsto \hat{\mathcal{R}}(u)+\lambda \Omega(u)$ is convex, we have 
 \BEQ
 \label{eq:dddd}
 \hat{\mathcal{R}}(v) + \lambda \Omega(v) 
 \leqslant \max\{
\hat {\mathcal{R}}(\hat{u}) + \lambda \Omega(\hat{u}) ,  \hat{\mathcal{R}}(u_\lambda^\ast) + \lambda \Omega(u_\lambda^\ast) 
 \} \leqslant  \hat{\mathcal{R}}(u_\lambda^\ast) + \lambda \Omega(u_\lambda^\ast) 
 + \eta.
 \EEQ
 We can then apply \eq{cc} to $v$ and $u^\ast_\lambda$, leading to
\BEAS
\mathcal{R}(v) - \hat{\mathcal{R}}(v)
+ \hat{\mathcal{R}}(u^\ast_\lambda)- \mathcal{R}(u^\ast_\lambda) 
\leqslant \frac{32 \Omega(u^\ast)}{\sqrt{n}} + \frac{256    \Omega(u^\ast)^2}{\sqrt{n}} + ( 8\Omega(u^\ast) + 8 \Omega(u^\ast)^2) \sqrt{ \frac{1}{n} \log \frac{2}{\delta} },
\EEAS
while \eq{dddd} leads to
\BEAS
\mathcal{R}(v) - \hat{\mathcal{R}}(v)
+ \hat{\mathcal{R}}(u^\ast_\lambda)- \mathcal{R}(u^\ast_\lambda) 
& \geqslant & 
\mathcal{R}(v) +\lambda \Omega(v) - \mathcal{R}(u^\ast_\lambda) - \lambda \Omega(u^\ast_\lambda)- \eta = \lambda \Omega(u^\ast) - \eta.
\EEAS
We consider  $\eta \leqslant \frac{1}{2}  \lambda \Omega(u^\ast)$, and since
\[
\frac{32 \Omega(u^\ast)}{\sqrt{n}} + \frac{256    \Omega(u^\ast)^2}{\sqrt{n}} + ( 8\Omega(u^\ast) + 8 \Omega(u^\ast)^2) \sqrt{ \frac{1}{n} \log \frac{2}{\delta} }< \frac{1}{2}  \lambda \Omega(u^\ast),
\]
we obtain a contradiction.

We thus have $\hat{u} \in \mathcal{C} \subset \mathcal{B}$, that is,
\[
\Omega(\hat{u}) \leqslant 2 \Omega(u^\ast)
\ \mbox{ and } \
\mathcal{R}(\hat{u}) \leqslant 
\mathcal{R}(\hat{u})  + \lambda \Omega(\hat{u})
\leqslant 
\mathcal{R}(u_\lambda^\ast)  + \lambda \Omega(u_\lambda^\ast) + \lambda \Omega(u^\ast)
\leqslant 
\mathcal{R}(u^\ast)  + 2 \lambda  \Omega(u^\ast).
\]
 \end{proof}
 
We can now extend Prop.~\ref{prop:nnbound} to the optimization with only $m$ neurons, by using the representation of an unconstrained minimizer obtained from the Frank-Wolfe algorithm, as done in~\cite[Sec.~9.3.6]{bach2024learning}.

\begin{proposition}
\label{prop:nnboundnew}
Assume that \(B =\Omega(u^*)<\infty\), and that \(\hat u\) is an
\(\eta\)-approximate minimizer of
\[
        \hat{\mathcal R}(u)+\lambda\Omega(u)
\]
over networks with at most \(m\) neurons, with
$
        \eta\leqslant \frac14\lambda B .
$
For a fixed \(\delta\in(0,1)\), assume
\[
        \lambda
        \geqslant
        \frac{16}{\sqrt n}
        \big(
            4+32B+(1+B)\sqrt{\log\frac{2}{\delta}}
        \big),
        \qquad
        \frac{128B}{m}\leqslant \lambda .
\]
Then, with probability at least \(1-\delta\),
\[
        \Omega(\hat u)\leq 2B,
        \qquad
        \mathcal R(\hat u)
        \leq
        \mathcal R(u^*)+2\lambda B .
\]
\end{proposition}

\begin{proof}
Let \(\bar u\) be an exact minimizer of
$
        \hat{\mathcal R}(u)+\lambda\Omega(u)
$
over the unconstrained, possibly infinite-neuron, class. We work on the
concentration event used in the proof of Prop.~\ref{prop:nnbound}. On this event,
the proof of Prop.~\ref{prop:nnbound} applied with \(\eta=0\) gives
\[
        \Omega(\bar u)\leqslant 2B .
\]

We now approximate \(\bar u\) by an \(m\)-neuron network. Applying
\(m\) steps of the Frank-Wolfe algorithm
(with step-size $2/(k+1)$ at iteration $k$)~\cite{jaggi2013revisiting}  to minimize
\(\hat{\mathcal R}(u)\) on the ball
\[
        \{u:\Omega(u)\leqslant \Omega(\bar u)\},
\]
there exists an \(m\)-neuron network \(w_m\) such that
\[
        \Omega(w_m)\leqslant \Omega(\bar u),
        \qquad
        \hat{\mathcal R}(w_m)
        \leqslant
        \hat{\mathcal R}(\bar u)
        +
        \frac{32B^2}{m}.
\]
Therefore,
\[
\begin{aligned}
        \hat{\mathcal R}(w_m)+\lambda\Omega(w_m)
        &\leqslant
        \hat{\mathcal R}(\bar u)+\lambda\Omega(\bar u)
        +
        \frac{32B^2}{m}        =
        \inf_u
        \big\{
            \hat{\mathcal R}(u)+\lambda\Omega(u)
        \big\}
        +
        \frac{32B^2}{m}.
\end{aligned}
\]
Since \(\hat u\) is an \(\eta\)-approximate minimizer over the
\(m\)-neuron class and \(w_m\) is an \(m\)-neuron candidate,
\[
\begin{aligned}
        \hat{\mathcal R}(\hat u)+\lambda\Omega(\hat u)
        &\leqslant
        \hat{\mathcal R}(w_m)+\lambda\Omega(w_m)+\eta        \\
        &\leqslant
        \inf_u
        \left\{
            \hat{\mathcal R}(u)+\lambda\Omega(u)
        \right\}
        +
        \eta
        +
        \frac{32B^2}{m}.
\end{aligned}
\]
By the assumptions
\[
        \eta\leqslant \frac14\lambda B,
        \qquad
        \frac{32B^2}{m}\leqslant \frac14\lambda B,
\]
\(\hat u\) is a \(\frac12\lambda B\)-approximate minimizer of the
unconstrained regularized empirical problem. Prop.~\ref{prop:nnbound} therefore
applies to \(\hat u\), yielding
$
        \Omega(\hat u)\leqslant 2B,
        $ and $
        \mathcal R(\hat u)
        \leqslant
        \mathcal R(u^*)+2\lambda B .
$
\end{proof}

\subsection{Final bound}

We now complete the proof of the neural-network bound. Let
\(n=\min\{n_p,n_q\}\), let \(r=dp/dq\), and assume that
\(r\in[\alpha,\alpha^{-1}]\) for some \(\alpha>0\). Assume moreover that
there exists \(W\in\mathbb R^{d\times d_{\rm eff}}\) with orthonormal columns, and a function
\(h\) such that
\[
        r(x)=h(W^\top x).
\]
For each \(\rho\in[0,1]\), the exact population optimizer in Eq.~(3) is
\[
        u_\rho^*(x)
        =
        \frac{r(x)-1}{\rho r(x)+1-\rho}
        =
        f_\rho(W^\top x),
\        \mbox{ with } \ 
        f_\rho(z)
        =
        \frac{h(z)-1}{\rho h(z)+1-\rho}.
\]
Since \(\rho r(x)+1-\rho\geqslant \alpha\) uniformly in \(\rho\) and \(x\),
the assumed Sobolev regularity of \(h\) transfers uniformly to the
family \((f_\rho)_{\rho\in[0,1]}\). Hence, by a vector-valued
ReLU variation-space bound on the reduced \(d_{\rm eff}\)-dimensional
space, the family \((u_\rho^*)_{\rho\in[0,1]}\) has finite
group-variation norm whenever
\[
        t>\frac{d_{\rm eff}+3}{2}.
\]
With
$
        B =\Omega(u^*)<\infty ,
$
for \(\delta\in(0,1)\), define
\[
        \varepsilon_B(\delta)
         =
        \frac{16B+128B^2}{\sqrt n}
        +
        (4B+4B^2)\sqrt{\frac1n\log\frac{4}{\delta}} .
\]
Assume that \(\hat u\) is an \(\eta\)-approximate minimizer of
\[
        \hat{\mathcal R}(u)+\lambda\Omega(u)
\]
over networks with at most \(m\) neurons, with
$
        \eta\leqslant \frac14\lambda B .
$
Assume also that
\[
        \lambda
        \geqslant
        \frac{16}{\sqrt n}
        \left(
            4+32B+(1+B)\sqrt{\log\frac{4}{\delta}}
        \right),
        \qquad
        m\geqslant \frac{128B}{\lambda}.
\]
Applying Prop.~\ref{prop:nnboundnew} with failure probability \(\delta/2\), and using
the norm bound obtained in the proof of Prop.~\ref{prop:nnboundnew} through
Prop.~\ref{prop:nnbound}, we have, with probability at least \(1-\delta/2\),
\[
        \Omega(\hat u)\leqslant 2B,
        \qquad
        \mathcal R(\hat u)
        \leqslant
        \mathcal R(u^*)+2\lambda B .
\]
On the other hand, applying Lemma~\ref{lemma:special} with \(t=2B\) to both
\(\mathcal R-\hat{\mathcal R}\) and
\(\hat{\mathcal R}-\mathcal R\), with failure probability
\(\delta/2\), gives
\[
        \sup_{\Omega(u)\leqslant 2B}
        \left|
            \mathcal R(u)-\hat{\mathcal R}(u)
        \right|
        \leqslant
        \varepsilon_B(\delta).
\]
Therefore, on the intersection of these two events, which has
probability at least \(1-\delta\),
\[
        \Omega(\hat u)\leqslant 2B,
        \qquad
        \left|
            \mathcal R(\hat u)-\hat{\mathcal R}(\hat u)
        \right|
        \leqslant
        \varepsilon_B(\delta),
\]
and
\[
        \mathcal R(\hat u)
        \leqslant
        \mathcal R(u^*)+2\lambda B .
\]

Since \(u^*\) is the exact optimizer of the population variational
problem, we have
\[
        \mathcal R(u^*)=-D(p\|q).
\]
Moreover, for every admissible \(u\), the population variational value
\(-\mathcal R(u)\) is a lower bound on \(D(p\|q)\). Hence
\[
        0
        \leqslant
        D(p\|q)-[-\mathcal R(\hat u)]
        =
        D(p\|q)+\mathcal R(\hat u)
        \leqslant
        2\lambda B .
\]

We now pass from the population variational value to the empirical
regularized value. Define
\[
        \hat V_{\rm pen}
        =
        -\hat{\mathcal R}(\hat u)-\lambda\Omega(\hat u).
\]
Then
\[
\begin{aligned}
        \hat V_{\rm pen}-D(p\|q)
        &=
        \mathcal R(\hat u)-\hat{\mathcal R}(\hat u)
        -
        \bigl(D(p\|q)+\mathcal R(\hat u)\bigr)
        -
        \lambda\Omega(\hat u)           \leqslant
        \varepsilon_B(\delta),
\end{aligned}
\]
and
\[
\begin{aligned}
        D(p\|q)-\hat V_{\rm pen}
        &=
        D(p\|q)+\mathcal R(\hat u)
        +
        \hat{\mathcal R}(\hat u)-\mathcal R(\hat u)
        +
        \lambda\Omega(\hat u)                                      \\
        &\leqslant
        2\lambda B+\varepsilon_B(\delta)+2\lambda B     =
        \varepsilon_B(\delta)+4\lambda B .
\end{aligned}
\]
Consequently,
\[
        \big|
            \hat V_{\rm pen}-D(p\|q)
        \big|
        \leqslant
        \varepsilon_B(\delta)+4\lambda B .
\]
By the lower bound assumed on \(\lambda\),
\[
        \varepsilon_B(\delta)\leqslant \frac14\lambda B,
\]
and therefore
\[
        \big|
            \hat V_{\rm pen}-D(p\|q)
        \big|
        \leqslant
        \frac{17}{4}\lambda B .
\]
Thus, for \(\lambda\asymp n^{-1/2}\) and
\(m\gtrsim B/\lambda\), equivalently \(n=O(m^2)\) up to constants,
the penalized empirical neural-network objective satisfies
\[
        \big|
            \hat V_{\rm pen}-D(p\|q)
        \big|
        =
        O(n^{-1/2})
\]
with the displayed logarithmic dependence on the confidence parameter.

It remains to relate this quantity to the finite-dimensional
weight-decay objective used by the algorithm. For a finite network in
the scale-balanced parametrization of App.~\ref{app:proofneuraldecay}, with feature map
\(\Psi_{\hat w,\hat b}\), the scale-optimized group-variation objective
satisfies
\[
        \hat V_{\rm pen}
        =
        F_\lambda(\hat p\|\hat q,\Psi_{\hat w,\hat b})
        -
        \frac{\lambda}{2}
        \sum_{j=1}^m
        \bigl(\|\hat w_j\|^2+\hat b_j^2\bigr).
\]
Equivalently,
\[
        \hat V_{\rm pen}
        =
        -\hat{\mathcal R}(\hat u)-\lambda\Omega(\hat u).
\]
Thus the bound above applies to the regularized weight-decay objective,
not to the unpenalized empirical quantity
\(F(\hat p\|\hat q,\hat\phi)\).

If the statement is instead expressed in terms of
\(F_\lambda(\hat p\|\hat q,\Psi_{\hat w,\hat b})\) without subtracting
the hidden-layer weight-decay term, then in the scale-balanced
parametrization,
\[
        F_\lambda(\hat p\|\hat q,\Psi_{\hat w,\hat b})
        =
        \hat V_{\rm pen}
        +
        \frac{\lambda}{2}\Omega(\hat u).
\]
Since \(\Omega(\hat u)\leqslant 2B\), this gives
\[
\begin{aligned}
        \big|
            F_\lambda(\hat p\|\hat q,\Psi_{\hat w,\hat b})
            -
            D(p\|q)
        \big|
        &\leqslant
        \big|
            \hat V_{\rm pen}-D(p\|q)
        \big|
        +
        \frac{\lambda}{2}\Omega(\hat u)           \leqslant
        \varepsilon_B(\delta)+5\lambda B     \leqslant
        \frac{21}{4}\lambda B .
\end{aligned}
\]
This proves the same \(O(n^{-1/2})\) high-probability rate for
\(F_\lambda(\hat p\|\hat q,\Psi_{\hat w,\hat b})\), provided the
output layer is the ridge-optimal one for the learned hidden features
and the finite network is represented in the scale-balanced form of
App.~\ref{app:proofneuraldecay}.

The argument above certifies the statistical behavior of an approximate
global minimizer of the regularized empirical risk. It does not by
itself certify that the stochastic gradient ascent procedure used in
the experiments reaches such a minimizer.

\section{Feature learning for mutual information}
\label{app:featlearnMI}
 We focus on the situation from \mysec{mutual}, where we have a joint distribution $p(x_1,x_2)$ on the product space $\X_1 \times \X_2$. We aim to learn two feature maps $\psi^{(1)}_{\Gamma_1}: \X_1 \to \rb^{m_1}$ and $\psi^{(2)}_{\Gamma_2}: \X_2 \to \rb^{m_2}$ (of potentially different dimensions) such that the $F$-mutual information obtained from these two maps is as large as possible. If we manage to reach the true mutual information, then, by studying equality cases in the data-processing inequality for $f$-divergences~\cite{polyanskiy2025information}, this means that $x_1$ and $x_2$ are independent given $\psi^{(1)}_{\Gamma_1}(x_1)$ and $\psi^{(2)}_{\Gamma_2}(x_2)$. 

In the linear case, the moments $\mu_p, \mu_q, \Sigma_p, \Sigma_q$, defined in \mysec{mutual}, have sizes $d_1d_2$ and $d_1d_2 \times d_1d_2$ with $d_i$ the dimension of the feature space $\H_i$ ($i=1,\dots,2$); we then need to maximize the function $F(p\|q, (\theta_2 \otimes \theta_1)^\top (\varphi_2 \otimes \varphi_1))$ with respect to $\theta_1$ and $\theta_2$. Since $F$ is convex and homogeneous, it is lower-bounded by its tangent, which is of the form
\[
2c^\top \! ( \theta_2\otimes \theta_1)\!^\top\! (\mu_p \!-\! \mu_q)   
+\! \tr\!\big[M  ( \theta_2\otimes \theta_1)\!^\top \Sigma_p   ( \theta_2\otimes \theta_1) \big]
\!+ \!\tr\!\big[N  ( \theta_2\otimes \theta_1)\!^\top \Sigma_q  ( \theta_2\otimes \theta_1) \big],
\]
which is tight at a given $(\theta_1,\theta_2)$ when $(M,N,c)$ are obtained from Prop.~\ref{prop:closedform} applied to the transformed moments.
A natural algorithm is thus alternating optimization over $\theta_1$, $\theta_2$ (each of these problems is quadratic), and $(M,N,c)$, where all steps can be performed in closed form (with an extra whitening step for stability so that $( \theta_2\otimes \theta_1)\!^\top \Sigma_q  ( \theta_2\otimes \theta_1) = \idm$). This is an ascent algorithm with no guarantees, except for the Pearson divergence, where this recovers exactly the power method for canonical correlation analysis (CCA)~\cite{anderson2003introduction,bach2002kernel}. Note that it is affine-invariant and that it is not sequential (that is, the features learned for certain $m_1,m_2$ may not include the ones learned for lower values), but that this sequentiality could be enforced directly.

The simplest implementation will compute the moment matrices of maximal size $d_1 d_2 \times d_1d_2$, and then needs to solve linear systems in dimension $d_1 m_1$ and $d_2 m_2$, thus with an overall complexity of $O(d_1^2 d_2^2 n + d_1^3 m_1^3 + d_2^3 m_2^3)$ per iteration. This is already an improvement when $m_1, m_2$ are much smaller than $d_1,d_2$, compared with the situation without feature learning, which has complexity  $O(d_1^2 d_2^2 n + d_1^3 d_2^3)$. Using gradient ascent techniques as in \mysec{featlearn}, possibly with variance reduction, we can reduce it to 
$O(n d_1 m_1 + n d_2 m_2 + m_1^3 m_2^3)$. This can be extended to more general feature spaces and stochastic approximation.

 \paragraph{Mutual information and logistic/softmax regression.} \label{app:softmax} With $m=m_1$ (feature dimension) and $k = m_2$ (number of classes), we have $\mu_p = (\pi_j \mu_j)_{j \in \{1,\dots,k\}} \in \rb^{mk}$, $\mu_q = (\pi_j \mu)_{j \in \{1,\dots,k\} } \in \rb^{mk}$, $\Sigma_p = \Diag( (\pi_j \Sigma_j)_{j \in \{1,\dots,k\}}) \in \rb^{mk \times mk}$, $\Sigma_q = \Diag( (\pi_j \Sigma)_{j \in \{1,\dots,k\}}) \in \rb^{mk \times mk}$. The estimation decouples into
\[
\sum_{j=1}^k \pi_j  \sum_{i=1}^m
\frac{ f(\lambda_i^{(j)} )}{(\lambda_i^{(j)} - 1)^2}  \big( (   \mu_j  - \mu )^\top v_i^{(j)} \big)^2,
\]
where $(v_i^{(j)})_{i \in \{1,\dots,m\}}$ are the (properly normalized so that $(v_i^{(j)})^\top \Sigma  v_i^{(j)} = 1$) eigenvectors of the pair 
$(\Sigma_j,\Sigma)$, with generalized eigenvalues $(\lambda_i^{(j)})_{i \in \{1,\dots,m\}}$. Note here that we use a slightly different normalization of eigenvectors (without the coefficient $\pi_j$).

From $M, N \in \rb^{mk \times mk}$ obtained from Prop.~\ref{prop:closedform}, we only need the diagonal $m\times m$ blocks $M_{j} \in \rb^{m \times m}$, while we need all the blocks $c_j \in \rb^{m}$ of $c \in \rb^{mk}$. 
The required blocks are, with $h(t) = f(t) / (t-1)^2$,
\[
  c_j=\sum_i h(\lambda_i^{(j)})v_i^{(j)}(v_i^{(j)})^\top\delta_j,
\]
\[
  M_j=\sum_{i,\ell}
  ((\delta_j)^\top v_i^{(j)})((\delta_j)^\top v_\ell^{(j)})
  \frac{h(\lambda_i^{(j)})-h(\lambda_\ell^{(j)})}
       {\lambda_i^{(j)}-\lambda_\ell^{(j)}}
  v_i^{(j)}(v_\ell^{(j)})^\top,
\]
and
\[
  N_j=-\sum_{i,\ell}
  ((\delta_j)^\top v_i^{(j)})((\delta_j)^\top v_\ell^{(j)})
  \frac{\lambda_i^{(j)}h(\lambda_i^{(j)})-\lambda_\ell^{(j)}h(\lambda_\ell^{(j)})}
       {\lambda_i^{(j)}-\lambda_\ell^{(j)}}
  v_i^{(j)}(v_\ell^{(j)})^\top,
\]
with the usual derivative interpretation when two eigenvalues coincide.

\section{Additional experiments and details}
\label{app:experiments}

\begin{figure}
\begin{center}
\hspace*{-.1cm}\includegraphics[width=14cm]{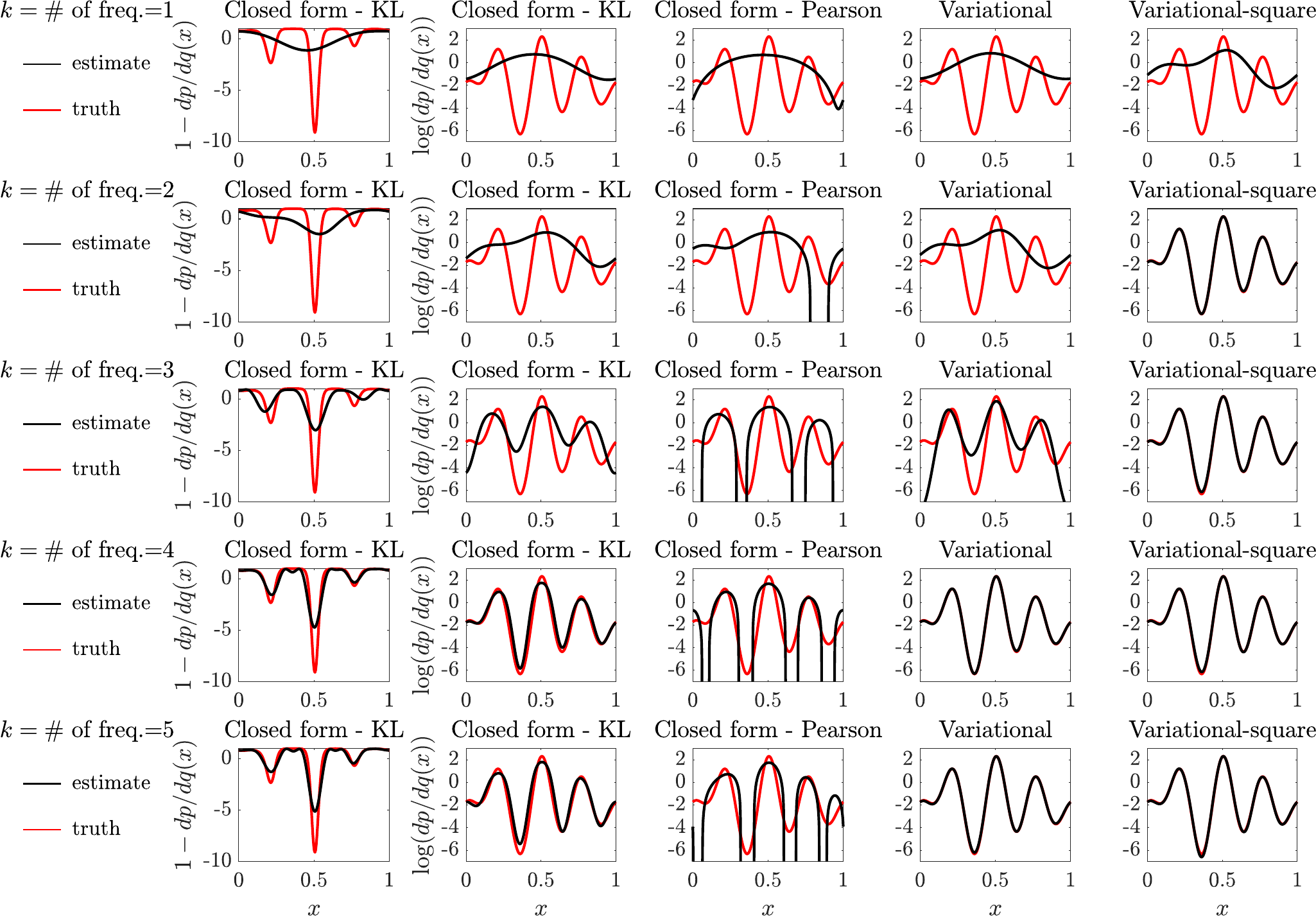}
\end{center}

\vspace*{-.3cm}

\caption{Estimation of $\log(dp/dq(x))$ for $x \in [0,1]$, and feature spaces composed of frequencies up to order $k$ (space of dimension $2k+1$), with sines and cosines, with $n = 100 000$ (thus approaching the population limit). From left to right: estimation of $1-dp/dq(x)$ 
using our closed-form estimator for the KL divergence, then estimation of $\log (dp/dq(x))$ with the KL estimator, the Pearson estimator (aimed at estimating $dp/dq(x)-1$), then the classical variational method based on \eq{varDPq} with a linear function in $\varphi(x)$ and a quadratic function (``square'').  The log-density can be perfectly fitted with 4 frequencies; hence the perfect fit for variational-square with $k = 2$ and for variational with $k = 4$.
Note that the KL estimator is not perfect after $k=4$ because the potential $w(x)$ cannot get a perfect fit (left plot). The Pearson estimate cannot deal well with small values of the relative density (small log). \label{fig:examples_regular}}
\end{figure}

\begin{figure}
\begin{center}
\hspace*{-.1cm}\includegraphics[width=14cm]{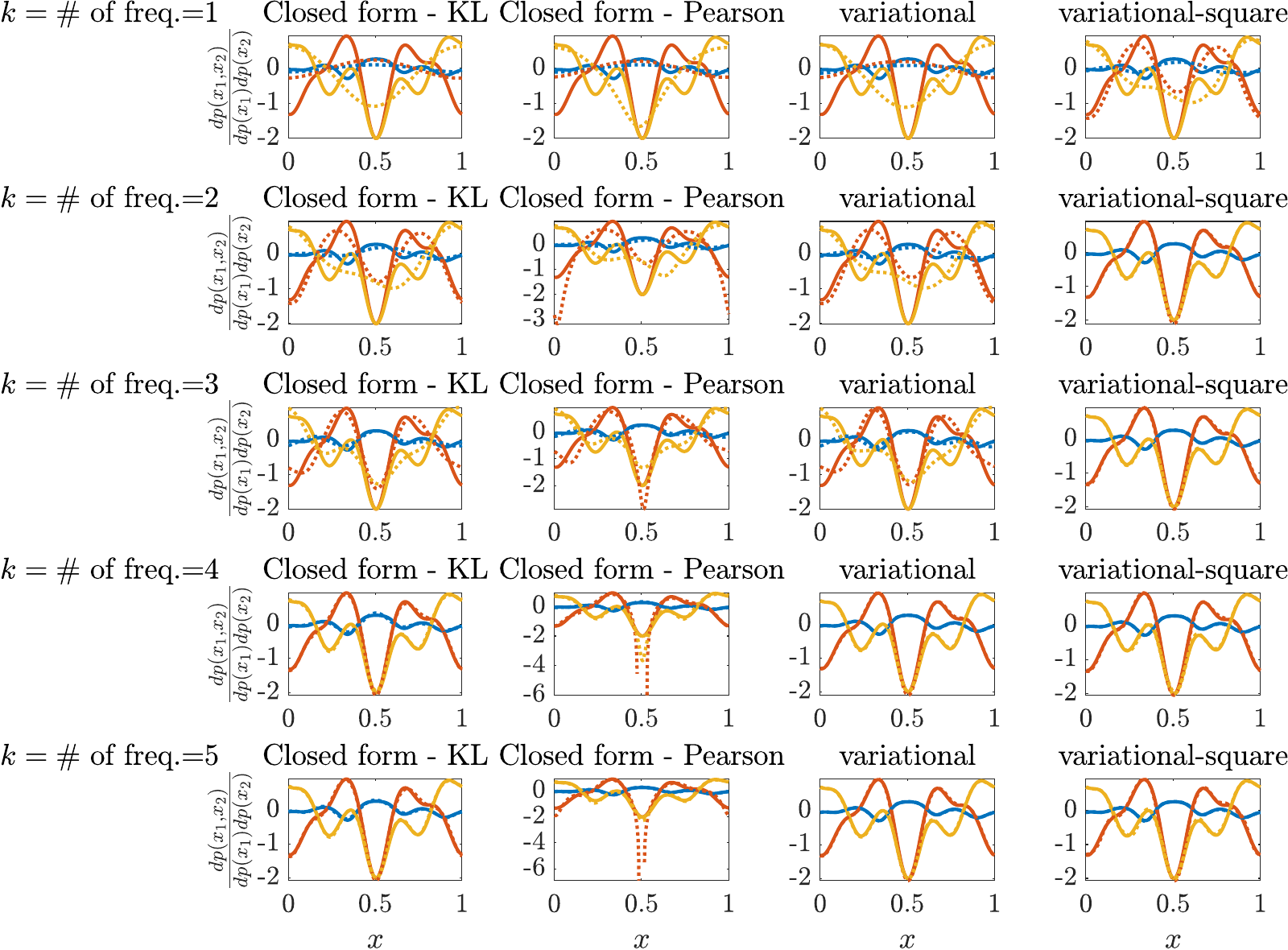}
\end{center}

\vspace*{-.2cm}

\caption{Estimation of $\log(dp(x_1,x_2)/dp(x_1)dp(x_2))$ for $x_1\in [0,1]$ and $x_2 \in \{1,2,3\}$, corresponding to a softmax regression problem, with feature spaces composed of frequencies up to order $k$ (space of dimension $2k+1$), with sines and cosines, and $n = 100 000$ (thus approaching the population limit), with one plot for each value of $x_2$. From left to right:  KL estimator,  Pearson estimator, then classical variational method with linear and square features. 
The log-density can be perfectly fitted with 4 frequencies, hence the perfect fit for variational-square with $k=2$, and variational with $k=4$. Note that the KL estimator is not perfect after $k=4$ but is a tight approximation. The Pearson estimate cannot deal well with small values of the relative density (small log). \label{fig:examples_MI}}
\end{figure}

\textbf{Compared estimators.} We compare our regularized closed-form estimators from Prop.~\ref{prop:closedform} with a single parameter (regularization parameter $\lambda$) with a fixed kernel (except for the last neural network experiment). We compare with the following existing \emph{simple} estimators (either closed-form or with a single optimization problem, and with a single hyperparameter): \BIT
\item Variational formulation~\cite{nguyen2010estimating}: we parameterize $v: \X\to\rb$ as $v(x) =   \theta^\top \varphi(x) $ and maximize  $\frac{1}{n_p} \sum_{i=1}^{n_p}
v(x_i) - \log \big( \frac{1}{n_q} \sum_{j=1}^{n_q} e^{v(y_j)} \big) - \frac{\lambda}{2} \| \theta\|^2$, which corresponds to optimizing in closed form with respect to a constant.\footnote{Indeed, we have 
$ \sup_{ a \in \rb}  
  \int_\X\! ( v(x) + a) dp(x)  -  \int_\X \! ( e^{v(x) + a} - 1) dq(x)
  =  \int_\X\! v(x) dp(x)  -  \log \big( \int_\X \! e^{v(x)  }  dq(x)\big)$.}
 This can be optimized using Newton's method (with each step having the same complexity as our estimator). This is not a closed-form estimator.  The ``square'' variant corresponds to using a quadratic function of the feature map, at a higher computational cost.
\item Plug-in estimator with kernel density estimation~\cite{joe1989estimation,wang2005divergence,kandasamy2015nonparametric}: we consider $\hat{p}(x) = \frac{1}{n_p} \sum_{i=1}^{n_p} k(x,x_i)$ and $\hat{q}(x) = \frac{1}{n_q} \sum_{i=1}^{n_q} k(x,y_i)$ for the Gaussian kernel as densities with respect to the Lebesgue measure, which leads to estimates of the relative density by taking ratios.
\item Pearson estimator: we use the closed-form estimator for the Pearson divergence~\cite{ribero2026regularized}, by adding a small constant when the learned density is thresholded to zero.
\EIT
For the mutual information, similar estimators can be defined, in particular for the classical variational formulation from \eq{varDPq}~\cite{belghazi2018mine,poole2019variational}.

\textbf{Relation to softmax regression.}
\label{app:softmaxvar}
When \(\X_2\) is discrete, the usual variational KL representation of mutual information, which corresponds to \eq{varDPq} applied to $p(x_1,x_2)$ and $q(x_1,x_2) = p(x_1) p(x_2)$ (product of the two marginal distributions of $p$):
\[
I(x_1;x_2)
=
\sup_{v : \X_1 \times \X_2 \to \rb}
\left\{
\mathbb E_{p(x_1,x_2)}[v(x_1,x_2)]
-
  \mathbb E_{p(x_1)p(x_2)}\!\left[e^{v(x_1,x_2)}\right]
+ 1\right\}.
\]
The variational formulation can be optimized in general using gradient ascent iterations (or Newton's  method) by adding a constant to $v$ and optimizing over it, as for the  variational formulation for generic $p$ and $q$, that is, by maximizing
\BEQ
\label{eq:varmi}
\sup_{v : \X_1 \times \X_2 \to \rb}
\left\{
\mathbb E_{p(x_1,x_2)}[v(x_1,x_2)]
-
 \log \Big( \mathbb E_{p(x_1)p(x_2)}\!\left[e^{v(x_1,x_2)}\right]
\Big) \right\}.
\EEQ
The formulation in \eq{varmi} is not equivalent to softmax regression. However, if instead we add a function of $x_1$ to $v(x_1,x_2)$ and optimize over it in closed form, we get the problem:
\BEAS
& & 
\sup_{v : \X_1 \times \X_2 \to \rb}
\sup_{ a: \X_1 \to \rb} 
\left\{
\mathbb E_{p(x_1,x_2)}[v(x_1,x_2) + a(x_1)]
-
 \mathbb E_{p(x_1)p(x_2)}\!\left[e^{v(x_1,x_2) + a(x_1)}\right]
+ 1\right\} \\
& = & 
\sup_{v : \X_1 \times \X_2 \to \rb}
\sup_{ a: \X_1 \to \rb} 
\left\{
\mathbb E_{p(x_1,x_2)}[v(x_1,x_2)]  + \E_{p(x_1)} 
\Big[ a(x_1) 
- e^{a(x_1)} \mathbb E_{p(x_2)} [ e^{v(x_1,x_2)}] +1 \Big]  \right\}  \\
\\
& = & 
\sup_{v : \X_1 \times \X_2 \to \rb}
 \left\{
\mathbb E_{p(x_1,x_2)}[v(x_1,x_2)]  -  \E_{p(x_1)} \log  \E_{p(x_2)} [ e^{v(x_1,x_2)}]
   \right\} ,
\EEAS
which is now exactly softmax regression.

\textbf{Computing resources.} All experiments were run on a single standard laptop (MacBook M4 Pro) without any use of GPUs (with an overall runtime under 4 hours).

\textbf{Additional experimental details.}
We consider experiments from \mysec{experiments}.
Unless otherwise stated, all synthetic distributions are normalized analytically and sampled exactly. True divergences and test errors are computed either analytically when available or by sampling with a large number of samples. Hyperparameters are selected from a log-spaced grid. When a figure is intended to compare the best statistical performance of estimators rather than model-selection procedures, we select the hyperparameter by minimizing the reported test criterion on independent validation data. Error bars show standard errors over the stated number of independent replications. Random features and neural-network initializations use independent random seeds across replications.

\BIT
\item
\emph{Estimation of $D(p\|q)$:} For \myfig{bias} (left and center), \(q\) is uniform on \([0,1]\). The density ratio \(dp/dq\) is chosen as an affine function of the Bernoulli-kernel feature map,
$
  \frac{dp}{dq}(x)=1+\beta r(x),
$
where \(r\) is the fixed centered trigonometric feature combination and \(\beta\) is chosen so that \(dp/dq \) is strictly positive. The kernel is
$
  k(x,y)=1+\sum_{i\ge1}\frac{2}{i^2}\cos(2\pi i(x-y)).
$
We use \(n_p=n_q=n\) with \(n=2^5,\ldots,2^{10}\), and report means and standard errors over 64 repetitions. We consider a fixed schedule of $\lambda$, which depends on $n$ proportionally to $n^{-2/3}$ (which is the one optimizing the regular estimator, so that the debiased estimator can only have smaller constant, not a different power).\footnote{In this paper, we perform two types of experiments: this one on scaling law, where $\lambda$ is fixed a priori before looking at the data, and we report training objectives, and all other ones, where $\lambda$ is estimated on validation data and we report testing objectives.}

 The compared methods are the regular estimator, its debiased version, and the quantum-divergence estimator.  The center plot reports the logarithm of the absolute error against \(\log_2 n\); the displayed exponents are least-squares slopes fitted to the data.

For \myfig{bias} (right), data are periodic on \([0,1]^2\) and mapped coordinatewise to the product of unit circles by \(x_\ell\mapsto(\cos 2\pi x_\ell,\sin 2\pi x_\ell)\). The reference distribution \(q\) is uniform. The log-density ratio \(\log(dp/dq)\) has independent coordinates and is a finite cosine expansion; the corresponding density ratio is normalized to integrate to one. We use \(m=512\) random ReLU features of order one and sample sizes \(n=512,1024,2048,4096\), with an independent test set of 1024 observations. The reported quantity is \(1-\hat D/D(p\|q)\), which is  the corresponding normalized test error for the potentials (lower is better).  For each method, the hyperparameter is chosen as the value giving the smallest validation error (on an independent dataset). 

\item
\emph{Comparison with softmax regression:} For \myfig{softmax} (left and center) we use a softmax-regression problem with \(k=3\) classes. Conditional log-densities on \([0,1]^2\) are finite cosine expansions. 

For the left plot, we draw \(n=100000\) training samples. A large random ReLU feature dictionary of size 2000 is first generated, principal component analysis (PCA) is applied to this dictionary, and the top \(r\) components are retained, with $r$ varying from 1 to 200. We compare the spectral closed-form estimator with standard softmax regression optimized by Newton's method using the same features. The left panel reports the training variational mutual-information objective as a function of \(r\). 

For the center panel, we generate 512 random features, consider the two estimators learned from $n=1000$ observations with a regularization parameter estimated on validation data of size $10000$, with test errors on the mutual-information estimate reported on $10000$ (independent) observations.

\item \emph{Estimation of mutual information.} For \myfig{softmax} (right), we consider both $x_1$ and $x_2$ on the unit circle in $\rb^2$, with a joint distribution whose log-density is a sum of a few cosines. We consider 500 ReLU random features and learn on $n$ observations, with $n = 100, 200, 400, 800$. Hyperparameters are selected on a validation set of size $2000$, and test performance is computed on an independent set of the same size.

\item
\emph{Feature learning for mutual information.} For \myfig{comparisonNN} (left), the data are generated from a distribution on $(x_1,x_2)$, where $x_1$ and $x_2$ are each two-dimensional and share a common dimension. We perform principal component analysis on 1000 random features and select $m=50$ (so that we can still run the full-size algorithm which requires eigenvalue decompositions of matrices of size $m^2 \times m^2$), on which we run the linear feature-learning algorithm of \mysec{featlearn} with reduced dimensions \(r=1,\ldots,24\), and compare the resulting mutual-information lower bound (on the training set of $n=2000$ samples) with the full-feature estimator. The plot shows that approximately 24 learned components recover the full-feature value on this example.

\item
\emph{Feature learning with neural networks.} For \myfig{comparisonNN} (right),   we use the latent-variable setting from Prop.~\ref{prop:latent}: the density ratio depends only on a fixed low-dimensional projection while the ambient dimension is \(d=2,4,8,16,32,64,128\). We use \(n=10000\) samples to learn and four replications. The neural-network models are one-hidden-layer ReLU networks with $m=100$ hidden units for our estimator, and $m= 100$ hidden units for the variational methods, trained with stochastic gradient ascent with weight decay (with the reduction through the matrix $\Gamma \in \rb^{m \times r}$ described in App.~\ref{app:proofneuraldecay}, with $r=8$). We compare the closed-form estimator with fixed random features (KL), the learned-feature neural version (KL-NN), the direct variational estimator with fixed features, and its neural-network version. Hyperparameters are selected on an independent validation set of size 10000, and the reported metric is the normalized test error \(1-\hat D/D(p\|q)\) (on an independent set of the same size), for which lower is better.
\EIT

\begin{figure}
\begin{center}
\hspace*{-.1cm}\includegraphics[width=6cm]{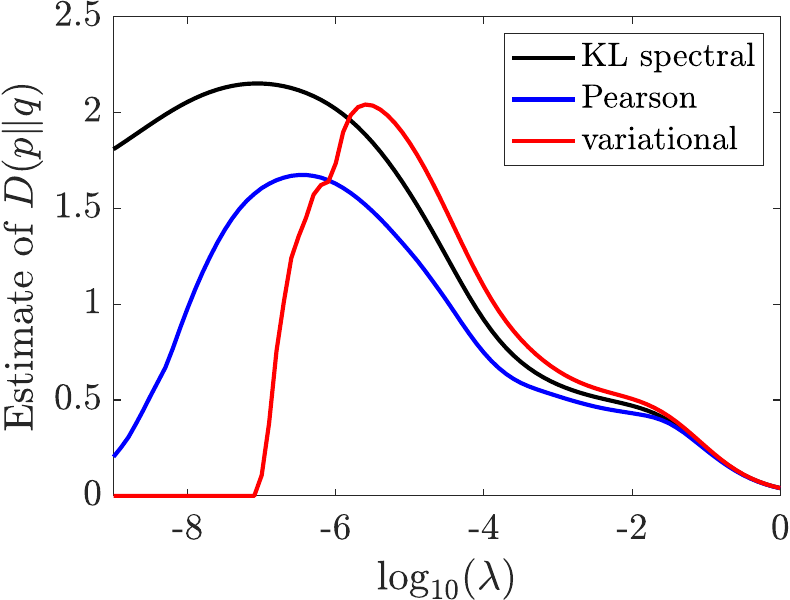}
\end{center}

\vspace*{-.2cm}

\caption{Comparison of estimates of $D(p\|q)$ for the KL divergence for different estimates, as a function of the ridge regularization parameter $\lambda$. Data from \myfig{bias}  with $n=1024$.\label{fig:lambdas}}
\end{figure}

%%%%%%%%%%%%%%%%%%%%%%%%%%%%%%%%%%%%%%%%%%%%%%%%%%%%%%%%%%%%

\vspace*{2.5cm}

\input{checklist.tex}

\end{document}

%% file: checklist.tex
\section*{NeurIPS Paper Checklist}

\begin{enumerate}

\item {\bf Claims}
    \item[] Question: Do the main claims made in the abstract and introduction accurately reflect the paper's contributions and scope?
    \item[] Answer:   \answerYes{}.
    \item[] Justification: The abstract and introduction summarize the closed-form spectral estimator, feature-learning extensions, theoretical guarantees, and synthetic experimental scope.
   
\item {\bf Limitations}
    \item[] Question: Does the paper discuss the limitations of the work performed by the authors?
    \item[] Answer:   \answerYes{}.
     \item[] Justification: We are explicit in saying that our experiments are on synthetic data, that for normalized models such as softmax regression, there are no statistical gains (and in fact a small loss).
  
\item {\bf Theory assumptions and proofs}
    \item[] Question: For each theoretical result, does the paper provide the full set of assumptions and a complete (and correct) proof?
    \item[] Answer:  \answerYes{}.
    \item[] Justification: \mysec{theory} provides all assumptions, with explicit detailed proofs in the Appendix.
     \item {\bf Experimental result reproducibility}
    \item[] Question: Does the paper fully disclose all the information needed to reproduce the main experimental results of the paper to the extent that it affects the main claims and/or conclusions of the paper (regardless of whether the code and data are provided or not)?
    \item[] Answer:  \answerYes{}.
    \item[] Justification: The main paper has a short explanation in the figures; App.~\ref{app:experiments} gives all details. A full MATLAB implementation is available at {\small\url{https://www.di.ens.fr/~fbach/fdiv_2_0.zip}}

\item {\bf Open access to data and code}
    \item[] Question: Does the paper provide open access to the data and code, with sufficient instructions to faithfully reproduce the main experimental results, as described in supplemental material?
    \item[] Answer: \answerYes{}.
    \item[] Justification: A full MATLAB implementation is available.

\item {\bf Experimental setting/details}
    \item[] Question: Does the paper specify all the training and test details (e.g., data splits, hyperparameters, how they were chosen, type of optimizer) necessary to understand the results?
     \item[] Answer:  \answerYes{}.
    \item[] Justification: The main paper has a short explanation in the figures; App.~\ref{app:experiments} gives the experimental details. A full MATLAB implementation is available.

\item {\bf Experiment statistical significance}
    \item[] Question: Does the paper report error bars suitably and correctly defined or other appropriate information about the statistical significance of the experiments?
    \item[] Answer:  \answerYes{}.
    \item[] Justification: Error bars are included in the stochastic comparison plots where repeated trials are reported; they denote standard errors over the stated repetitions.

\item {\bf Experiments compute resources}
    \item[] Question: For each experiment, does the paper provide sufficient information on the computer resources (type of compute workers, memory, time of execution) needed to reproduce the experiments?
    \item[] Answer: \answerYes{}.
    \item[] Justification: As stated in the Appendix, all experiments were run on a single standard laptop without any use of GPUs.

\item {\bf Code of ethics}
    \item[] Question: Does the research conducted in the paper conform, in every respect, with the NeurIPS Code of Ethics \url{https://neurips.cc/public/EthicsGuidelines}?
    \item[] Answer:  \answerYes{}.
    \item[] Justification: This is mostly a theoretical / algorithmic paper with synthetic experiments.

\item {\bf Broader impacts}
    \item[] Question: Does the paper discuss both potential positive societal impacts and negative societal impacts of the work performed?
    \item[] Answer:  \answerNA{}
    \item[] Justification: This is mostly a theoretical / algorithmic paper with synthetic experiments.
        
\item {\bf Safeguards}
    \item[] Question: Does the paper describe safeguards that have been put in place for responsible release of data or models that have a high risk for misuse (e.g., pre-trained language models, image generators, or scraped datasets)?
    \item[] Answer:  \answerNA{}.
    \item[] Justification: This is mostly a theoretical / algorithmic paper with synthetic experiments.

\item {\bf Licenses for existing assets}
    \item[] Question: Are the creators or original owners of assets (e.g., code, data, models), used in the paper, properly credited and are the license and terms of use explicitly mentioned and properly respected?
    \item[] Answer:  \answerNA{}.
    \item[] Justification: \answerNA{}.
  
\item {\bf New assets}
    \item[] Question: Are new assets introduced in the paper well documented and is the documentation provided alongside the assets?
    \item[] Answer: \answerYes{}.
    \item[] Justification: Documentation/instructions are included.

\item {\bf Crowdsourcing and research with human subjects}
    \item[] Question: For crowdsourcing experiments and research with human subjects, does the paper include the full text of instructions given to participants and screenshots, if applicable, as well as details about compensation (if any)? 
    \item[] Answer: \answerNA{}.
    \item[] Justification: \answerNA{}.

\item {\bf Institutional review board (IRB) approvals or equivalent for research with human subjects}
    \item[] Question: Does the paper describe potential risks incurred by study participants, whether such risks were disclosed to the subjects, and whether Institutional Review Board (IRB) approvals (or an equivalent approval/review based on the requirements of your country or institution) were obtained?
    \item[] Answer: \answerNA{}.
    \item[] Justification: \answerNA{}.

\item {\bf Declaration of LLM usage}
    \item[] Question: Does the paper describe the usage of LLMs if it is an important, original, or non-standard component of the core methods in this research? Note that if the LLM is used only for writing, editing, or formatting purposes and does \emph{not} impact the core methodology, scientific rigor, or originality of the research, declaration is not required.
    %this research? 
    \item[] Answer: \answerNA{}.
    \item[] Justification: LLMs were used for editing, formatting, and checking purposes, to help with the related-work section and to write a few helper functions for the code (as noted in the released code).

\end{enumerate}

%% file: fdiv-reloaded-neurips-arxiv.bbl
\begin{thebibliography}{10}

\bibitem{perez2008kullback}
Fernando P{\'e}rez-Cruz.
\newblock Kullback-{L}eibler divergence estimation of continuous distributions.
\newblock In {\em International Symposium on Information Theory}, pages
  1666--1670, 2008.

\bibitem{polyanskiy2025information}
Yury Polyanskiy and Yihong Wu.
\newblock {\em Information Theory: From Coding to Learning}.
\newblock Cambridge University Press, 2025.

\bibitem{broniatowski2006minimization}
Michel Broniatowski and Amor Keziou.
\newblock Minimization of $\varphi$-divergences on sets of signed measures.
\newblock {\em Studia Scientiarum Mathematicarum Hungarica}, 43(4):403--442,
  2006.

\bibitem{nguyen2010estimating}
XuanLong Nguyen, Martin~J. Wainwright, and Michael~I. Jordan.
\newblock Estimating divergence functionals and the likelihood ratio by convex
  risk minimization.
\newblock {\em IEEE Transactions on Information Theory}, 56(11):5847--5861,
  2010.

\bibitem{liu2015estimating}
Qiang Liu, Jian Peng, Alexander Ihler, and John~W. Fisher~III.
\newblock Estimating the partition function by discriminance sampling.
\newblock In {\em Conference on Uncertainty in Artificial Intelligence}, 2015.

\bibitem{chehab2023provable}
Omar Chehab, Aapo Hyvarinen, and Andrej Risteski.
\newblock Provable benefits of annealing for estimating normalizing constants:
  Importance sampling, noise-contrastive estimation, and beyond.
\newblock In {\em Advances in Neural Information Processing Systems}, 2023.

\bibitem{hastie1995penalized}
Trevor Hastie, Andreas Buja, and Robert Tibshirani.
\newblock Penalized discriminant analysis.
\newblock {\em The Annals of Statistics}, 23(1):73--102, 1995.

\bibitem{harchaoui2008testing}
Zaid Harchaoui, Francis Bach, and Eric Moulines.
\newblock Testing for homogeneity with kernel {F}isher discriminant analysis.
\newblock Technical Report 0804.1026, arXiv, 2008.

\bibitem{kanamori2009least}
Takafumi Kanamori, Shohei Hido, and Masashi Sugiyama.
\newblock A least-squares approach to direct importance estimation.
\newblock {\em Journal of Machine Learning Research}, 10:1391--1445, 2009.

\bibitem{sugiyama2012density}
Masashi Sugiyama, Taiji Suzuki, and Takafumi Kanamori.
\newblock {\em Density Ratio Estimation in Machine Learning}.
\newblock Cambridge University Press, 2012.

\bibitem{bach2024learning}
Francis Bach.
\newblock {\em Learning Theory from First Principles}.
\newblock MIT Press, 2024.

\bibitem{scholkopf2002learning}
Bernhard Sch{\"o}lkopf and Alexander~J. Smola.
\newblock {\em Learning with Kernels: Support Vector Machines, Regularization,
  Optimization, and Beyond}.
\newblock MIT Press, 2002.

\bibitem{mccullagh2019generalized}
Peter McCullagh and John~A. Nelder.
\newblock {\em Generalized Linear Models}.
\newblock Chapman \& Hall, 1989.

\bibitem{yamada2013relative}
Makoto Yamada, Taiji Suzuki, Takafumi Kanamori, Hirotaka Hachiya, and Masashi
  Sugiyama.
\newblock Relative density-ratio estimation for robust distribution comparison.
\newblock {\em Neural Computation}, 25(5):1324--1370, 2013.

\bibitem{lesniewski1999monotone}
Andrew Lesniewski and Mary~Beth Ruskai.
\newblock Monotone {R}iemannian metrics and relative entropy on noncommutative
  probability spaces.
\newblock {\em Journal of Mathematical Physics}, 40(11):5702--5724, 1999.

\bibitem{bach2022sum}
Francis Bach.
\newblock Sum-of-squares relaxations for information theory and variational
  inference.
\newblock {\em Foundations of Computational Mathematics}, 25(3):865--903, 2025.

\bibitem{bach2022information}
Francis Bach.
\newblock Information theory with kernel methods.
\newblock {\em IEEE Transactions on Information Theory}, 69(2):752--775, 2023.

\bibitem{cappe2009line}
Olivier Capp{\'e} and Eric Moulines.
\newblock On-line expectation--maximization algorithm for latent data models.
\newblock {\em Journal of the Royal Statistical Society Series B: Statistical
  Methodology}, 71(3):593--613, 2009.

\bibitem{bach2017breaking}
Francis Bach.
\newblock Breaking the curse of dimensionality with convex neural networks.
\newblock {\em Journal of Machine Learning Research}, 18(19):1--53, 2017.

\bibitem{hyvarinen2005score}
Aapo Hyv{\"a}rinen.
\newblock Estimation of non-normalized statistical models by score matching.
\newblock {\em Journal of Machine Learning Research}, 6(4), 2005.

\bibitem{gutmann2012noise}
Michael~U. Gutmann and Aapo Hyv{\"a}rinen.
\newblock Noise-contrastive estimation of unnormalized statistical models, with
  applications to natural image statistics.
\newblock {\em Journal of Machine Learning Research}, 13(11):307--361, 2012.

\bibitem{ribero2026regularized}
M{\'o}nica Ribero, Antonin Schrab, and Arthur Gretton.
\newblock Regularized $ f $-divergence kernel tests.
\newblock Technical Report 2601.19755, arXiv, 2026.

\bibitem{JMLR:v12:reid11a}
Mark~D. Reid and Robert~C. Williamson.
\newblock Information, divergence and risk for binary experiments.
\newblock {\em Journal of Machine Learning Research}, 12(22):731--817, 2011.

\bibitem{henrion2020moment}
Didier Henrion, Milan Korda, and Jean~Bernard Lasserre.
\newblock {\em The Moment-{SOS} Hierarchy: Lectures in Probability, Statistics,
  Computational Geometry, Control and Nonlinear PDEs}.
\newblock World Scientific, 2020.

\bibitem{eric2007testing}
Zaid Harchaoui, Francis Bach, and Eric Moulines.
\newblock Testing for homogeneity with kernel {F}isher discriminant analysis.
\newblock In {\em Advances in Neural Information Processing Systems}, 2007.

\bibitem{bach2002kernel}
Francis Bach and Michael~I. Jordan.
\newblock Kernel independent component analysis.
\newblock {\em Journal of Machine Learning Research}, 3(Jul):1--48, 2002.

\bibitem{gretton2012kernel}
Arthur Gretton, Karsten~M. Borgwardt, Malte~J. Rasch, Bernhard Sch{\"o}lkopf,
  and Alexander Smola.
\newblock A kernel two-sample test.
\newblock {\em Journal of Machine Learning Research}, 13(1):723--773, 2012.

\bibitem{rahimi2008random}
Ali Rahimi and Benjamin Recht.
\newblock Random features for large-scale kernel machines.
\newblock In {\em Advances in Neural Information Processing Systems}, 2008.

\bibitem{rudi2017generalization}
Alessandro Rudi and Lorenzo Rosasco.
\newblock Generalization properties of learning with random features.
\newblock In {\em Advances in Neural Information Processing Systems}, 2017.

\bibitem{mahoney2009cur}
Michael~W. Mahoney and Petros Drineas.
\newblock {CUR} matrix decompositions for improved data analysis.
\newblock {\em Proceedings of the National Academy of Sciences},
  106(3):697--702, 2009.

\bibitem{martinsson2020randomized}
Per-Gunnar Martinsson and Joel~A. Tropp.
\newblock Randomized numerical linear algebra: Foundations and algorithms.
\newblock {\em Acta Numerica}, 29:403--572, 2020.

\bibitem{williams2000using}
Christopher Williams and Matthias Seeger.
\newblock Using the {N}ystr{\"o}m method to speed up kernel machines.
\newblock In {\em Advances in Neural Information Processing Systems}, 2000.

\bibitem{rudi2015less}
Alessandro Rudi, Raffaello Camoriano, and Lorenzo Rosasco.
\newblock Less is more: {N}ystr{\"o}m computational regularization.
\newblock In {\em Advances in Neural Information Processing Systems}, 2015.

\bibitem{bach2023relationship}
Francis Bach.
\newblock On the relationship between multivariate splines and infinitely-wide
  neural networks.
\newblock Technical Report 2302.03459, arXiv, 2023.

\bibitem{csiszar1967information}
Imre Csisz{\'a}r.
\newblock Information-type measures of difference of probability distributions
  and indirect observation.
\newblock {\em Studia Scientiarum Mathematicarum Hungarica}, 2:229--318, 1967.

\bibitem{letizia2024mutual}
Nunzio~A. Letizia, Nicola Novello, and Andrea~M. Tonello.
\newblock Mutual information estimation via $f$-divergence and data
  derangements.
\newblock In {\em Advances in Neural Information Processing Systems}, 2024.

\bibitem{suzuki2010sufficient}
Taiji Suzuki and Masashi Sugiyama.
\newblock Sufficient dimension reduction via squared-loss mutual information
  estimation.
\newblock In {\em International Conference on Artificial Intelligence and
  Statistics}, 2010.

\bibitem{belghazi2018mine}
Mohamed~Ishmael Belghazi, Aristide Baratin, Sai Rajeshwar, Sherjil Ozair,
  Yoshua Bengio, Aaron Courville, and Devon Hjelm.
\newblock Mutual information neural estimation.
\newblock In {\em International Conference on Machine Learning}, 2018.

\bibitem{poole2019variational}
Ben Poole, Sherjil Ozair, Aaron van~den Oord, Alex Alemi, and George Tucker.
\newblock On variational bounds of mutual information.
\newblock In {\em International Conference on Machine Learning}, 2019.

\bibitem{bach2004multiple}
Francis Bach, Gert R.~G. Lanckriet, and Michael~I. Jordan.
\newblock Multiple kernel learning, conic duality, and the {SMO} algorithm.
\newblock In {\em International Conference on Machine Learning}, 2004.

\bibitem{rakotomamonjy2008simplemkl}
Alain Rakotomamonjy, Francis Bach, St{\'e}phane Canu, and Yves Grandvalet.
\newblock Simple{MKL}.
\newblock {\em Journal of Machine Learning Research}, 9:2491--2521, 2008.

\bibitem{gonen2011multiple}
Mehmet G{\"o}nen and Ethem Alpayd{\i}n.
\newblock Multiple kernel learning algorithms.
\newblock {\em Journal of Machine Learning Research}, 12:2211--2268, 2011.

\bibitem{follain2025enhanced}
Bertille Follain and Francis Bach.
\newblock Enhanced feature learning via regularisation: Integrating neural
  networks and kernel methods.
\newblock {\em Journal of Machine Learning Research}, 26(172):1--56, 2025.

\bibitem{hunter2004tutorial}
David~R Hunter and Kenneth Lange.
\newblock A tutorial on {MM} algorithms.
\newblock {\em The American Statistician}, 58(1):30--37, 2004.

\bibitem{mairal2013stochastic}
Julien Mairal.
\newblock Stochastic majorization-minimization algorithms for large-scale
  optimization.
\newblock In {\em Advances in Neural Information Processing Systems}, 2013.

\bibitem{golub83matrix}
Gene~H. Golub and Charles F.~Van Loan.
\newblock {\em Matrix Computations}.
\newblock Johns Hopkins University Press, 1996.

\bibitem{Bach2026_gaussian}
Francis Bach.
\newblock Regularized variational and spectral log-density-ratio estimation in
  the {G}aussian location model.
\newblock Technical Report 2607.01895, arXiv, 2026.

\bibitem{sriperumbudur2011universality}
Bharath~K. Sriperumbudur, Kenji Fukumizu, and Gert R.~G. Lanckriet.
\newblock Universality, characteristic kernels and {RKHS} embedding of
  measures.
\newblock {\em Journal of Machine Learning Research}, 12(70), 2011.

\bibitem{laurent1996efficient}
B{\'e}atrice Laurent.
\newblock Efficient estimation of integral functionals of a density.
\newblock {\em The Annals of Statistics}, 24(2):659--681, 1996.

\bibitem{lee2019u}
Alan~J. Lee.
\newblock {\em U-Statistics: Theory and Practice}.
\newblock Routledge, 2019.

\bibitem{chizat2018global}
Lenaic Chizat and Francis Bach.
\newblock On the global convergence of gradient descent for over-parameterized
  models using optimal transport.
\newblock In {\em Advances in Neural Information Processing Systems}, 2018.

\bibitem{schubert2018numerically}
Erich Schubert and Michael Gertz.
\newblock Numerically stable parallel computation of (co-)variance.
\newblock In {\em International Conference on Scientific and Statistical
  Database Management}, 2018.

\bibitem{gribonval2021sketching}
R{\'e}mi Gribonval, Antoine Chatalic, Nicolas Keriven, Vincent Schellekens,
  Laurent Jacques, and Philip Schniter.
\newblock Sketching data sets for large-scale learning: Keeping only what you
  need.
\newblock {\em IEEE Signal Processing Magazine}, 38(5):12--36, 2021.

\bibitem{wainwright2008graphical}
Martin~J. Wainwright and Michael~I. Jordan.
\newblock Graphical models, exponential families, and variational inference.
\newblock {\em Foundations and Trends in Machine Learning}, 1(1-2):1--305,
  2008.

\bibitem{suzuki2011least}
Taiji Suzuki and Masashi Sugiyama.
\newblock Least-squares independent component analysis.
\newblock {\em Neural Computation}, 23(1):284--301, 2011.

\bibitem{sugiyama2007direct}
Masashi Sugiyama, Shinichi Nakajima, Hisashi Kashima, Paul von B\"{u}nau, and
  Motoaki Kawanabe.
\newblock Direct importance estimation with model selection and its application
  to covariate shift adaptation.
\newblock In {\em Advances in Neural Information Processing Systems}, 2007.

\bibitem{nowozin2016f}
Sebastian Nowozin, Botond Cseke, and Ryota Tomioka.
\newblock f-{GAN}: Training generative neural samplers using variational
  divergence minimization.
\newblock In {\em Advances in Neural Information Processing Systems}, 2016.

\bibitem{zellingerbinary}
Werner Zellinger.
\newblock Binary losses for density ratio estimation.
\newblock In {\em International Conference on Learning Representations}, 2025.

\bibitem{zellinger2023adaptive}
Werner Zellinger, Stefan Kindermann, and Sergei~V. Pereverzyev.
\newblock Adaptive learning of density ratios in {RKHS}.
\newblock {\em Journal of Machine Learning Research}, 24(395):1--28, 2023.

\bibitem{menon2016linking}
Aditya Menon and Cheng~Soon Ong.
\newblock Linking losses for density ratio and class-probability estimation.
\newblock In {\em International Conference on Machine Learning}, 2016.

\bibitem{Boyd2004Convex}
Stephen Boyd and Lieven Vandenberghe.
\newblock {\em Convex Optimization}.
\newblock Cambridge {U}niversity {P}ress, 2004.

\bibitem{ando1979concavity}
Tsuyoshi Ando.
\newblock Concavity of certain maps on positive definite matrices and
  applications to {H}adamard products.
\newblock {\em Linear Algebra and its Applications}, 26:203--241, 1979.

\bibitem{lewis1996derivatives}
Adrian~S. Lewis.
\newblock Derivatives of spectral functions.
\newblock {\em Mathematics of Operations Research}, 21(3):576--588, 1996.

\bibitem{cap2007optimal}
Andrea {Caponnetto} and Ernesto {De Vito}.
\newblock {Optimal rates for the regularized least-squares algorithm}.
\newblock {\em Foundations of Computational Mathematics}, 7(3):331--368, 2007.

\bibitem{lin2020optimal}
Junhong Lin, Alessandro Rudi, Lorenzo Rosasco, and Volkan Cevher.
\newblock Optimal rates for spectral algorithms with least-squares regression
  over {H}ilbert spaces.
\newblock {\em Applied and Computational Harmonic Analysis}, 48(3):868--890,
  2020.

\bibitem{bach2013sharp}
Francis Bach.
\newblock Sharp analysis of low-rank kernel matrix approximations.
\newblock In {\em Conference on Learning Theory}, 2013.

\bibitem{alaoui2015fast}
Ahmed El~Alaoui and Michael~W. Mahoney.
\newblock Fast randomized kernel ridge regression with statistical guarantees.
\newblock In {\em Advances in Neural Information Processing Systems}, 2015.

\bibitem{tropp2015introduction}
Joel~A. Tropp.
\newblock An introduction to matrix concentration inequalities.
\newblock {\em Foundations and Trends in Machine Learning}, 8(1-2):1--230,
  2015.

\bibitem{neubauer1997converse}
Andreas Neubauer.
\newblock On converse and saturation results for {T}ikhonov regularization of
  linear ill-posed problems.
\newblock {\em SIAM Journal on Numerical Analysis}, 34(2):517--527, 1997.

\bibitem{fischer2020sobolev}
Simon Fischer and Ingo Steinwart.
\newblock Sobolev norm learning rates for regularized least-squares algorithms.
\newblock {\em Journal of Machine Learning Research}, 21(205):1--38, 2020.

\bibitem{kurkova2001bounds}
Vera Kurkov{\'a} and Marcello Sanguineti.
\newblock Bounds on rates of variable-basis and neural-network approximation.
\newblock {\em IEEE Transactions on Information Theory}, 47(6):2659--2665,
  2001.

\bibitem{maurer2016vector}
Andreas Maurer.
\newblock A vector-contraction inequality for {R}ademacher complexities.
\newblock In {\em International Conference on Algorithmic Learning Theory},
  2016.

\bibitem{jaggi2013revisiting}
Martin Jaggi.
\newblock Revisiting {F}rank-{W}olfe: Projection-free sparse convex
  optimization.
\newblock In {\em International Conference on Machine Learning}, 2013.

\bibitem{anderson2003introduction}
Theodore~W. Anderson.
\newblock {\em An Introduction to Multivariate Statistical Analysis}.
\newblock Wiley, 2003.

\bibitem{joe1989estimation}
Harry Joe.
\newblock Estimation of entropy and other functionals of a multivariate
  density.
\newblock {\em Annals of the Institute of Statistical Mathematics},
  41(4):683--697, 1989.

\bibitem{wang2005divergence}
Qing Wang, Sanjeev~R. Kulkarni, and Sergio Verd{\'u}.
\newblock Divergence estimation of continuous distributions based on
  data-dependent partitions.
\newblock {\em IEEE Transactions on Information Theory}, 51(9):3064--3074,
  2005.

\bibitem{kandasamy2015nonparametric}
Kirthevasan Kandasamy, Akshay Krishnamurthy, Barnabas Poczos, Larry Wasserman,
  and James~M. Robins.
\newblock Nonparametric von {M}ises estimators for entropies, divergences and
  mutual informations.
\newblock In {\em Advances in Neural Information Processing Systems}, 2015.

\end{thebibliography}
